%% file: main.tex
\documentclass[twoside,11pt]{article}

\input{version.tex}
\usepackage{ifthen}
\newcommand{\jmlronly}[1]{%
    \ifthenelse{\equal{\version}{jmlr}}%
    {%
        #1%
    }%
    {}%
}

\newcommand{\arxivonly}[1]{%
    \ifthenelse{\equal{\version}{arxiv}}%
    {%
        #1%
    }%
    {}%
}

\arxivonly{\usepackage[preprint]{jmlr2e}}
\jmlronly{\usepackage{jmlr2e}}
\input{cleveref.tex}

\usepackage{appendix}
\input{preamble.tex}
\input{local-preamble.tex}
\input{defs.tex}
\input{symbols.tex}
\input{utils.tex}

\usepackage{lastpage} 
\jmlrheading{23}{2022}{1-\pageref{LastPage}}{1/21; Revised 5/22}{9/22}{21-0000}{Ze Peng, Lei Qi, Yinghuan Shi and Yang Gao}
\ShortHeadings{A Theoretical Explanation of Activation Sparsity}{Peng, Qi, Shi and Gao}
\firstpageno{1}

\jmlronly{\input{jmlronly.tex}}

\usepackage[nameinlink,capitalize]{cleveref}
\usepackage{autonum}

\crefname{lemma}{Lemma}{Lemmas}
\crefname{theorem}{Theorem}{Theorems}
\crefname{corollary}{Corollary}{Corollaries}
\crefname{observation}{Observation}{Observations}
\crefname{remark}{Remark}{Remarks}
\crefname{definition}{Definition}{Definition}

\begin{document}

\title{A Theoretical Explanation of Activation Sparsity through Flat Minima and Adversarial Robustness}

\author{\name Ze Peng \email pengze@smail.nju.edu.cn \\
        \addr State Key Laboratory for Novel Software Technology\\ Nanjing University\\
        Nanjing, China
        \AND
        \name Lei Qi \email qilei@seu.edu.cn \\
        \addr School of Computer Science and Engineering\\ Southeast University\\
        Nanjing, China
        \AND
        \name Yinghuan Shi\thanks{Corresponding author.} \email syh@nju.edu.cn
        \AND
        \name Yang Gao \email gaoy@nju.edu.cn\\
        \addr State Key Laboratory for Novel Software Technology\\ Nanjing University\\
        Nanjing, China}
\date{6 September 2023}
\jmlronly{\editor{\hphantom{1}}}

\maketitle

\input{intro/abstract.tex}

\input{intro/intro.tex}
\input{related/related.tex}
\input{preliminary/preliminary.tex}
\input{theory/theory.tex}
\input{algo/refined.tex}
\input{experiments/validational.tex}
\input{experiments/productive.tex}

\input{experiments/theoretical.tex}
\input{conclusion/conclusion.tex}

\appendix
\appendixpage
\addappheadtotoc
\input{theory/appendix.tex}
\input{theory/effective-window.tex}
\input{experiments/appendix.tex}
\arxivonly{\input{theory/magic.tex}}

\bibliography{ref}

\end{document}

%% file: version.tex
\newcommand{\version}{arxiv}

%% file: cleveref.tex
\usepackage{amsmath, mathtools}

\makeatletter
\def\cleartheorem#1{\expandafter\let\csname#1\endcsname\relax
    \expandafter\let\csname c@#1\endcsname\relax
}
\makeatother
\cleartheorem{example}
\cleartheorem{theorem}
\cleartheorem{lemma}
\cleartheorem{proposition}
\cleartheorem{definition}
\cleartheorem{corollary}
\cleartheorem{remark}

\newtheorem{theorem}{Theorem}
\newtheorem{lemma}{Lemma} 
\newtheorem{remark}{Remark} 
 
\newtheorem{corollary}{Corollary}
\newtheorem{definition}{Definition}
\newtheorem{observation}{Observation} 

%% file: preamble.tex
\usepackage{amsfonts, amssymb, mathrsfs}
\usepackage{hyperref}
\hypersetup{hidelinks}

\usepackage{xparse}
\usepackage{siunitx}

%% file: local-preamble.tex
\usepackage{abstract}
\usepackage{caption}
\usepackage{subcaption}
\usepackage{pgffor}
\usepackage{pgfmath}
\usepackage{listings}
\usepackage{booktabs}
\usepackage{bbding}
\usepackage{thmtools}
\usepackage{thm-restate}
\usepackage{etoolbox}

\usepackage[section]{placeins} % 阻止在段落前放置图片

\def\addvalue#1#2{\expandafter\gdef\csname my@data@#1\endcsname{#2}}
\def\usevalue#1{\csname my@data@#1\endcsname}
\usepackage{algorithm}
\usepackage{algpseudocode}

%% file: symbols.tex
\RenewDocumentCommand{\diag}{m}{\mathrm{Diag}\left(#1\right)}
\NewDocumentCommand{\dataset}{}{\mathcal{D}}
\NewDocumentCommand{\loss}{}{\mathcal{L}}
\NewDocumentCommand{\derivatives}{m m O{}}{\frac{\partial#3 #1}{\partial #2#3}}
\NewDocumentCommand{\hessian}{O{\theta}}{H_{#1}}
\NewDocumentCommand{\mlp}{}{\operatorname{MLP}}
\NewDocumentCommand{\activation}{}{\sigma}
\NewDocumentCommand{\inner}{m m}{\left\langle #1, #2 \right\rangle}
\NewDocumentCommand{\relu}{}{\operatorname{ReLU}}
\NewDocumentCommand{\gelu}{}{\operatorname{GELU}}
\RenewDocumentCommand{\defeq}{}{\coloneq}
\NewDocumentCommand{\positivecomplex}{}{\complexes^{+}}

\RenewDocumentCommand{\transpose}{}{\top}
\NewDocumentCommand{\defto}{}{\eqcolon}
\NewDocumentCommand{\sas}{}{\mathcal{S{\alpha}S}}
\NewDocumentCommand{\Eta}{}{H}

\NewDocumentCommand{\dbmlp}{}{\operatorname{DB-MLP}}
\NewDocumentCommand{\jrelu}{}{\operatorname{J-SquaredReLU}}
\NewDocumentCommand{\Alpha}{}{A}

\NewDocumentCommand{\HXT}{}{H^l \left(X^{l-1} + D^l\right)^\transpose}
\NewDocumentCommand{\GAlphaT}{}{G^l_V \left(\Alpha^{l}\right)^\transpose}

\NewDocumentCommand{\Epsilon}{}{E}
\NewDocumentCommand{\efftheta}{}{\tilde{\theta}}

\NewDocumentCommand{\vectorize}{m}{\mathrm{vec}\left(#1\right)}
\NewDocumentCommand{\kkT}{}{K^l\left(K^l\right)^\transpose}

\NewDocumentCommand{\magic}{O{}}{\mathrm{#1MagicSynapse}}
\NewDocumentCommand{\proxyW}{}{\hat{W}}

%% file: utils.tex
\newif\ifHeightAcquired
\NewDocumentCommand{\resetHeight}{}{\HeightAcquiredfalse}
\NewDocumentCommand{\myincludegraphics}{O{} m}{%
    \ifHeightAcquired%
        \includegraphics[height=\figheight]{#2}%
    \else%
        \let\tempbox\relax%
        \newsavebox{\tempbox}%
        \sbox{\tempbox}{\includegraphics[#1]{#2}}%
        \global\edef\figheight{\the\ht\tempbox}%
        \global\HeightAcquiredtrue%
        \usebox{\tempbox}%
    \fi%
}

%% file: intro/abstract.tex
\begin{abstract}%
A recent empirical observation \citep{observation} of activation sparsity in MLP blocks offers an opportunity to drastically reduce computation costs for free. 
Although having attributed it to training dynamics, existing theoretical explanations of activation sparsity are restricted to shallow networks, small training steps and special training, despite its emergence in deep models standardly trained for a large number of steps. 
To fill these gaps, we propose the notion of gradient sparsity as one source of activation sparsity and a theoretical explanation based on it that sees sparsity as a necessary step to adversarial robustness w.r.t. hidden features and parameters, which is approximately the flatness of minima for well-learned models.
The theory applies to standardly trained LayerNorm-ed MLPs, and further to Transformers or other architectures trained with weight noises.
Eliminating other sources of flatness except for sparsity, we discover the phenomenon that the ratio between the largest and smallest non-zero singular values of weight matrices is small. 
When discussing the emergence of this spectral concentration, we use random matrix theory (RMT) as a powerful tool to analyze stochastic gradient noises.
Experiments for validation are conducted to verify our gradient-sparsity-based explanation. 
We propose two plug-and-play modules for both training and finetuning for sparsity.
Experiments on ImageNet-1K and C4 demonstrate their $50\%$ sparsity improvements, indicating further potential cost reduction in both training and inference.
\end{abstract}


%% file: intro/intro.tex
\section{Introduction}

Despite the success of overparameterized deep neural networks, how they learn, work and generalize well is not fully understood. 
Although more than half of parameters and computation are devoted to them \citep{knowledge_neurons}, even in Transformers, MLP blocks have remained black box for years, blocking interpretation, manipulation and pruning in them.
Recently, some works \citep{mlp_as_database,knowledge_neurons} have tried to dive into MLP blocks and reveal its relation with learned knowledge by rewriting MLP blocks of Transformers into an attention mechanism where keys and values are provided by the first and second linear layers, and form a key-value memory that stores learned knowledge. 

A more recent work \citep{observation} further discovers activation sparsity in MLP blocks, i.e., only a small portion of neurons are activated during inference, within MLP blocks of pure MLP, ResNet, T5, ($\relu$) ViT and other architectures on various tasks, \emph{without} explicit regularization. This discovery not only brings research attention back to MLP blocks when it comes to understanding Transformers but also leads to potential aggressive but unstructured and dynamic neuron pruning during inference and thus large reduction in inference cost. Specifically, if pruning is ideally implemented to skip non-activated neurons, theoretically in T5 the inference costs happening in MLP blocks can be astonishingly reduced by about $90\%$. There is also an increasing tendency of activation sparsity, measured in percentage of zero activations, when models become larger \citep{observation}, hinting at potentially great cost reduction for large models. 
Activation sparsity in CNNs, although weaker \citep{observation,exploit_sparsity_in_CNN}, has been exploited to provide on-the-fly pruning or compression during inference \citep{exploit_sparsity_in_CNN}. 

However, the emergence of sparsity is not yet fully understood. Several works \citep{observation,sharpness_aware,large_step,from_noises} have been proposed to explain the emergence of sparsity from training dynamics. 
\citet{observation} explain activation sparsity of the last layer during the first step by computing gradients and exploiting properties of initialization methods. 
When sharpness-aware optimization is used, \citet{sharpness_aware} find positive components in gradients that point towards norm reduction of activations. However, only a shallow 2-layer MLP is studied by \citet{sharpness_aware}. \citet{large_step} consider second-order behaviors of SGD and proves sparsity on diagonal 2-layer MLPs, but for general networks it is only conjectured. 
\citet{from_noises} find that noises added to samples improve sparsity. However, the noises are manually imposed and are not included in standard augmentations. 
Although having achieved better understanding on sparsity and training dynamics and hinting at the roles of noises and flatness in activation sparsity, these works are still restricted to shallow networks, small steps, additional regularization or augmentations that cannot be found in ubiquitous training protocols. Nevertheless, they point out the role of flatness or noises in the emergence of sparsity.

Filling the gap between experiments and these theoretical results, we propose a new theoretical explanation that applies to deep networks, large training steps and standard training practices, by further emphasizing flatness and noises.
In particular, we first rewrite the flatness bias of SGD into a tendency to improve implicit adversarial robustness w.r.t. hidden features and parameters. 
Gradient sparsity and effective gradient sparsity are proposed as causes of activation sparsity. 
To support this, we prove a theorem stating that these kinds of sparsity can be one of the sources of implicit adversarial robustness. Since flat minima bias puts constraints on all layers, basing our explanation on this inductive bias allows us to reason about deep layers.
To eliminate other potential sources of implicit adversarial robustness, we exploit LayerNorm layers, refer to an already discovered inductive bias called parameter growth \citep{parameter_growth} and empirically discover a new phenomenon of spectral concentration, i.e., the fraction between the largest and smallest non-zero singular values of the first weight matrices in MLP blocks is not very large.
We prove the emergence of spectral concentration at initialization. We also theoretically discuss its re-emergence and maintenance during later training using random matrix theory (RMT) by extracting from stochastic gradients two large random matrices, indicating that training stochasticity's contribution to sparsity is two-folded. Notably, thanks to RMT our formulation of stochastic gradient noises is very direct, without assuming Gaussian or $\sas$ distributions on them but uniform and independent sampling \emph{within a batch} and bounds on anisotropy and gradient norms that can be estimated empirically.
Following these theoretical insights, we propose two plug-and-play and orthogonal architectural modifications, which brings relatively approximately $50\%$ training sparsity improvements and at least $36\%$ on testing sparsity compared to existing naturally emergent sparsity of $10\%$ or $20\%$ non-zero activations. 
The structure of this work is as follows:
\begin{itemize}
    \item In \cref{sec:preliminary}, we introduce preliminaries and background information. But in \cref{sec:preliminary:sparsity}, we propose gradient sparsity, distinguish it from activation sparsity and argue its importance over activation sparsity;
    \item In \cref{sec:illustration}, we build an intuitive framework of our explanation from flat minima and implicit adversarial robustness. Two almost plug-and-play architectural modifications, $\dbmlp$ and $\jrelu$, are proposed to further improve sparsity and ease formal analyses;
    \item In \cref{sec:gradients}, gradients of MLP blocks are computed, laying the basis for later formal analyses. Effective gradient sparsity is defined in this subsection, and its connection with training updates as well as gradient sparsity is discussed. Specifically, \cref{lemma:eff_and_sparsity} connects effective gradient sparsity measure in $L_2$ norms to activation sparsity directly measured in $L_0$ norms for $\relu$ networks, while the similar connection under $\jrelu$ is approximately discussed;
    \item In \cref{sec:three_elements}, we prove \cref{theorem:main} that relates flat minima and implicit adversarial robustness to effective gradient sparsity and activation sparsity by proving relatively tight chained upperbounds among them, demonstrating that sparsity can be the source of implicit adversarial robustness imposed by flat minima;
    \item In \cref{sec:discussion}, we instantiate \cref{theorem:main} in several specific settings involving pure MLPs and Transformers.
    Among them, \cref{theorem:main_with_hidden_vectors_and_layernorm} proves the tendency toward effective gradient sparsity on pure MLPs with LayerNorms. We argue that effective gradient sparsity is more stable and powerful during the entire training than direct activation sparsity.
    \cref{theorem:main_with_effective_duplication} deals with Transformers and other architectures by assuming perturbation training like dropout or tokenwise synapse noises. \arxivonly{Another under-testing module called $\magic$ is immediate after \cref{theorem:main_with_effective_duplication} and is elaborated on in \cref{appendix:magic}.}
        We discuss the effectiveness of zeroth biases in the rest of this subsection.
        Aside from effective gradient sparsity, implicit adversarial robustness and flatness can potentially be achieved by reducing norms of a matrix or misaligning gradients with that matrix in the term brought by our modification.
        To eliminate the first, the already discovered phenomenon of parameter growth is exploited. The latter is handled in later subsections.
    \item Eliminating the latter, we discover another phenomenon in ViT and T5 that most non-zero eigenvalues of the matrix differ by at most 100 times for most of the time, leaving only two possibilities: adversarial robustness is achieved only by (effective) gradient sparsity, or back-propagated gradients are totally lost. In \cref{sec:spectral_init}, we prove the emergence of this spectral concentration at initialization, exploiting modern initialization techniques. A drastic architectural modification, wide MLPs, is proposed to fill a gap of theories;
    \item In \cref{sec:spectral_training} we discuss spectral concentration's maintenance and re-emergence in latter stochastic training, applying random matrix theory by extracting two large random matrices from the updates to the weight matrices;
    \item In \cref{sec:refine}, $\dbmlp$ is refined following theoretical insights of \cref{sec:theory};
    \item In \cref{sec:v_experiments}, we conduct experiments 1) to show that activation sparsity can be lost but gradient sparsity is stable, and 2) to verify our explanation;
    \item In \cref{sec:p_experiments}, we train modified ViT-Base/16 on ImageNet-1K as well as T5-Base on C4 from scratch to examine the effectiveness of our modifications in the sense of sparsity and further verify our explanation. We also finetune trained weights for sparsity after plugging the two modifications to demonstrate a cheaper way to become sparser;
    \item In \cref{sec:t_experiments}, assumptions made in \cref{sec:three_elements} and \cref{sec:spectral_training} are examined empirically.
\end{itemize}
To summarize our contribution, we propose the notions of gradient sparsity, effective gradient sparsity and implicit adversarial robustness. 
We explain activation sparsity with flat minima and implicit adversarial robustness, and propose two architectural theoretically guided modifications to improve sparsity.
The benefits of LayerNorm layers and the practice of excluding their parameters from weight decay are emphasized in theories and experiments.
To our knowledge, we are the first to utilize random matrix theory to reason about inductive biases of stochastic training. As a result, the modeling of stochastic gradient noises (SGN) is very direct and avoids any debatable SGN modeling like Gaussian or $\sas$ models.
Experiments show that our explanation is more applicable than other potential ones and our modification further improves activation sparsity by a large margin.

\jmlronly{Codes of experiments can be found on GitHub\footnote{\coderepo{}} and raw logs recorded by TensorBoard are available on Huggingface\footnote{\logrepo{}}.}

%% file: related/related.tex
\section{Related Works}

In this section, we list works on the same topic as ours. \cref{sec:preliminary} contains works on different topics that our explanation depends on, we omit their details for simplicity here.

In our point of view, the research of activation sparsity in MLP modules starts from the discovery of the relation between MLP and knowledge gained during training. \citet{mlp_as_database} first rewrite MLPs in Transformers into an unnormalized attention mechanism where queries are inputs to the MLP block while keys and values are provided by the first and second weight matrices instead of inputs. So MLP blocks are key-value memories. 
\citet{knowledge_neurons} push forward by detecting how each key-value pair is related to each question exploiting activation magnitudes as well as their gradients, and providing a method to surgically manipulate answers for individual questions in Q\&A tasks. These works reorient research attention back to MLPs, which are previously shadowed by self-attention.

Recently, comprehensive experiments conducted by \citet{observation} demonstrate activation sparsity in MLPs is a prevailing phenomenon in various architectures and on various CV and NLP tasks. 
\citet{observation} also eliminate alternative explanations and attribute activation sparsity solely to training dynamics. 
The authors explain the sparsity theoretically with initialization and by calculating gradients, but their explanation is restricted to the last layer and the first step because in later steps the independence between weights and samples required by the explanation is broken. 
They also discover that some activation functions, such as $\tanh$, hinder the sparsity \citep[see][Fig B.3(c)]{observation}, but did not elaborate on it. 
Compared to their explanations, our explanation applies to all layers and large steps, and accounts for the activation functions' critical role in activation sparsity.

Following empirical discoveries by \citet{observation}, \citet{sharpness_aware} show that sharpness-aware (SA) optimization has a stronger bias toward activation sparsity. 
They explain theoretically by calculating gradients and finding that SA optimization imposes in gradients a component toward reducing norms of activations. However, their explanation is still conducted on shallow 2-layer pure MLPs and requires SA optimization, which is not included in standard training practice. Nevertheless, this explanation hints at the role of flatness in the emergence of activation sparsity. Inspired by them, we explain \emph{deep} networks trained by standard SGD or other stochastic trainers by substituting flat minima for SA optimization.

A more recent work by \citet{from_noises} holds a point that sparsity is a resistance to noises. However, noises are manually imposed and not included in standard data augmentations. We substitute gradient noise from SGD or other stochastic optimizers for them.
\citet{large_step} prove sparsity on 2-layer diagonal MLPs and conjecture similar things to happen in more general networks. Both works hint at the relation between noises (Gaussian sample noises and stochastic gradient noises) and activation sparsity, also leading to the flatness bias of stochastic optimization.

\citet{adversarial_of_moe} study the adversarial robustness of Mixture of Experts (MoE) models brought by architecture-imposed sparsity. They inspire us to relate sparsity with adversarial robustness, although we do it reversely. It is the major inspiration for our results.

To sum up, existing discoveries hint at the relation between activation sparsity and noises, flatness and activation functions but they are still restricted to shallow layers, small steps and special training. Inspired by them and filling their gaps, our explanation applies to deep networks and large training steps, and sticks to standard training procedures.

Although not devoting much to explaining the emergence of activation sparsity in CNNs, \citet{exploit_sparsity_in_CNN} boost activation sparsity through Hoyer regularization\citep{hoyer} and a new activation function FATReLU that uses dynamic thresholds between activation and deactivation. They also design algorithms to exploit this sparsity, leading to $\ge 1.75\mathrm{x}$ speedup in CNN's inference. Compared to their sparsity encouragement method that requires well-designed procedures to select thresholds, hyperparameters for our theoretically induced modifications can be easily selected. The discontinuity of FATReLU also bothers training from scratch\citep{exploit_sparsity_in_CNN}, while we recommend applying our modifications from scratch to enjoy better sparsity and additionally smaller \emph{training} costs. Regarding exploitation, we consider it out of the manuscript's scope.
\citet{L1_sparsity} encourages activation sparsity in CNN by explicit $L_1$ regularization. We intend to investigate the emergence of activation sparsity from implicit regularization as demonstrated by \citet{observation}, so we solely rely on implicit regularization boosted by modifications. Nevertheless, our methods are architecturally orthogonal and we believe applying both together can further boost activation sparsity.

There are other works that are not devoted to activation sparsity but are related. \citet{sparse_symbol} formulate, with Shapley value, and prove that there are sparse ``symbols'' as groups of patches that are the only major contributors to the output of any well-trained and masking-robust AIs. They provide a sparsity independent of training dynamics. Their theory focuses on symbols and sparsity in inputs, which is inherently different from ours.

In Primer \citep{primer}, several architectural changes given by architecture searching include a new activation function Squared-ReLU. In this work, we induce a similar squared $\relu$ activation but with the non-zero part shifted left and use it to guide the search for flat minima and gradient/activation sparsity. \cite{primer} demonstrate impressive improvements of Squared-ReLU in both ablation and addition experiments, and our work provides a potential explanation for this improvement.

%% file: preliminary/preliminary.tex
\section{Preliminary}\label{sec:preliminary}

Before starting, we concisely introduce elements that build our theory. We start from basic symbols and concepts of multiple kinds of sparsity and go through bias toward flat minima, adversarial robustness and random matrix theory that deals with large random matrices.

\subsection{Notation}

We use $x, y$ to indicate samples and labels, respectively, and $\dataset = \set{(x_s, y_s): 1 \le s \le \size{\dataset}, s \in \nats^+}$ for the training data set. We restrict focus on activation sparsity on \emph{training} samples, so we do not introduce symbols for the testing set in formal analyses. 
Lowercase letters such as $f$ are used to indicate unparameterized models and those subscripted by parameter $\theta$, such as $f_\theta$, indicate parameterized ones.
Classification task is assumed, so $f_\theta$ is assumed to output a probability distribution over the label space, and the probability of label $y$ is denoted by $f_\theta(y \mid x)$ or $f(y \mid \theta, x)$, with $\sum_y f(y \mid \theta, x) = 1$.
Let $\loss$ be a certain loss, then $\loss(f_\theta, (x_s, y_s))$ is the loss of $f_\theta$ on sample $(x_s, y_s)$. 
If $l$ is a scalar function of matrix or vector $X$, we use $\derivatives{l}{X} \defeq \begin{bmatrix} \derivatives{l}{X_{i, j}} \end{bmatrix}_{i, j}$ to denote the partial derivatives of $l$ w.r.t. $X$'s entries collected in the same shape, while $\nabla_{\vectorize{X}} l(X)$ is used after flattening. 
$\diag{x}$ is used to transform a vector into a diagonal matrix.
We use subscription to indicate substructures of matrices, i.e., $X_i$ is the $i$-th row of matrix $X$ while $X_{\cdot, j}$ is its $j$-th column. Nevertheless, sometimes columns are frequently referred to so we use the subscripted lowercase matrix name $x_j$ to indicate the $j$-th column $X_{\cdot, j}$ of matrix $X$.
The Hessian of scalar function $l$ w.r.t. to vectorized parameter $\theta$ is $\hessian \defeq \begin{bmatrix} \frac{\partial^2 l}{\partial \theta_i \partial \theta_j} \end{bmatrix}_{i, j}$. Throughout this work the scalar function in the Hessian will be the empirical loss $\ex[(X, Y) \sim \dataset]{\loss(f_\theta, (X, Y))}$. 
We assume models are stacked with similar modules, each of which contains at least one MLP block. The layer index is indicated with superscription $l$. 
We assume the hidden features of models on \emph{single} samples are single (column) vectors/tokens $x^l$, as in pure MLPs, or stacked tokens, i.e., matrix $X^l$, where the $j$-th column $x^l_j \defeq X^l_{\cdot, j}$ is the vector for the $j$-th token in Transformers.
A vanilla $\mlp^l$ block in the $l$-th layer contains two linear layers $\mlp^l_K$ and $\mlp^l_V$, each with a learnable matrix $K^l \in \reals^{n^l \times d^l}$ or $V^l \in \reals^{d^l \times n^l}$ of weights, and a bias $b_K^l \in \reals^{n^l}$ or $b_V^l \in \reals^{d^l}$. $\mlp^l_K$ has a non-linear activation function $\activation$ while $\mlp^l_V$ does not. 
Following the terminology of \citet{knowledge_neurons} and \citet{mlp_as_database}, $\mlp_K^l$ is called a ``key'' layer while $\mlp_V^l$ is called a ``value'' layer, whose weights are called ``keys'' or ``values'', respectively. 
They compute the next hidden feature as the following:
\begin{align}
    \Alpha^l \defeq& \mlp_K^l\left(X^{l-1}\right)
    \defeq \activation\left(K^l X^{l-1} + b_K^l\right),\\
    Z^{l} \defeq&  \mlp_V^l\left(\Alpha^l\right)
    \defeq V^l \Alpha^l + b_V^l,\\
    Z^{l} =& \mlp^l\left(X^{l-1}\right) \defeq \mlp_V^l\left(\mlp_K^l\left(X^{l-1}\right)\right).
\end{align}
Note that here $\Alpha$ is capitalized ``$\alpha$'' that stands for ``activation'' instead of capitalized ``$a$''.
For non-stacked tokens $x^{l-1}$, the above equations simplify to
\begin{align}
    \alpha^l \defeq& \mlp_K^l\left(x^{l-1}\right) 
    \defeq \activation\left(K^l x^{l-1} + b_K^l\right)
    =       \begin{bmatrix} \activation\left(\inner{K^l_i}{x^{l-1}} + \left(b_K^l\right)_i \right)\end{bmatrix}_i \in \reals^{n^l},\\
    z^{l} \defeq&    \mlp_V^l\left(\alpha^{l}\right)
    \defeq \activation\left(V^l \alpha^l + b_V^l\right)
    =       \begin{bmatrix} \activation\left(\inner{V^l_i}{\alpha^l} + \left(b_V^l\right)_i \right)\end{bmatrix}_i \in \reals^{d^l},
\end{align}
where $K^l_i, V^l_i$ are the $i$-th rows of weights $K^l, V^l$. In above equations, $\Alpha^l$ and $\alpha^l$ are called the activation pattern of $l$-th $\mlp$ block. Usually, the shapes of $K^l$s and $V^l$s are respectively identical across different $\mlp$ blocks so we drop superscriptions $l$ in $n^l$ and $d^l$.
In \cref{sec:discussion}, we will investigate the implications of \cref{theorem:main} under a notion of ``pure'' MLP. Our definition of it is models where weight matrices are updated by only one token during backward propagation. Under this definition, the most primitive fully connected network (possibly with residual connections and normalizations) is pure MLP, while CNNs, Transformers and MLP-Mixers are not because hidden features consist of multiple tokens and each MLP block must process all of them.

During proofs, we frequently use the properties of matrix trace due to its connection with (elementwise) norms of vectors and matrices. Unless explicitly pointed out, $\norm{\cdot}_p$ for vectors is $L_p$ norm, and $\norm{\cdot}_q$ for matrices is Schatten $q$-norm, i.e., $L_q$ norm of singular values. In the main text, only real matrices are involved and only $L_2$ matrix norm is used for matrices. So particularly, elementwise $L_2$ norm for matrix, or Frobenius norm, of vector $x$ or matrix $X$ can be computed by trace:
\begin{align}
    \trace{x^\transpose x} = \trace{x x^\transpose} = \norm{x}_2^2,
    \norm{X}_2^2 = \trace{X^\transpose X} = \trace{X X^\transpose} = \norm{X}_F^2,
\end{align}
because $\trace{A^T B} = \sum_{i, j} A_{i, j} B_{i, j}.$
This argument also indicates that elementwise $L_2$ norm coincides with Schatten 2-norm by noting $X^\transpose X$'s eigenvalues are squared singular values of $X$ and are summed by the trace. 
Therefore, in later proofs we use trace to express $L_2$ norms and use a bunch of trace properties, for example its linearity \arxivonly{
\begin{align}
    \trace{a A + b B} = a\trace{A} + b \trace{B}
\end{align}
}and the famous cyclic property with matrix products
\begin{align}
    \trace{A B C} = \trace{C A B}.
\end{align}
When trace is combined with Hadamard product $\hadamard$, i.e., elementwise shape-keeping product, there is
\begin{align}
    \trace{A^\transpose (B \hadamard C)} = \sum_{i, j} A_{i, j} B_{i, j} C_{i, j} = \trace{B^\transpose (A \hadamard C)} = \trace{C^\transpose (A \hadamard B)}.
\end{align}
When $A, B$ are real symmetric positive semi-definite, the trace of their product can be bounded by their minimum eigenvalue and individual traces:
\begin{align}
    \trace{A B} \ge \lambda_{\min}(A) \trace{B}, \label{eq:trace_product_and_eigenvalue}
\end{align}
where $\lambda_{\min}(\cdot)$ indicates the smallest eigenvalue of a matrix.
Finally, Hadamard product can be related to matrix products with diagonal matrices:
\begin{align}
    \diag{x} A \diag{y} = \left(x y^T\right) \hadamard A \label{eq:hadamard_and_diagonal}.
\end{align}
This is because LHS scales rows of $A$ first and then scales its columns, while RHS computes the scaling factors for elements in $A$ first and then scaling them in the Hadamard product.

\subsection{Sparsity, Activation Sparsity and Gradient Sparsity}\label{sec:preliminary:sparsity}

In the most abstract sense, sparsity is the situation where most elements in a set are zero.
Many kinds of sparsity exist in neural networks such as weight sparsity, dead neuron, attention sparsity and activation sparsity \citep{sparsity_handbook}. 
The most static one is weight sparsity, where many elements in weight matrices are zero \citep{stochastic_collapse}. 
More dynamic ones found in hidden features are of more interest because they lead to potential pruning but without damaging model capacity. An example of them is attention sparsity found in attention maps of Transformers that many previous sparsity works \citep{attention_sparsity_1,attention_sparsity_2,attention_sparsity_3} focus on.

The activation sparsity of our focus is the phenomenon where activations in MLPs of trained models contain only a few non-zero elements as in \cref{def:activation_sparsity}
\begin{definition}[Activation Sparsity]\label{def:activation_sparsity}
    Recall the definition of activation pattern
    \begin{align}
        \alpha^l \defeq& \mlp_K^l\left(x^{l-1}\right)
        \defeq \activation\left(K^l x^{l-1} + b_K^{l}\right)
        =   \begin{bmatrix}
            \activation\left(\inner{K^l_i}{x^{l-1}} + \left(b_K^l\right)_i\right)
        \end{bmatrix}_i \in \reals^n.
    \end{align}
    For uniformity, define activation sparsity pattern of $\mlp^l$ on sample $X^0$ at token $x^{l-1}$  to be the activation pattern $\alpha^l$.

    Activation sparsity is the phenomenon that most entries in activation sparsity patterns are zero for most samples and tokens.
    For mathematical convenience, squared $L_2$ norm $\norm{\alpha^l}_2^2$ is used as a proxy for activation sparsity.
\end{definition}
Activation sparsity is observed in pure MLPs, CNNs such as ResNet, MLP blocks in Transformers, and channel mixing blocks in MLP-Mixers \citep{observation}. Note that activation sparsity cannot be explained by dead neurons or extreme weight sparsity, because every neuron has its activation for some sample \citep{observation}. Also be noted that it is not done simply by weight decay \citep{sharpness_aware} but by moving pre-activations towards the negative direction in $\relu$ networks. Activation sparsity potentially allows \emph{dynamic} aggressive neuron pruning during inference with zero accuracy reduction \citep{observation}. 

Similar sparsity in forward propagation can be imposed by architectural design such as Mixture of Experts (MoE) \citep{moe_vit,switch_transformer,moe_DG}, where each block is equipped with multiple MLPs and tokens are dynamically and sparsely routed to only 1 or 2 of them. However, MoE has to deal with discrete routing that unstables the training, and emergent imbalanced routing that makes only one expert live and learn. Emergent activation sparsity has no such concerns and can also emerge within each of the experts, and thus still deserves attention even combined with MoE. 

For common activation functions like $\relu$, being activated coincides for itself and its derivative, i.e., activation is zero if and only if its derivative is zero. For other activation functions like GELU and Leaky $\relu$, derivatives in the ``suppressed'' state are also coincidentally small. These coincidences make us wonder if activation sparsity is purely direct activation and if it is actually gradient sparsity to some extent.
Following this insight, we further propose gradient sparsity as a source of activation sparsity. 
Gradient sparsity is that most elements in the derivatives $\activation'(K^l x^{l-1} + b^l)$ of activations are zero as in \cref{def:gradient_sparsity}.
\begin{definition}[Gradient Sparsity]\label{def:gradient_sparsity}
    Let $\gamma^l$ to be the entrywise derivatives of activation in $\mlp^l$ with activation function $\activation$, i.e., 
    \begin{align}
        \gamma^l 
        \defeq& \activation'\left(K^l x^{l-1} + b_K^l\right)
        =      \begin{bmatrix}
            \activation'\left(\inner{K^l_i}{x^{l-1}} + \left(b_K^l\right)_i\right)
        \end{bmatrix}_i \in \reals^n.
    \end{align}
    If the model uses stacked hidden features of $k$ tokens, then let matrix $\Gamma^l \in \reals^{n \times k}$ be the stacked version.
    Define gradient sparsity pattern of $\mlp^l$ on sample $X^0$ at token $x^{l-1}$ to be $\gamma^l$.

    Gradient sparsity is the phenomenon where most entries in gradient sparsity pattern $\gamma^l$ are zero for most samples and tokens. 
    For mathematical convenience, squared $L_2$ norm $\norm{\gamma^l}_2^2$ is used as a proxy for gradient sparsity.
\end{definition}
\begin{remark}\label{remark:L2_and_L0}
    We model gradient sparsity by $L_2$ norm for mathematical convenience, which has weaker direct relations to sparsity than $L_1$ norm or $L_0$ ``norm''.
    However, when $\relu$ is used, its derivative is a $0$-$1$ indicator for activations and derivatives, so the squared $L_2$ norm of activation derivatives is exactly $L_0$ norm of activations as well as derivatives.
    Our modified activation function $\jrelu$ defined in later \cref{eq:jsrelu} shares the same style of $0$-$1$ jump discontinuity in its derivative at zero, but the derivative increases if the pre-activation further increases. The connection between $L_2$ and $L_0$ norm is weakened by this increase, but when the $L_2$ norm decreases so that it is small enough, the connection will be stronger and one must improve sparsity to achieve a smaller $L_2$ norm after squeezing all pre-activations to near zero.
\end{remark}
For activation functions like $\relu$, gradient sparsity coincides with activation sparsity and gives birth to the latter, and the coincidence is why only activation sparsity is proposed in the previous empirical work by \citet{observation}.
We argue gradient sparsity is more essential and stable than direct activation sparsity. Our theoretical analyses will show that gradient sparsity is a stable cause of activation sparsity through their coincidences, although there is unstable but direct implicit regularization on activation sparsity. Experiments for validation in \cref{sec:v_experiments} show activation can be manipulated to be dense by gradient sparsity and experiments in \cref{sec:t_exp:contribution} show that direct activation sparsity is weak compared to gradient sparsity at least in deep layers. 

Here, we argue the practical importance of gradient sparsity over activation sparsity. Interestingly, activation sparsity itself can be generalized and thus leads to gradient sparsity. To prune most neurons, one actually does not need exact activation sparsity where most activations are zero, but only the fact that most activations have the \emph{same} (possibly non-zero) value that can be \emph{known a priori}. If so, activations can be shifted by the a priori most-likely activation to obtain the exact activation sparsity followed by pruning, and adding the sum of key vectors multiplied by that value back can compensate for the shift. If certain regularities of the activation function (for example, being monotonically increasing like $\relu$) can be assumed, then these most-likely activations must reside in a contiguous interval in the activation function's domain, which leads to gradient sparsity. 
The conversed version of this argument also shows that gradient sparsity leads to pruning during inference.
Therefore, in addition to the fact that both of them are sufficient conditions, gradient sparsity is much closer to a necessary condition for massive pruning than activation sparsity.
Moreover, gradient sparsity also allows aggressive neuron pruning during \emph{training}, which is beneficial especially for academia in the era of large models.
Therefore, more attention should be paid to gradient sparsity.

A more generalized version, effective gradient sparsity, is the sparsity defined on the gradient w.r.t. pre-activations, or equivalently the Hadamard product, or entrywise product, between the gradient sparsity pattern $\gamma^l$ and the gradient w.r.t. to the activation. Exactly how and why it is defined can be found in \cref{def:effective_gradient_sparsity} to avoid confusion. Using this notion of sparsity, better theoretical results can be obtained. Making these theoretical shifts practically meaningful, gradient w.r.t. activation allows further pruning because neurons with near-zero gradients w.r.t. to themselves 1) have little contribution to gradients back propagated to shallower layers and 2) have little influence on the output during forward propagation if activation is also small.
As one shall see in \cref{sec:gradients}, from a theoretical and interpretational view, effective gradient sparsity is also what $\mlp$ blocks try to memorize in their key matrices.

\subsection{Empirically Measuring Sparsity}

\citet{observation} utilize the percentage of nonzeros in activations, or equivalently $L_0$ ``norm'' of $\alpha^l$s, on testing samples as a simple measurement of sparsity in $\relu$ network. We adopt a similar measure but it is also conducted on \emph{training} samples in each batch and we observe its revolution during the entire training. We additionally observe training sparsity in order to see how potentially well gradient sparsity reduces training cost, and we restrict samples to those in the current batch because in practical training, samples outside the batch will not be used and are irrelevant to actual training costs. Percentages are further averaged across layers and integrated across steps to approximate the overall reduction in MLPs' training costs.

\subsection{Flat Minima and Stochastic Gradient Noise}\label{sec:flat_minima}

One of the accounts for good generalization unexpected by traditional statistical learning theory in deep networks is the implicit regularization introduced by architectures, training procedures, etc. 
One of the most considered is the inductive bias toward flat minima, i.e., loss minima given by SGD are very likely to be flat. 
This flatness in the training loss landscape indicates that the loss landscape near the minima will not rise acutely due to distribution shift and explains small loss increase and good generalization in testing \citep{flat_minima}. 

Bias toward flat minima is usually considered due to stochastic gradient noises (SGN) introduced by SGD or other stochastic optimizers, which drives the parameter to escape from sharp minima \citep{escape}. Although SGD is usually considered to have a stronger flatness bias, parameters optimized by other adaptive optimizers such as Adam still escape sharp minima, only in a slower manner \citep{escape}.

In works that study SGN and flat minima like those by \citet{alpha_stable} and \citet{escape}, the updates of SGD at step $t$ are often written as 
\begin{align}
    \theta^{t+1} = \theta^{t} - \eta \nabla_{\theta} \loss(\theta) + \eta U^t,
\end{align}
where $\loss(\theta) \defeq \ex[(X, Y)]{\loss(f_\theta, (X, Y))}$ is the full-batch gradient, and noise term $U^t \defeq \nabla_\theta \loss(\theta) - \frac{1}{\size{B_t}} \sum_{s \in B_t} \nabla_{\theta} \loss(\theta, (x_s, y_s))$ is the difference between full-batch gradient and the gradient provided by the current batch $B_t$. If samples in batches are uniformly sampled, the expected mini-batch gradient $\sum_{s \in B_t} \nabla_{\theta} \loss(\theta, (x_s, y_s))$ is naively $\nabla_\theta \loss(\theta)$ and the noise $U^{t}$ is centered by definition. $U^t$ is previously modeled by Gaussian distribution as a result of the Central Limit Theorem. Recent works \citep{alpha_stable,escape} argue that it should better be modeled with symmetric $\alpha$-stable ($\sas$) distribution based on Generalized Central Limit Theorem where finite variance is not necessary. Under this model, the noise norm is long-tailed and the expected norm can be very large since an $\sas$ distribution has infinite variance if it is not Gaussian. 

In this work we only rely on the empirical and theoretical results that parameters are optimized toward flat minima. Using random matrix theory, we are allowed to model SGN in the most direct way without relying on Gaussian or $\sas$ distribution, but by assuming within-batch independence as well as norm and anisotropy bounds on gradient and feature vectors.

Following theoretical works on information bottlenecks \citep{ib_disentangle,chaudhari_stochastic,weight_information_bottleneck}, we start from the nuclear norm of the Hessian at a minimum to measure how flat the minimum is. When a local minimum is reached the Hessian is real symmetric and positive semi-definite, so the nuclear norm equals to its trace, i.e.
\begin{align}
    \norm{\hessian}_* = \trace{\sqrt{\hessian^* \hessian}} = \sum_i \abs{\lambda_i} = \sum_i \lambda_i= \trace{\hessian} 
    % = \trace{\begin{bmatrix}
        % \frac{\partial^2 \ex[(X, Y) \sim \dataset]{\loss(f_\theta, (X, Y))}}{\partial \theta_i \partial \theta_j}
    % \end{bmatrix}_{i, j}} 
    ,
\end{align}
where $\lambda_i$ is the $i$-th largest eigenvalues of $\hessian$. For mathematical convenience, the trace $\trace{\hessian}$ is actually used to measure flatness, as \citet{anti_PGD}. We assume as \citet{chaudhari_stochastic}, \citet{weight_information_bottleneck} and \citet{ib_disentangle} that this trace is suppressed during stochastic training.
For non-minima, there is $\norm{\hessian}_* \ge \trace{\hessian}$, so the explanation will be still meaningful at non-minima if sparsity lowerbounds $\trace{\hessian}$, as we will prove in the following sections.

\subsection{Random Matrix Theory and Marchenko-Pastur Distribution}

The latter half of our theory explaining sparsity's tendency during stochastic training is based on spectral analysis on sample covariance matrix of large random matrices, which are the main focus of random matrix theory (RMT). Particularly about the most classic setting, consider a sequence of random matrices $\set{X^p \in \complexes^{p \times n}}_{p}$ with size increasing to infinity, where all entries in the sequence are centered, standardized and I.I.D. sampled, and $n = n(p) = \Theta(p)$ increases linearly with $p$. The sample covariance of $X^p$ is $S^p \defeq \frac{1}{n} X^p \left(X^p\right)^\transpose$. To measure their spectral properties, define the empirical spectral distribution $F^{S^p}(x) \defeq \frac{1}{p} \sum_{i} \indic{\lambda_i\left(S^p\right) \le x}$ and corresponding density $f^{S^p}$ of eigenvalues for each matrix $S^p$. Note that $F^{S^p}$ is a random variable as $S^p$'s function, but \citet{mp_original} prove that when $p$ goes to infinity, $F^{S^p}$ converges to a non-random distribution later named as Marchenko-Pastur distribution, formally stated in \cref{theorem:singular_value_of_product_of_random_matrices}.

\begin{theorem}[Marchenko–Pastur distribution \citep{mp_distribution}]\label{theorem:singular_value_of_product_of_random_matrices}
    Let $X_{i, j}$, $1 \le i \le p, 1 \le j \le n$, be I.I.D. complex random variables with $\ex{X_{i, j}} = 0$ and $\ex{X_{i, j}^2} = 1$. Let $X^p \defeq [X_{i, j}]_{i, j} \in \reals^{p \times n}$ be the matrix comprised of these random variables. Let $\lambda_k(\cdot)$ be the $k$-th largest eigenvalues of the symmetric matrix 
    \begin{align}
        S^p \defeq \frac{1}{d} X^p \left(X^p\right)^\transpose,
    \end{align}
    and define its empirical spectral distribution (ESD) by
    \begin{align}
        F_p(x) = F^{S^p}(x) = \frac{1}{p} \sum_{k=1}^{p} \indic{\lambda_k\left(S^p\right) \le x} \label{eq:def:mp_density}.
    \end{align}
    and corresponding density $f_p = f^{S^p}$.

    When $p, n\to \infty$ with $p / n \defto c$, $f^p \to f$ in probability, where density $f$ of Marchenko-Pastur distribution is defined by
    \begin{align}
        f(x) = \frac{1}{2 x c \pi} \sqrt{(b - x)(x - a)} \indic{a \le x \le b} + \indic{c \ge 1} \cdot \left(1 - \frac{1}{c}\right) \delta(x) \label{eq:mp_density},
    \end{align} 
    $a = (1 - \sqrt{c})^2, b = (1 + \sqrt{c})^2$, and $\delta$ is the Dirac delta function.
\end{theorem}

We will discover $X^p \left(X^p\right)$-like matrices products in \cref{sec:spectral_init} and \cref{sec:spectral_training}, and apply generalized \cref{theorem:singular_value_of_product_of_random_matrices} to them, for example at initialization and in stochastic additive updates, given that hidden features are of several hundreds dimensions and there are millions of steps involved in training. We consider large random matrices as a powerful tool in analyzing behaviors of deep models given their high dimension and randomness.

\subsection{Adversarial Robustness}

It is well known that deep models are prone to adversarial attacks, i.e., small changes in inputs may result in large changes in output or classification.
Most works on adversarial attack and robustness consider perturbations at the beginning of the network. Given the layered structure of deep networks, we naturally consider perturbations to hidden intermediate features, simulated by perturbations on parameters in shallower layers. 
We refer to them as implicit adversarial attacks and implicit adversarial robustness and see (effective) gradient sparsity as a necessary step in resistance of them. 
Following works that connect flat minima or sparsity to adversarial robustness \citep{adversarial_and_flat_minima,adversarial_of_moe}, we use squared $L_2$ norm $\norm{\derivatives{\loss(f_\theta, (x_s, y_s))}{x^l}}_2^2$ of gradients w.r.t. hidden features on \emph{individual samples} to measure adversarial robustness for mathematical convenience.

%% file: theory/theory.tex
\section{Gradiential Explanation of Activation Sparsity}\label{sec:theory}

\IfFileExists{theorems.aux}{\input{theorems.aux}}{}

Before starting the formal analysis of emergence of activation sparsity, we provide intuitive framework of our explanation. Slight architectural changes in MLPs, which ease theoretical analyses, are also already available under this intuitive illustration. Another radical change is proposed after more detailed analyses.

\input{theory/illustration.tex}

\input{theory/analysis.tex}

%% file: theory/illustration.tex
\subsection{Illustration}\label{sec:illustration}

Deep networks are built by stacking. Between layers, hidden intermediate features are generated and passed to deeper parts of the network. 
The first important intuition is that the output of shallower layers can be seen as the input of deeper layers, which allows us to see the \emph{parameters} of shallower layers as a part of \emph{inputs} to deeper layers and apply notions previously defined at the real inputs $x^0$, for example, the notion of (implicit) adversarial robustness.

For deeper layers to have adversarial robustness, a way is to reduce the norm of activations' derivatives, because the gradients w.r.t. $x^{l-1}$ is calculated by multiplying the gradients by $x^l$ with derivatives of the activations and weights. 
So by suppressing the gradient norm of activations to the extreme, i.e., promoting gradient sparsity, adversarial robustness increases for sure.
Note that good sparsity and implicit adversarial robustness do not necessarily hinder approximability because they only imply local flatness and naturally distinct samples differ so much that there are multiple flat regions between them in the output-input landscape, i.e., approximability may be achieved by drastically changing which neurons are activated \citep{parameter_growth}.

Without explicit adversarial training, adversarial samples may come from perturbations in shallower layers.
The inputs to deeper layers are not static and even not stable, because gradient noises in stochastic training are driving the parameters to run around randomly, which gives birth to adversarial robustness.
Note that the idea of wrapping perturbations in weights into hidden features has been used by \citet{weight_information_bottleneck}, where this technique is implicitly applied to define representation information bottleneck, but they did not explicitly consider the perturbations to shallow parameters as those to hidden features.
This explains how deep neural networks gain implicit adversarial robustness and then gradient sparsity from stochastic optimization in standard training.

From a more static point of view, i.e., considering the parameter after training, we model the perturbations in a virtual way as illustrated in \cref{fig:illustration}.
\begin{figure}
    \centering
    \includegraphics[width=8cm]{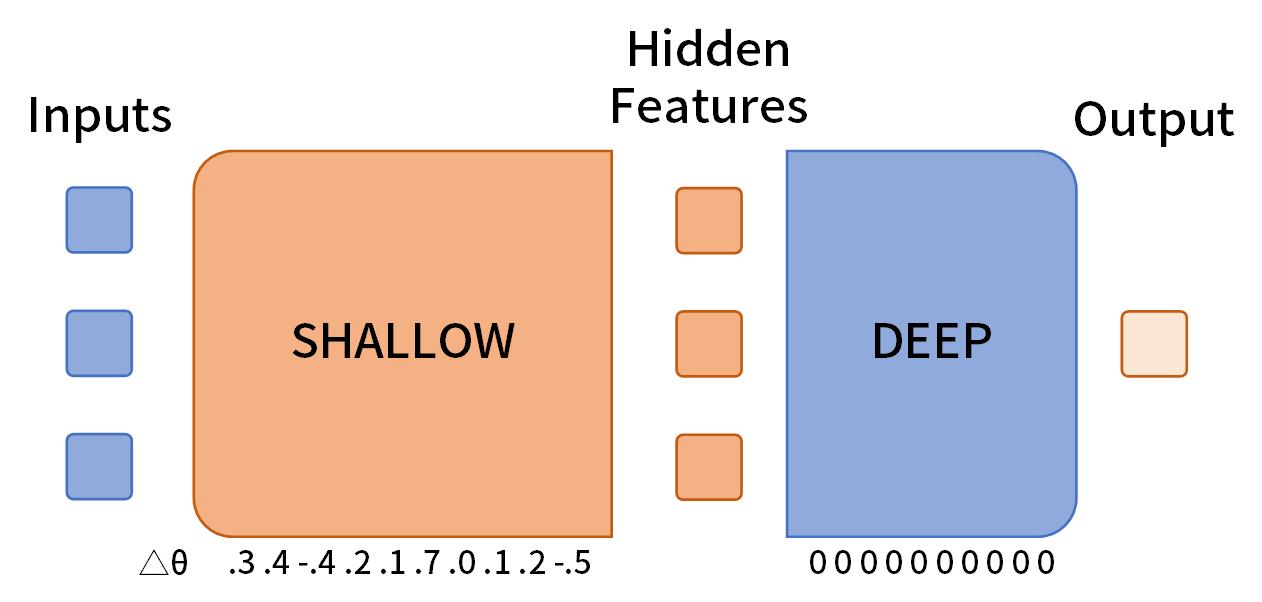}
    \caption{Illustration of virtual perturbations to shallow layers. Given a trained parameter at a flat minimum, if a small change around $0$ in shallow parameters is added, the loss will not drastically increase due to the flatness. These virtual perturbations in shallow parts, their outputs and the final output are illustrated in red, while frozen inputs and deeper parts are in blue. The minor increase in loss, colored by pale red, indicates that the deep layers have to learn adversarial robustness w.r.t. hidden features.}\label{fig:illustration}
\end{figure}
If the trained parameter is a flat minimum, then slight changes in its parameters do not drastically increase loss.
Changes restricted to the shallow layers inherit this property. So parameter perturbations that are reflected in hidden features cause no drastic increase in losses.  As explained before, the frozen deep layers must have learned implicit adversarial robustness to ensure flatness w.r.t. virtually perturbed inputs for deep layers to mitigate these virtual perturbations to hidden features.
This intuition leads to the notion of flat minima and we will follow it in \cref{sec:three_elements} by proving that both implicit adversarial robustness and some sparsity-related inner products are upperbounded by the measure of flatness, i.e., $\trace{\hessian}$.

Following this line, one question to answer is how diverse the perturbations are to hidden features. 
Perturbations to parameters have to go through non-linear operations, which hinders the expressibility of parameter perturbations especially in non-activated neurons. 
To make things worse, \emph{parameter sharing} creates hidden features whose dimension is much larger than the parameter that directly produces it, in contrast with explicit adversarial training where perturbations are directly added to $x^0$ in a full-dimensional way. As a result, the perturbation to parameters creates correlated perturbations in hidden features, which cannot cover explicit adversarial attacks that may perturb in an uncorrelated manner.
We consider the problem from parameter sharing to be extremely sever in CNN because normally a small ($\sim 4 \times 4$) convolution kernel should produce a large feature map. This problem in Transformers is moderate because there are seemingly full-dimensional parameters (e.g. weights and bias in value matrices and MLPs), but tens of tokens are subject to the same parameter so it is low in this token-stacking dimension. Pure MLPs have the smallest problem because its weights and biases are full-dimensional and there is only one token as features, but the perturbations are still subjected to activation functions.

To alleviate these problems through adding \emph{non-shared} parameters \emph{after} activation functions, we propose Doubly Biased MLP ($\dbmlp$) by introducing an extra bias $D^l$ or $d^l$, called zeroth bias (ZB) due to its position, before other operations (so it is after activation functions of previous layers). ZB is a \emph{matrix} of the same shape with hidden features if they are matrices of stacked tokens, i.e.
\begin{align}
    \dbmlp^l(X^{l-1})
    \defeq& \mlp^l(X^{l-1} + D^{l}),
\end{align}
where $D^l$ is the zeroth bias of current $\dbmlp$. For hidden features of single tokens, it writes
\begin{align}
    \dbmlp^l(x^{l-1})
    \defeq& \mlp^l(x^{l-1} + d^l) = \mlp_V^l\left(\activation(K^l (x^{l-1} + d^l) + b^l)\right).
\end{align}
For single-token hidden features, the zeroth bias removes the hindering effect of activation functions on the previous layer's bias $b^{l-1}$, and for matrix hidden features, $D^l$ of the same shape allows full-dimensional perturbations. Note that zeroth biases are added to hidden features, meaning that they share gradients. This property will ease the analyses bridging flatness and implicit adversarial robustness.

Another implication of the above illustration is that activation function matters in gradient and/or activation sparsity, because its derivatives are multiplied when computing gradients w.r.t. hidden features to compute gradient norm that proxies implicit adversarial robustness. 
In \cref{sec:v_experiments}, we design strange activation to differentiate gradient and activation sparsity and show that activation sparsity is lost to reach gradient sparsity defined by the activation function.
On the other hand, one problem of common $\relu$ and $\gelu$ is that their derivatives are almost piecewisely constant for most regions, having difficulties in guiding the second-order search for flat minima. By squaring $\relu$, which coincides with \citet{primer}, the non-constant derivatives of activations provide guidance to sparser neighbors, and $\Omega(x)$-large derivatives for large activations drive the parameters into sparse ones. 
However, in tentative experiments, we found that the derivatives are too small around zero and provide little drive toward sparsity, so we shift the non-zero parts left and give
\begin{align}
    \jrelu(x) \defeq \begin{cases}
        0   &   x < 0,\\
        \frac{1}{2} \left( (x + 1)^2 - 1 \right) & x \ge 0,
    \end{cases} \label{eq:jsrelu}
\end{align}
where ``J'' means jumping across the ignored interval $[0, 1]$ and $\frac{1}{2}$ is used to calibrate derivatives at $x=0$ to $1$, approximating $\relu$'s and $\gelu$'s behaviors near $0$. The visualization of $\relu$, $\textrm{Squared-ReLU}$ and $\jrelu$ can be found in \cref{figure:activations}.

\addvalue{relu}{$\relu$}
\addvalue{relu2}{$\textrm{Squared-ReLU}$}
\addvalue{jsrelu}{$\jrelu$}
\begin{figure}
    \centering
    \foreach \act in {relu,relu2,jsrelu}{
        \begin{subfigure}{0.3\textwidth}
            \includegraphics[width=\textwidth]{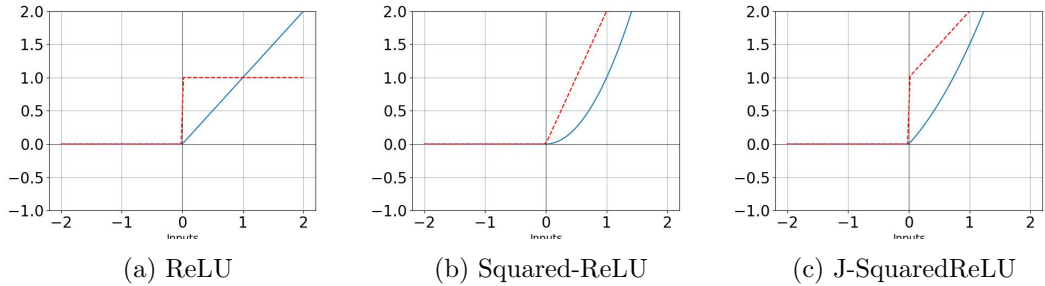}
            \caption{\usevalue{\act}}
        \end{subfigure}
    }
    \caption{Visualization of $\relu$, $\textrm{Squared-ReLU}$ and $\jrelu$. Activation and its derivative are indicated by blue and dashed red lines, respectively.}\label{figure:activations}
\end{figure}

One of the benefits of our architectural modifications is that they are plug-and-play, not only for training from scratch, but also for directly loading weights of vanilla architectures with changes plugged in and finetuning for sparsity, given that $D^l$ and $d^l$ can be initialized to $0$ for any weights, and $\jrelu$ approximates $\relu$ or $\gelu$ to the first order at $0$. Due to their simplicity, they are also likely orthogonal to other potential methods.

Now we have established the framework of our explanation. A formal analysis following this framework, based on $\dbmlp$, is conducted in the following three subsections. We will first calculate gradients in MLP blocks. Starting the analyses, we relate flatness to implicit adversarial robustness, and then implicit adversarial robustness to gradient sparsity in \cref{sec:three_elements}. After that, we exclude alternative sources of implicit robustness and flatness. To give the most general form of results, we assume $\dbmlp$ in later subsections, if $\mlp$ is not explicitly pointed out. To obtain results for $\mlp$, simply remove the ZB-related term (usually the first term) in inequalities.

%% file: theory/analysis.tex
\input{theory/gradients.tex}

\subsection{Flat Minima, Implicit Adversarial Robustness and Effective Gradient Sparsity}\label{sec:three_elements}

\NewDocumentCommand{\ce}{}{\loss_{\mathrm{CE}}}
As discussed in \cref{sec:illustration}, we start from $\trace{\hessian}$ and relate it to implicit adversarial robustness w.r.t. some hidden feature $x^l$, or $\ex[(X, Y)]{\norm{\derivatives{\loss}{x^l}}_2^2}$. 
Cross Entropy loss is assumed under classification tasks. Note that aside from explicit classification tasks, next-token classification and Cross Entropy loss form the basis for many self-supervision objectives in NLP pretraining, including causal language modeling and masking language modeling. Therefore this assumption can be applied across broad practical scenarios. Assuming $\loss = \ce = -\log f(y \mid \theta, x)$, the Hessian writes
\begin{align}
    \nabla_\theta^2 \ex[(X, Y)]{-\log f(Y \mid \theta, X)}
    =& -\ex[(X, Y)]{\nabla_\theta^2 \log f(Y \mid \theta, X)},
\end{align}
which, together with $\ex[(X, Y)]{\norm{\derivatives{\loss}{x^l}}_2^2}$, reminds us of the famous equality of Fisher's Information Matrix, i.e.,
\begin{align}
    \mathcal{I}(\theta)
    = -\ex[X]{\nabla_\theta^2 \log g_\theta(X)}
    = \ex[X]{\left(\derivatives{g_\theta(X)}{\theta}\right)\left(\derivatives{g_\theta(X)}{\theta}\right)^\transpose},
\end{align}
the trace of RHS of which is exactly the expected squared norm of the gradients. So adapting the classic proof of this equality, we connect flatness measured by Hessian trace to the norm of gradients in \cref{lemma:flatness_and_grad_norm}.
\begin{lemma}[Flatness and samplewise gradient norm]
    \label{lemma:flatness_and_grad_norm}
    Assume $f_\theta$ is a neural network parameterized by $\theta$, trained by Cross Entropy loss $\ce$. Given $\hessian$ being the Hessian matrix of loss at $\theta$, there is
    \begin{align}
        \trace{\hessian}
        =   \ex[(X, Y) \sim \dataset]{\norm{\nabla_\theta \ce(\theta, (X, Y))}_2^2} - \ex[(X, Y) \sim \dataset]{\frac{\trace{\nabla^2_\theta f(Y \mid \theta, X)}}{f(Y \mid \theta, X)}}.
    \end{align}
    Further for well learned models, i.e., $f_\theta(Y \mid X) \approx \dataset(Y \mid X)$ for all training samples, there is
    \begin{align}
        \trace{\hessian}
        \approx   \ex[(X, Y) \sim \dataset]{\norm{\nabla_\theta \ce(\theta, (X, Y))}_2^2}.
    \end{align}
\end{lemma} 

The proof of it can be found in \cref{proof:flatness_and_grad_norm}. This lemma invokes a samplewise point of view on Hessian and flatness, which is closer to adversarial attacks because they are added in a samplewise manner. Aside from implication in the context of sparsity, \cref{lemma:flatness_and_grad_norm} indicates that a flat minimum also provides solutions that are locally good for most individual samples. 
We acknowledge that similar results with gradient outer products under abstract losses have been given by \citet{empirical_hessian}, but we believe under Cross Entropy loss the result becomes more direct, intriguing and implicative. 
As an aside, Cross Entropy seems an interesting loss function, for example, \citet{ib_disentangle} rewrite CE loss expected among all training traces into mutual information between training data and final parameters.

After that, we move perturbations from parameters to hidden features in \cref{lemma:grad_norm_and_adversarial}.
\begin{lemma}[Gradient norm and implicit adversarial robustness]\label{lemma:grad_norm_and_adversarial}
    Let $f_\theta$ be a neural network parameterized by $\theta$, and the parameters for the $l$-th layer is $\theta_l$. Let $\dbmlp^l$ be the $l$-th doubly biased MLP in $f_\theta$ whose input is $X^{l-1}$, zeroth bias is $D^l$. Then there is
    \begin{align}
        \norm{\nabla_{\theta_l} \ce(\theta, (x, y))}_2^2 \ge \norm{\nabla_{X^{l-1}} \ce(\theta, (x, y))}_2^2 + \norm{\HXT}_2^2 + \norm{\GAlphaT}_2^2.
    \end{align}
    If $\mlp^l$ is not doubly biased, then the first term in RHS simply disappears.
\end{lemma}
\begin{proof}
    By noticing $\theta_l$ contains at least $K^l, V^l$ and $D^l$, there is
    \begin{align}
        \norm{\nabla_{\theta_l} \ce}_2^2 \ge \norm{\nabla_{D^l} \ce}_2^2 + \norm{\nabla_{K^l} \ce}_2^2 + \norm{\nabla_{V^l} \ce}_2^2.
    \end{align}
    Consider how $X^{l-1}$ is processed in $\dbmlp^l$: It is added with $D^l$ before any other operations. So there is $\nabla_{D^l} \ce = \nabla_{X^{l-1}} \ce = \nabla_{X^{l-1} + D^l} \ce$. 
    Combining \cref{lemma:gradients_with_eff} for the second and the third terms, the lemma follows.
\end{proof}

From the proof of \cref{lemma:grad_norm_and_adversarial} we can see that zeroth biases avoid tedious linear or non-linear operations and drastically ease our analysis. This implies that one can design theoretically oriented architecture that allows easier theoretical analyses.

We then connect implicit adversarial robustness to gradient sparsity.
\begin{lemma}[Implicit adversarial robustness and gradient sparsity]\label{lemma:adversarial_and_sparsity}
    Under the same condition of \cref{lemma:grad_norm_and_adversarial}, together with the assumption that in $\dbmlp^l$ the weight matrix is $K^l \in \reals^{n \times d}$ and $\gamma^l$ is the entrywise derivatives of activations, there is
    \begin{align}
        \norm{\nabla_{x^{l-1}} \ce(\theta, (x, y))}_2^2 = \left(\gamma^l\right)^\transpose \left(\left(K^l (K^l)^\transpose\right) \hadamard M^l\right) \gamma^l = \left(\eta^l\right)^\transpose \kkT \eta^l,
    \end{align}
    if hidden features are single tokens, where $M^l \defeq g^l_K \left(g^l_K\right)^\transpose$ is a symmetric positive semi-definite matrix of rank at most $1$, $\hadamard$ denotes Hadamard product, i.e., entrywise product.
    
    If hidden features are matrices, then there is
    \begin{align}
        \norm{\nabla_{X^{l-1}} \ce(\theta, (x, y))}_2^2 = \trace{\Eta^l \left(\Eta^l\right)^\transpose \kkT}.
    \end{align}
\end{lemma}
\begin{proof}
    By \cref{lemma:gradients_with_eff}, the gradients w.r.t. $d^l$ is 
    \begin{align}
        \derivatives{\ce(\theta, (x, y))}{d^l} = \left(K^l\right)^\transpose \eta^l.
    \end{align}
    Similar to the proof of \cref{lemma:grad_norm_and_adversarial}, $x^{l-1}$ and $d^l$ share the same gradient, so
    \begin{align}
        \nabla_{x^{l-1}} \ce = \derivatives{\ce(\theta, (x, y))}{x^{l-1}} = \left(K^l\right)^\transpose \eta^l.
    \end{align}

    Now we can compute the squared norm of gradients by its relation with trace
    \begin{align}
        \norm{\nabla_{x^{l-1}} \ce}_2^2
        =& \trace{\left(\nabla_{x^{l-1}} \ce\right)^\transpose \nabla_{x^{l-1}} \ce}
        =  \trace{\left(\eta^l\right)^\transpose \kkT \eta^l}\\
        =&  \left(\eta^l\right)^\transpose \kkT \eta^l.
    \end{align}
    To see how $M^l \defeq g^l_{K} \left(g^l_K\right)^\transpose$ emerges, expand the definition of $\eta^l$ and obtain
    \begin{align}
        \norm{\nabla_{x^{l-1}} \ce}_2^2
        =&  \left(\eta^l\right)^\transpose \kkT \eta^l
        =   \left(\diag{g^l_{K}} \gamma^l\right)^\transpose \kkT \left(\diag{g^l_K} \gamma^l\right)\\
        =&  \left(\gamma^l\right)^\transpose \diag{g^l_K} \kkT \diag{g^l_k} \gamma^l\\
        =&  \left(\gamma^l\right)^\transpose \left(g^l_K \left(g^l_K\right)^\transpose \hadamard \kkT \right) \gamma^l
        =  \left(\gamma^l\right)^\transpose \left(M \hadamard \kkT \right) \gamma^l,
    \end{align}
    where the second last step is to apply \cref{eq:hadamard_and_diagonal}.
    Note that $M^l$'s rank is at most $1$.

    When hidden features are matrices, $\norm{\nabla_{X^{l-1}} \ce}_2^2$ sums the gradient norms for all tokens, which leads to
    \begin{align}
        \norm{\nabla_{X^{l-1}}\ce}_2^2
        =&  \sum_i \norm{\nabla_{x^{l-1}_i} \ce}_2^2
        =  \trace{\sum_i \left(\eta^l_i\right)^\transpose \kkT \eta^l_i}\\
        =&  \trace{\left(\sum_i \eta^l_i \left(\eta^l_i\right)^\transpose\right) \kkT}
        =  \trace{\Eta^l \left(\Eta^l\right)^\transpose \kkT}.
    \end{align}
\end{proof}

Combining \cref{lemma:flatness_and_grad_norm}, \cref{lemma:grad_norm_and_adversarial} and \cref{lemma:adversarial_and_sparsity} together, we have the main theorem.
\begin{theorem}[Flatness, implicit adversarial robustness and sparsity]\label{theorem:main}
    Let $f_\theta$ be a well-learned neural network parameterized by $\theta$, trained under Cross Entropy loss $\ce$. Let $\hessian$ be the Hessian matrix w.r.t. parameters at $\theta$. Let $\dbmlp^l$ be the $l$-th doubly biased MLP in $f_\theta$ whose input is $x^{l-1}$. There is
    \begin{align}
        &\trace{\hessian}\\
        \ge& \sum_{l} \left(
            \ex{\norm{\nabla_{X^{l-1}} \ce(\theta, (X, Y))}_2^2} 
            +   \ex{\norm{\HXT}_2^2} + \ex{\norm{\GAlphaT}_2^2}\right)\\
        =& \sum_{i, l} \ex{\left(\eta^l_i\right)^\transpose \kkT \eta^l_i} + \sum_{l} \ex{\norm{\HXT}_2^2} + \sum_l \ex{\norm{\GAlphaT}_2^2}. \label{eq:main}
    \end{align}
    The first term can also be expressed by $\left(\gamma^l_i\right)^\transpose \left(\kkT \hadamard M^l_i\right) \gamma^l_i$, where $M^l$ is a symmetric positive semi-definite matrix of rank at most $1$.
    Further by Schur's Theorem, $\left(K^l \left(K^l\right)^\transpose\right) \hadamard M^l_i$ is also positive semi-definite.

    If vanilla $\mlp$s are used, then the first term in RHS simply disappears.
\end{theorem}

The chained upperbounds connect flatness and implicit adversarial robustness to effective gradient sparsity (the first two terms in Equation \ref{eq:main}) and as well as activation sparsity (the last term in Equation \ref{eq:main}), indicating that both gradient and activation sparsity can be sources of implicit adversarial robustness and flatness. If flatness is achieved, then it is possibly done through (effective) gradient sparsity and activation sparsity. Note that this bound is very tight because $\mlp$s take a large portion of parameters \citep{knowledge_neurons} even in Transformers, so by Cauchy's Interlace Theorem most large eigenvalues of $\hessian$ are retained in the submatrix of $\mlp$ parameters. Therefore to achieve flatness, the terms in \cref{eq:main} must be suppressed.

\subsection{Discussions on \cref{theorem:main}}\label{sec:discussion}

In this section, we discuss the implications of \cref{theorem:main} under several particular settings, including pure MLPs, pure LayerNorm-ed MLPs, Transformers and Transformers with hypothetical massive perturbation training. 
We point out their tendency toward effective gradient sparsity, which leads to gradient and activation sparsity as discussed in \cref{sec:gradients}, among which effective gradient sparsity is more stable.

\subsubsection{Pure MLPs}

The last two terms in \cref{eq:main} have similar forms, so we inspect them together. To have a clearer understanding on them, first consider the situations where models use single-token hidden features in \cref{corollary:main_with_hidden_vectors}
\begin{corollary}[Flatness and sparsity in pure MLPs]\label{corollary:main_with_hidden_vectors}.
    Inherit the assumptions of \cref{theorem:main}. Assume additionally that the model uses hidden features of single tokens, then there is
    \begin{align}
        \trace{\hessian} 
        \ge& \sum_{l} \ex{\left(\eta^l\right)^\transpose \kkT \eta^l} + \sum_{l} \ex{\norm{x^{l-1} + d^l}_2^2 \norm{\eta^l}_2^2} + \sum_l \ex{\norm{g^l_V}_2^2 \norm{\alpha^l}_2^2} \label{eq:main_with_hidden_vectors}.
    \end{align}
\end{corollary}
\begin{proof}
    With single-token hidden features, $\norm{\HXT}_2^2$ reduces to 
    \begin{align}
        &   \norm{\HXT}_2^2
        =   \norm{\eta^l \left(x^{l-1} + d^l\right)^\transpose}_2^2
        =   \trace{\left(\eta^l \left(x^{l-1} + d^l\right)^\transpose\right)^\transpose \eta^l \left(x^{l-1} + d^l\right)^\transpose}\\
        =&  \trace{\left(x^{l-1} + d^l\right) \left(\eta^l\right)^\transpose \eta^l \left(x^{l-1} + d^l\right)^\transpose}
        =  \trace{\left(\eta^l\right)^\transpose \eta^l \left(x^{l-1} + d^l\right)^\transpose \left(x^{l-1} + d^l\right)}\\
        =& \norm{\eta^l}_2^2 \norm{x^{l-1} + d^l}_2^2.
    \end{align}
    $\norm{\GAlphaT}_2^2$ can be reduced similarly.
\end{proof}
If normalization layers are imposed, for example, LayerNorm layers before MLP blocks, then $\norm{x^{l-1}}_2^2 = d$ will not change during training, eliminating all other sources of suppressing the second term aside from effective gradient sparsity. 
Since key matrices also take a large portion of parameters, flatness in these parameters must be achieved as well and $\trace{\hessian[\theta_K]}$ is not too small from $\trace{\hessian}$, the second term alone will have a strong tendency to decrease. Therefore, a rigorously proved, strong and stable tendency of pure LayerNorm-ed MLPs toward effective gradient sparsity is presented in \cref{theorem:main_with_hidden_vectors_and_layernorm}.
\begin{theorem}[Flatness and sparsity in pure MLPs with LayerNorms]\label{theorem:main_with_hidden_vectors_and_layernorm}
    Inherit the assumptions of \cref{theorem:main}. Assume additionally that the model uses vector hidden features and LayerNorm layers, with affine transformation turned off, are imposed before every MLP block. Temporarily assume non-$\dbmlp$ models are used, then there is
    \begin{align}
        \trace{\hessian} 
        \ge& d \sum_{l} \ex{\norm{\eta^l}_2^2} + \sum_l \ex{\norm{g^l_V}_2^2 \norm{\alpha^l}_2^2}\label{eq:main_with_hidden_vectors_and_layernorm}.
    \end{align}

    If $\dbmlp$s are used and LayerNorm layers are placed \emph{before} zeroth biases, by clipping the norm of columns in zeroth biases to $c$, there will be
    \begin{align}
        \trace{\hessian} 
        \ge& \sum_{l} \ex{\left(\eta^l\right)^\transpose \kkT \eta^l} + \left(\sqrt{d} - c\right)^2 \sum_{l} \ex{\norm{\eta^l}_2^2} + \sum_l \ex{\norm{g^l_V}_2^2 \norm{\alpha^l}_2^2} \label{eq:main_with_hidden_vectors_and_layernorm_dbmlp}.
    \end{align}
    By \cref{lemma:eff_and_sparsity}, for $\relu$ networks, there is further
    \begin{align}
        \trace{\hessian} 
        \ge& \sum_{l} \ex{\left(\eta^l\right)^\transpose \kkT \eta^l} \\&+ \left(\sqrt{d} - c\right)^2 \cdot \ex[X, l, i]{\left(g^l_{K, i}\right)^2 \mid \alpha^l_i > 0} \cdot \sum_{l} \ex{\norm{\alpha^l}_0} + \sum_l \ex{\norm{g^l_V}_2^2 \norm{\alpha^l}_2^2}
    \end{align}
    for $\dbmlp$ networks, and similar results can be obtained for architectures without zeroth biases.
\end{theorem}
LayerNorm places quite strong and stable drives towards effective gradient sparsity, if they are placed right before (DB-)MLP blocks and affine factors are turned off to avoid their reduction due to updates or weight decay. \cref{theorem:main_with_hidden_vectors_and_layernorm} can be one explanation of the benefits of LayerNorms and the practice to exclude their parameters from weight decay.

The last term in \cref{eq:main_with_hidden_vectors} relating activation sparsity is less ensured than the second term. In experiments (although conducted with Transformers) we observe $\norm{g_V^l}_2^2$ becomes small in deep layers, indicating that effective gradient sparsity is the main cause of activation sparsity in deep layers.

\subsubsection{Transformers and Other Architectures}

When stacked hidden features are used, for example in Transformers, the discussion is more tricky. It is possible that gradients of different tokens cancel each other in 
\begin{align}
    \HXT = \sum_{i} \eta^l_i \left(x^{l-1}_i + d^l_i\right)^\transpose.
\end{align}
We can only rigorously conclude a possibly loose lowerbound with $\norm{\eta^l_i}_2^2$ in \cref{corollary:main_with_minimum_eigenvalue} using \cref{eq:trace_product_and_eigenvalue}.
\begin{corollary}[Flatness and sparsity in Transformers]\label{corollary:main_with_minimum_eigenvalue}
    Inherit assumptions from \cref{theorem:main}, then there is
    \NewDocumentCommand{\mineigenbound}{m m}{\sum_{l} \ex{\lambda_{\min}\left(\left(#1\right)^\transpose #1\right) \sum_i \norm{#2}_2^2}}
    \begin{align}
        \trace{\hessian} 
        \ge&
            \sum_{i, l} \ex{\left(\eta^l_i\right)^\transpose \kkT \eta^l_i}\\
            &+ \mineigenbound{\left(X^{l-1} + D^l\right)}{\eta^l_i}\\
            &+ \mineigenbound{G^{l}_V}{\alpha^l_i},
    \end{align}
    where $\lambda_{\min}\left(\cdot\right)$ indicates the minimum eigenvalue of a matrix.
\end{corollary}
\arxivonly{\begin{proof}
    \begin{align}
        \HXT 
        =&  \trace{\left(\HXT\right)^\transpose \HXT}\\
        =&  \trace{\left(X^{l-1}\right)^\transpose X^{l-1} \left(H^l\right)^\transpose H^l}.
    \end{align}
    The last term can be similarly rearranged. Applying \cref{eq:trace_product_and_eigenvalue} to both of them and noticing $\trace{\left(H^l\right)^\transpose H^l} = \sum_i \norm{\eta^l_i}_2^2$ finish the proof.
\end{proof}}

There are tricky ways to bypass the canceling, however. For example, consider augments conducted on hidden features such as dropout. 
They effectively duplicate the parameter into $k$ views and perturb each view independently (using dropout, rows of weight matrices are randomly pruned) if there are $k$ tokens. 
If flatness can be extended to these effective duplicated parameters, i.e., if there is still flatness when we really duplicate the weight matrices and assign one token for each matrix, then each view is only handling one token and we can repeat \cref{corollary:main_with_hidden_vectors}. 
However, traditional dropout may hinder the sparsity by eliminating activations and forcing the model to back up the representation. Additionally, its perturbations are not local enough, hindering theoretical analyses. 
A soft dropout by slightly perturbing before activation functions is more preferable. 
Moreover, the perturbation should better be conducted on weight matrices in an entrywise manner to avoid summing and canceling gradients.
Under this hypothetical synapse perturbation \citep{synaptic_noise_1,synaptic_noise_2} but in a tokenwise manner, we assume flatness can be obtained w.r.t. the duplicated parameters because, in the real model, losses are suppressed even under independent perturbation so the effective model is not sensitive to independent changes in individual parameter duplicates. This intuition leads to \cref{lemma:flatness_of_perturbed_model}.

\NewDocumentCommand{\proxytheta}{}{\hat{\theta}}
\NewDocumentCommand{\dupW}{}{\tilde{W}}
\begin{lemma}[Flatnesses of perturbed model and perturbation-induced effective model]\label{lemma:flatness_of_perturbed_model}
    Assume weight matrices are perturbed by Gaussian noise \emph{independently} for each token, i.e., the perturbed $\mlp_*^l$ outputs
    \begin{align}
        \activation\left(\begin{bmatrix}
            \proxyW^{l, 1} x_1 &   \cdots & \proxyW^{l, i} x_i & \cdots & \proxyW^{l, k} x_k
        \end{bmatrix} + b^l\right) \label{eq:massive_perturbation}
    \end{align}
    for input hidden matrix $X = \begin{bmatrix} x_1 & \cdots & x_i \cdots & x_k \end{bmatrix}$, where $\proxyW^{l, i} \defeq W^l + \Epsilon^{l, i}$ , and $\Epsilon^{l, 1}, \dots, \Epsilon^{l, k}$ are $k \times n \times d$ independent centered Gaussian variables with variance $\sigma^2$. 
    Let random variable $\Epsilon$ denote the collection of all perturbations.
    Let $\proxytheta$ be the collection of proxied parameters, where $\proxyW^{l, i}$s are taken into consideration instead of $W^{l}$, while other parameters are inherited from $\theta$.

    Let $g_{\efftheta}$ be the effective parameter by duplicating each weight matrix $W^l$ for $k$ times into $\dupW^{l, 1}, \dots, \dupW^{l, k}$, each of which deals with exactly one hidden vector $x_i$ during inference.

    Then
    \begin{align}
        &   \frac{1}{\sigma^2} \ex[(X, Y), E]{\left(\loss(f_\theta, E, (X, Y)) - \loss(f_\theta, 0, (X, Y))\right)^2} + n d L k \cdot o\left(1\right)\\
        %=& \ex{\norm{\derivatives{\loss(f_\theta, 0, (X, Y))}{\proxytheta}}_2^2}
        %= \ex{\norm{\derivatives{\loss(g_{\efftheta}, (X, Y))}{\efftheta}}_2^2}\\
        =& \sum_{l} \sum_{W \in \set{K, V}} \sum_{i}  \ex{\norm{\derivatives{\loss(f_\theta, 0, (X, Y))}{\proxyW^{l, i}}}_2^2}
        = \sum_{l} \sum_{W \in \set{K, V}} \sum_{i} \ex{\norm{\derivatives{\loss(g_{\efftheta}, (X, Y))}{\dupW^{l, i}}}_2^2},
    \end{align}
    where $\loss(f_\theta, E, (X, Y))$ indicates the loss of $f_\theta$ on sample $(X, Y)$ when perturbation is $E$, $L$ is the number of $\mlp$ layers.
    If Cross Entropy loss is assumed and $f_\theta$ is well learned when perturbations are removed then by \cref{lemma:flatness_and_grad_norm} applied to $g_{\efftheta}$,
    \begin{align}
        &   \frac{1}{\sigma^2} \ex[(X, Y), E]{\left(\ce(f_\theta, E, (X, Y)) - \ce(f_\theta, 0, (X, Y))\right)^2} + n d L k \cdot o\left(1\right)\\
        =& \sum_{l} \sum_{W \in \set{K, V}} \sum_{i}  \ex{\norm{\derivatives{\ce(f_\theta, 0, (X, Y))}{\proxyW^{l, i}}}_2^2}
        = \sum_{l} \sum_{W \in \set{K, V}} \sum_{i} \ex{\norm{\derivatives{\ce(g_{\efftheta}, (X, Y))}{\dupW^{l, i}}}_2^2}\\
        \le& \ex{\norm{\derivatives{\ce(f_\theta, 0, (X, Y))}{\proxytheta}}_2^2} 
        = \ex{\norm{\derivatives{\ce(g_{\efftheta}, (X, Y))}{\efftheta}}_2^2} \approx \trace{\hessian[\efftheta]}\\
    \end{align}
\end{lemma}
\begin{proof}
    By construction of $g_{\efftheta}$, gradients w.r.t. $\proxyW^{l, i}$ and $\dupW^{l, i}$ share the same path in $f_\theta$ and $g_{\efftheta}$. If the same sample is used and the perturbation is removed, then they are equal.

    We can approximate $\loss(f_\theta, \Epsilon, (x, y))$ from $\loss(f_\theta, 0, (x, y))$ by
    \begin{align}
        \left(\loss(f_\theta, \Epsilon, (x, y)) - \loss(f_\theta, 0, (x, y))\right)^2
        =&  \left(\left(\nabla_{\vectorize{\Epsilon}} \loss(f_\theta, 0, (x, y)) \right)^\transpose \vectorize{E} +  o\left(\norm{\vectorize{E}}_2 \right)\right)^2\\
        =&  \left(\left(\nabla_{\vectorize{\Epsilon}} \loss(f_\theta, 0, (x, y)) \right)^\transpose \vectorize{E}\right)^2 +o\left(\norm{\vectorize{E}}_2^2 \right)\\
    \end{align}
    Since Gaussian $\vectorize{E}$ has covariance $\sigma^2 I$, $\nabla_{\vectorize{\Epsilon}}^\transpose \loss \times  \vectorize{E}$ is also Gaussian whose variance is 
    \begin{align}
        \sigma^2 \left(\nabla_{\vectorize{\Epsilon}} \loss(f_\theta, 0, (x, y))\right)^\transpose \left(\nabla_{\vectorize{\Epsilon}} \loss(f_\theta, 0, (x, y))\right) = \sigma^2 \norm{\nabla_{\vectorize{\Epsilon}} \loss}_2^2.    
    \end{align}
    Taking expectation over noises $E$, there is
    \begin{align}
            \ex[E]{\left(\loss(f_\theta, \Epsilon, (x, y)) - \loss(f_\theta, 0, (x, y))\right)^2}
        =& \sigma^2 \norm{\nabla_{\vectorize{\Epsilon}} \loss(f_\theta, 0, (X, Y))}_2^2 + o(n d L k \sigma^2)\\
        =&  \sigma^2 \norm{\nabla_{\vectorize{\proxyW}} \loss(f_\theta, 0, (X, Y))}_2^2 + o(n d L k \sigma^2),
    \end{align}
    where $\dupW$ denote the collection of all $\dupW^{l, i}$s. 
    The rest of the proof is easy according to the equivalence between $f_\theta$ with perturbations removed and $g_{\efftheta}$.
\end{proof}
\begin{remark}
    Although $\trace{\hessian[\efftheta]}$ is involved by an inequality, considering the large portion of $\dupW$ parameters, $\sum_{l} \sum_{W \in \set{K, V}} \sum_{i} \ex{\norm{\derivatives{\ce(g_{\efftheta}, (X, Y))}{\dupW^{l, i}}}_2^2}$ can in fact represent $\trace{\hessian[\efftheta]}$ well. If this argument is not satisfying, then perturb all parameters in the same way so that ``$\le$'' becomes ``$=$''.
\end{remark}

So by training a weight-perturbed non-pure-MLP network to have low losses, we are helping its pure-MLP equivalence reaching flat minima, where effective gradient sparsity can be directly obtained in \cref{theorem:main_with_effective_duplication}. If we assume 
\begin{align}
    \frac{1}{\sigma^2} \ex[(X, Y), E]{\left(\ce(f_\theta, E, (X, Y)) - \ce(f_\theta, 0, (X, Y))\right)^2} \le \trace{\hessian[\efftheta]}
\end{align}
is indeed suppressed during training because losses are suppressed to near-zero values, then \cref{theorem:main_with_effective_duplication} is meaningful.

\begin{theorem}[Flatness and sparsity under tokenwise synapse noise perturbations]\label{theorem:main_with_effective_duplication}
    Inherit the assumptions of \cref{theorem:main} as well as notations in \cref{lemma:flatness_of_perturbed_model}. Further assume that weight matrices are independently perturbed before multiplying with any individual tokens during training, then
    \begin{align}
        &   \frac{1}{\sigma^2} \ex[(X, Y), E]{\left(\ce(f_\theta, E, (X, Y)) - \ce(f_\theta, 0, (X, Y))\right)^2} + n d L k \cdot o\left(1\right) \\& + \sum_{i, l} \ex{\left(\eta^l_i\right)^\transpose \kkT \eta^l_i}\\
        =& \sum_{i, l} \ex{\left(\eta^l_i\right)^\transpose \kkT \eta^l_i} + \sum_{i, l} \ex{\norm{x^{l-1}_i + d^l_i}_2^2 \norm{\eta^l_i}_2^2} + \sum_{i, l} \ex{\norm{g^l_{V, i}}_2^2 \norm{\alpha^l_i}_2^2} \\
        \le&    \trace{\hessian[\tilde{\theta}]} + \trace{\hessian[\theta_D]}
        \label{eq:main_with_effective_duplication},
    \end{align}
    where $\tilde{\theta}$ stands for the perturbation-induced effective parameter where weight matrices are really duplicated so that each of them serves one token and $\theta_D$ is the collection of parameters in all zeroth biases.
    If no-affine LayerNorms are applied, there is further
    \begin{align}
        &   \frac{1}{\sigma^2} \ex[(X, Y), E]{\left(\ce(f_\theta, E, (X, Y)) - \ce(f_\theta, 0, (X, Y))\right)^2} + n d L k \cdot o\left(1\right)\\
        \\& + \sum_{i, l} \ex{\left(\eta^l_i\right)^\transpose \kkT \eta^l_i}\\
        =&  \sum_{i, l} \ex{\left(\eta^l_i\right)^\transpose \kkT \eta^l_i} + \left(\sqrt{d} - c\right)^2 \sum_{i, l} \ex{\norm{\eta^l_i}_2^2} + \sum_{i, l} \ex{\norm{g^l_{V, i}}_2^2 \norm{\alpha^l_i}_2^2}\\
        \le&   \trace{\hessian[\tilde{\theta}]} + \trace{\hessian[\theta_D]},\\
    \end{align}
    where $c=0$ if non-$\dbmlp$s are used, otherwise $c$ is the norm bound of columns in zeroth biases.
    
    By \cref{lemma:eff_and_sparsity}, for $\relu$ networks, $\left(\sqrt{d} - c\right)^2 \sum_{i, l} \ex{\norm{\eta^l_i}_2^2}$ terms in the above equations can be replaced by $\left(\sqrt{d} - c\right)^2 \cdot \ex[X, l, i, j]{\left(g^l_{K, i, j}\right)^2 \mid \alpha^l_{i, j} > 0} \cdot \sum_{i, l} \ex{\norm{\alpha^l_i}_0}$ to have a direct relation with activation sparsity.
\end{theorem}
Aside from Transformers, tokenwisely perturbed CNNs, channel mixing layers of MLP-Mixers and other potential architectures apply, as long as they have MLP blocks and the perturbed loss is small enough.
Additionally, this bound is also very tight by the tightness of \cref{lemma:flatness_of_perturbed_model} or by counting parameters, so perturbed error's reduction or the flatness of the effective model inevitably leads to a reduction in sparsity.
\arxivonly{A simple algorithm $\magic$ is immediate after \cref{lemma:flatness_of_perturbed_model} and \cref{theorem:main_with_effective_duplication}, which is listed in \cref{appendix:magic} in order not to disrupt the presentation of major theoretical results.}

\subsubsection{The First Term from Zeroth Biases}

Now we look back to the first term in \cref{eq:main}. Although the first term seems to have minor weight compared to others, either by counting parameters or by decomposing eigenvalues of $\kkT$ (see \cref{fig:eigenvalues_of_kkT}), an investigation is still worthy since it leads to another phenomenon of spectral concentration in $\kkT$ and introduces random matrix theory to reason about training dynamics. 

In the first term, since the gradient of the inner product w.r.t. $\eta^l$ is 
\begin{align}
    2 \kkT \eta^l,
\end{align}
there is an \emph{overall} positive tendency toward sparsity if $\trace{\hessian}$ is suppressed, because $\left(\eta^l\right)^\transpose 2 \kkT \eta^l$ is non-negative, i.e., partial derivatives w.r.t. $\eta^l$ are always overall positive if they are weighted by values in $\eta^l$ themselves. 

It is better to reach a non-overall conclusion. Moreover, there are two possibilities that can also achieve implicit adversarial robustness: reducing the norm of $\kkT$, or misaligning the non-null space of $\kkT$ with $\eta^l$. 
The first alternative can already be eliminated by parameter growth observed by \citet{parameter_growth}, where Transformers are observed to have increasing parameter norms during training. We also empirically verify this phenomenon under CV and NLP settings and show that the trace, or the sum of eigenvalues, will not decrease drastically during training in \cref{sec:t_exp:spectral_increase}, even under weight decay. 
Another possible elimination is normalization layers, which make parameters scale-invariant and introduce normalized effective parameters \citep{zhang_three_2018} with moderate norms. However, this requires a total reform of the theory to utilize effective weight matrices so we simply hide normalization layers in $M$ and leave it, and especially its interaction with weight decay, for future works.

The other alternative is dealt with in the next two subsections, where single-token features are used in proofs but the theories apply to stacked hidden features.
To give a brief account, in the following subsections we will prove that non-zero eigenvalues of $\kkT$ have similar values. So if gradients have moderate projections in the non-null subspace of $\kkT$, then $\left(\eta^l\right)^\transpose \kkT \eta$ can be lowerbounded by $\lambda r \left(\eta^l\right)^\transpose \left(\eta^l\right)$, where $r \le 1$ indicates how much of $\eta^l$ falls in the non-null space of $\kkT$ and $\lambda$ is the smallest non-zero eigenvalue or some averaged non-zero eigenvalues of $\kkT$, which is not too small compared to the largest one. Assuming gradients are still back-propagated to shallower layers, $\lambda r \left(\eta^l\right)^\transpose \eta^l$ can only be suppressed by decreasing $\norm{\eta^l}_2^2$ given that $\lambda$ increases with the trace of $\kkT$.
In \cref{sec:spectral_init} we prove this phenomenon at initialization in \cref{theorem:initial_spectral_properties} that the largest eigenvalue is initially at most $9$ times larger than the smallest non-zero one in Base-sized Transformers. The proof is based on ubiquitous Xavier or Kaiming initializations and Marchenko-Pastur distribution from random matrix theory.
In \cref{sec:spectral_training}, we theoretically discuss its re-emergence during stochastic training. We first rewrite the updates to the weight matrix $K^l$ into two large random matrices, whose shape is hidden dimension times the number of samples used in training. We then extend Marchenko-Pastur distribution in \cref{theorem:spectral_of_accumulated} under the practical inter-batch dependence, intra-batch independence and non-asymptotic scenario to prove an upperbound on the fraction between the largest and smallest non-zero eigenvalues. Conditions and assumptions of the theorem are verified empirically in \cref{sec:t_exp:anisotropy}. There are still gaps in combining the two random matrices, so we measure the spectral concentration in $\kkT$ empirically in \cref{sec:t_exp:spectral_concentration} and leave a more rigorous discussion for future works.

\input{theory/spectral.tex}

%% file: theory/gradients.tex
\subsection{Gradients w.r.t. Weight Matrices and Zeroth Biases}\label{sec:gradients}

Since we are interested in the flat minima, gradients will be heavily involved. Therefore, we will compute gradients and updates to weight matrices in this subsection. We will also define the notion of effective gradient sparsity and argue its practical and theoretical importance, justifying our theories that are built on effective gradient sparsity.

Let $\mlp_*^l$ be any sublayer of the block $\mlp^l$, whose weight matrix is $W^l$. Let $g^{l}_{*, i}$ be the gradient w.r.t. the output of $\mlp_*^l$ on sample $(x_s, y_s)$ at a token $u_i^{l}$, where $u_i^{l}$ abstracts $x^{l-1}$, $x^{l-1} + d^l$ or $\alpha^l$. Let $G^{l}_{*}$ be the stacked version of $g^{l}_{*, i}$ and $U^l$ be that of $u_i^l$.

One can compute the gradient w.r.t. $W^l$ on a token input $u^{l}_i$ by
\begin{align}
    \left(g^{l}_{*, i} \hadamard \sigma'_i \right) \left(u^{l}_i\right)^{\transpose},
\end{align}
where $\sigma'_i = \gamma^l_i$ for weight matrices, while $\sigma'$ is $[1]_{i \in \set{1, \dots, d}}$ for value matrices since there are no activation functions in value layers.
Summing updates from all tokens to obtain gradients of sample $(x_s, y_s)$ gives
\begin{align}
    \derivatives{\loss(\theta, (x_s, y_s))}{W^l}
    \defeq&  \sum_{i} \left(g^{l}_{*, i} \hadamard \sigma'_i \right) \left(u^{l}_{i}\right)^\transpose
    =      \left(G^{l}_* \hadamard \Sigma'\right) \left(U^l\right)^\transpose.
\end{align}

The gradients w.r.t. zeroth biases, if they exist, are also computed.
\begin{align}
    \derivatives{\loss(\theta, (x_s, y_s))}{d^l_i}
    =& \left(J_{\loss}\left(\alpha^l\right) J_{\alpha^l}\left(K^l \left(x^l_i + d^l_i\right) + b_K^l\right) J_{K^l \left(x^l_i + d^l_i\right) + b_K^l}\left(d^l_i\right)\right)^\transpose\\
    =& \left(\left(g^{l}_{K, i}\right)^\transpose \diag{\gamma^l_i} K^l\right)^\transpose
    =   \left(K^l\right)^\transpose \left(g^l_{K, i} \hadamard \gamma^l_i\right).
\end{align}

Instantiating these results on key and value matrices obtains \cref{lemma:gradients}.
\begin{lemma}[Gradients w.r.t. weight matrices]\label{lemma:gradients}
    The gradients w.r.t. weight matrices and zeroth biases in $\dbmlp^l$ of sample $(x_s, y_s)$ are
    \begin{align}
        \derivatives{\loss(\theta, (x_s, y_s))}{K^l}
        =& \left(G^{l}_K \hadamard \Gamma^l\right) \left(X^{l-1} + D^l\right)^\transpose,
        \derivatives{\loss(\theta, (x_s, y_s))}{V^l}
        = G^{l}_V \left(\Alpha^{l}\right)^\transpose,\\
        \derivatives{\loss(\theta, (x_s, y_s))}{D^l}
        =&  \left(K^l\right)^\transpose \left(G^l_K \hadamard \Gamma^l\right).
    \end{align}
    Particularly if hidden features are single tokens, there are
    \begin{align}
        \derivatives{\loss(\theta, (x_s, y_s))}{K^l}
        =& \left(g^{l}_K \hadamard \gamma^l\right) \left(x^{l-1} + d^l\right)^\transpose,
        \derivatives{\loss(\theta, (x_s, y_s))}{V^l}
        = g^{l}_V \left(\alpha^{l}\right)^\transpose,\\
        \derivatives{\loss(\theta, (x_s, y_s))}{d^l}
        =&  \left(K^l\right)^\transpose \left(g_{K}^l \hadamard \gamma^l\right).
    \end{align}
\end{lemma}

\cref{lemma:gradients} introduces a Hadamard product between gradients from deeper layers and the derivatives of the activation in weight layers. We define it as the effective gradient pattern.
\begin{definition}[Effective gradient sparsity]\label{def:effective_gradient_sparsity}
    Define effective gradient patterns of $\mlp^l$ on sample $(x_s, y_s)$ at token $x^{l-1}$ to be 
    \begin{align}
        \eta^l \defeq \diag{g^l_K} \gamma^l = g^l_K \hadamard \gamma^l. 
    \end{align}
    Let $\Eta^l = G^l_K \hadamard \Gamma^l \in \reals^{n \times k}$ (capitalized ``$\eta$'' instead of capitalized ``$h$'') be its stacked version when there are $k$ tokens.

    Effective gradient sparsity is that most elements in $\eta^l$ are near zero for most samples and tokens.
    For mathematical convenience, define $\norm{\eta^l}_2^2 = \norm{\diag{g^l_K} \gamma^l}_2^2 = \norm{g^l_K \hadamard \gamma^l}_2^2$ to be the effective gradient sparsity measured in squared $L_2$ norm.
\end{definition}

This notion first simplifies \cref{lemma:gradients}.
\begin{lemma}[Gradients w.r.t. weight matrices, restated with $\eta$ and $\Eta$]\label{lemma:gradients_with_eff}
    The gradients w.r.t. weight matrices and zeroth biases in $\dbmlp^l$ of sample $(x_s, y_s)$ are
    \begin{align}
        \derivatives{\loss(\theta, (x_s, y_s))}{K^l}
        =& \Eta^l \left(X^{l-1} + D^l\right)^\transpose,
        \derivatives{\loss(\theta, (x_s, y_s))}{V^l}
        = G^{l}_V \left(\Alpha^{l}\right)^\transpose,\\
        \derivatives{\loss(\theta, (x_s, y_s))}{D^l}
        =&  \left(K^l\right)^\transpose \Eta^l.
    \end{align}
    Particularly if hidden features are single tokens, there are
    \begin{align}
        \derivatives{\loss(\theta, (x_s, y_s))}{K^l}
        =& \eta^l \left(x^{l-1} + d^l\right)^\transpose,
        \derivatives{\loss(\theta, (x_s, y_s))}{V^l}
        = g^{l}_V \left(\alpha^{l}\right)^\transpose,
        \derivatives{\loss(\theta, (x_s, y_s))}{d^l}
        =  \left(K^l\right)^\transpose \eta^l.
    \end{align}
\end{lemma}

This sparsity inherits sparsity in $\gamma$s but also allows more ``sparsity'' due to $g^l_K$s. Sparsity in $g^l_K$ is also meaningful in the sense that if the $i$-th entry $g^l_i$ of $g^l_K$ is small in magnitude, then 1) there is little contribution of $i$-th neuron in gradients to shallower layers during the backward propagation and 2) $\alpha^l_i$ does not influence the output much in forward propagation if the activation is also near-zero. Therefore, the $i$-th neuron can also be pruned during backward propagation and possibly during inference with minor cost in accuracy, and thus this notion is of even more practical value, although $g^l_K$ cannot be known before backward propagation. 
The notion of effective gradient sparsity in $\eta^l$ considers the two kinds of sparsity in a combined manner.
This incorporation of gradients w.r.t. activations also reminds us of the improved knowledge attribution method proposed by \citet{knowledge_neurons} where not only activation magnitude but also the gradients of model output w.r.t. the activation are exploited. 
Last but not least, the effective sparsity hides the complexity of deeper layers and attributes its emergence solely to $K^l$, which is somehow shallow despite there being dozens of deeper modules and allows easier theoretical manipulations.

More importantly, effective gradient sparsity patterns are what key layers try to memorize in columns, given in \cref{obs:memorizing_eff}. 
\begin{observation}[$\eta$s are memorized in key matrices columns]\label{obs:memorizing_eff}
    Consider the update of one sample to the key matrix given by \cref{lemma:gradients_with_eff}
    \begin{align}
            &\derivatives{\loss(\theta, (x_s, y_s))}{K^l}
            = \Eta^l \left(X^{l-1} + D^l\right)^\transpose
            = \sum_{i} \eta^l_i \left(x^{l-1}_i + d^l_i\right)^\transpose\\
            =&  \begin{bmatrix}
                \sum_i \left(X^{l-1}_{i, 1} + D^{l}_{i, 1}\right)\eta^l_i & \cdots & \sum_{i} \left(X^{l-1}_{i, j} + D^{l}_{i, j}\right) \eta^l_i & \cdots & \sum_{i} \left(X^{l-1}_{i, d} + D^{l}_{i, d}\right) \eta^l_i
            \end{bmatrix}.
    \end{align}
    In the update of each column, a mixture of effective gradient sparsity patterns is borne into the key matrix with different weights given by the input.
\end{observation}
Taking a transposed view, key layers also memorize $x^{l-1}_i$s, under the control of $\eta^l_i$s.
\begin{observation}[$\eta$s control row memorization in key matrices]\label{obs:memorizing_inputs}
    Consider the update of one sample to the key matrix given by \cref{lemma:gradients_with_eff}
    \begin{align}
            \derivatives{\loss(\theta, (x_s, y_s))}{K^l}
            =& \Eta^l \left(X^{l-1} + D^l\right)^\transpose
            = \sum_{i} \eta^l_i \left(x^{l-1}_i + d^l_i\right)^\transpose
            =  \begin{bmatrix}
                \sum_i \Eta^l_{j, i} \left(x^{l-1}_i + d^l_i\right)^\transpose
            \end{bmatrix}_j.
    \end{align}
    In the update of each row, a mixture of inputs is borne into the key matrix, weighted by entries of effective gradient sparsity pattern.
\end{observation}

Similar things also happen in value matrices but with $g^{l}_V$s and $\alpha^l$s memorized. 
It is interesting to see that linear layers are trying to resemble the gradients back-propagated to them. 
This observation may lead to a notion of pseudo gradients which can be calculated as the forward propagation sweeps by. Maybe a portion of samples can be trained by these pseudo gradients to save computation budgets. However, this is out of our scope and is left for future works to explore.

Unfortunately, effective gradient sparsity is not strictly related to activation sparsity even under common activation functions, so we keep the notion of gradient sparsity as well. 
The gap is due to the Hadamard product with $g^l_K$, which can cause smaller $\norm{\eta^l}_2^2$ without gradient sparsity on $\gamma^l$ by 1) reducing the norm of itself, or 2) misaligning itself with $\gamma^l$, i.e. multiplying small entries in $g^l_K$ with the derivatives of activated neurons and leaving large gradients to non-activated neurons. The $L_2$ modeling of sparsity also hinders the direct relation with sparsity measured in $L_0$ norms.
The first possibility can be eliminated by the phenomenon of parameter growth already discovered by \citet{parameter_growth}. This indicates that the norm of Transformers' parameters will increase even under weight decay and normalization. Since gradients are obtained by multiplying parameters and hidden features, $\norm{g^l_K}_2^2$ is also likely to increase. We empirically examine it in \cref{sec:t_exp:spectral_increase} where $\trace{M^l_i} \defeq \trace{g_{K, i}^l \left(g_{K, i}^l\right)^\transpose}= \norm{g_{K, i}^l}_2^2$ is observed to increase, at least in ViTs.
For the second possibility and the disconnection between $L_2$ and $L_0$ norm, under $\relu$-activation, similarly to \cref{remark:L2_and_L0}, $\norm{\eta^l}_2 = \norm{g^l \hadamard \gamma^l}_2$ can be seen as the $L_0$ norm of activations, but weighted by entries in $g^l_K$. In \cref{sec:t_exp:align_eff}, we will empirically demonstrate that $\gamma^l$ aligns with $g^l_K$ very well, i.e., the distribution of squared values of entries in $g^l_K$ that corresponds to non-zero entries in $\gamma^l$ is similar or even righter-shifted compared to the distribution of all entries' squared values in $g^l_K$, and the former avoids the long tail of the latter with small magnitudes. This indicates that it is not the case that effective gradient sparsity measured in $L_2$ norms is achieved by adversarially aligning non-zero derivatives of activation functions to small gradients and aligning zero derivatives to large gradients. Therefore, the weighting by $g^l$ to the $L_0$ norm is quite moderate, and a considerable portion of $0$-$1$ entries in $\gamma^l$ are attached to large weights that can be approximated by $\norm{g^l_K}_2^2 / d$ (which is increasing under parameter growth), so at least this portion of entries enjoy the approximate connection from effective gradient sparsity to gradient sparsity, and finally to activation sparsity. This intuition leads to \cref{lemma:eff_and_sparsity} that fully exploits coincidences in $\relu$ and the piecewise constancy of $L_0$ norm.
\begin{lemma}[Relation between $\eta^l$ and $\gamma^l$ for $\relu$ networks]\label{lemma:eff_and_sparsity}
    Let $\set{\gamma^l_i}_{i \in \set{1, \dots, n}} \subseteq \set{0, 1}^d$ be a set of $n$ $0$-$1$ vectors  and $\set{g^l_{K, i}}_{i \in \set{1, \dots, n}} \subseteq \reals^d$ be a set of $n$ real vectors. Let $\eta^l_i \defeq g^l_{K, i} \hadamard \gamma^l_i$, resembling \cref{def:effective_gradient_sparsity}.
    For any distribution $D$ over subscripts $i \in \set{1, \dots, n}$, there is
    \begin{align}
        &   \ex[i \sim D]{\norm{\eta^l_i}_2^2}\\
        =&  \ex[i \sim D]{d \cdot \ex[j \sim U[1, \dots, d]]{\left(g^l_{K, i, j}\right)^2 \left(\gamma^l_{i, j}\right)^2}}
        =   d \cdot \ex[i, j]{\left(g^l_{K, i, j}\right)^2 \gamma^l_{i, j}}\\
        =&  d \cdot \prob{\gamma^l_{i, j} = 1} \cdot 1 \cdot \ex{\left(g^l_{K, i, j}\right)^2 \mid \gamma^l_{i, j} = 1}
            + d \cdot \prob{\gamma^l_{i, j} = 0} \cdot 0 \cdot \ex{\left(g^l_{K, i, j}\right)^2 \mid \gamma^l_{i, j} = 0}\\
        =&  d \cdot \prob{\gamma^l_{i, j} = 1} \cdot 1 \cdot \ex{\left(g^l_{K, i, j}\right)^2 \mid \gamma^l_{i, j} = 1}
            + d \cdot \prob{\gamma^l_{i, j} = 0} \cdot 0 \cdot \ex{\left(g^l_{K, i, j}\right)^2 \mid \gamma^l_{i, j} = 1}\\
        =&  d \cdot \ex{\left(g^l_{K, i, j}\right)^2 \mid \gamma^l_{i, j} = 1} \cdot \left(\prob{\gamma^l_{i, j} = 1} \cdot 1 + \prob{\gamma^l_{i, j} = 0} \cdot 0\right)\\
        =&  \ex[i \sim D, j \sim U[1,\dots, d]]{\left(g^l_{K, i, j}\right)^2 \mid \gamma^l_{i, j} = 1} \cdot \ex[i \sim D]{\norm{\gamma^l_i}_0},
    \end{align}
    where $U[1, \dots, d]$ stands for uniform distribution among $\set{1,2, \dots, d}$, $g^l_{K, i, j}$ and $\gamma^l_{i, j}$ stand for the $j$-th entry of the $i$-th vectors.
\end{lemma}
In \cref{sec:t_exp:align_eff} we will demonstrate that $\ex{\left(g^l_{K, i, j}\right)^2 \mid \gamma^l_{i, j} = 1}$ is comparable to or even larger than $\ex{\left(g^l_{K, i, j}\right)^2}$ that is increasing due to parameter growth.
For other activation functions with jump discontinuity between deactivation and activation like $\jrelu$, the alignment is also moderate and the weighted $L_2$ norm first pushes activations towards zero, after which $L_2$ norms become closer to $L_0$ norm due to derivatives' jump discontinuity at $0$. Following a similar argument of \cref{lemma:eff_and_sparsity}, effective gradient sparsity then approximates sparsity measured in $L_0$ norms as well.
Another very informal and heuristic argument for the alignment as well as $\eta^l$'s connection to activation and gradient sparsity is that, if the $i$-th entry in $\eta^l$ is near zero, then during the row memorization described by \cref{obs:memorizing_inputs}, little change is imposed by $x^{l-1}$ to the $i$-th row of $K^l$. Next time $x^{l-1}$ arrives, although with changes due to shallower layers, it is more likely the $i$-th row in $K^l$ forgets $x^{l-1}$ and $\left(K^l x^{l-1}\right)_i = \inner{K^l_i}{x^{l-1}}$ or $\left(K^l \left(x^{l-1} + d^l\right)\right)_i = \inner{K^l_i}{x^{l-1} + d^l}$ are closer to $0$ or negative values, leading to smaller possibility of activation, followed by activation sparsity in $\alpha^l$, and thus gradient sparsity in $\gamma^l$ under common activation functions.
To sum up, effective gradient sparsity measured by $L_2$ norm can be connected to activation or gradient sparsity measured in $L_0$ norms, although the result requires empirical assumptions and is heuristic and informal for activation functions other than $\relu$. Therefore, we will base our theorems on $\norm{\eta^l}_2^2$ considering the connection discussed above and the mathematical convenience brought by $L_2$ norm and the Hadamard product.

%% file: theory/spectral.tex
\subsection{Spectral Concentration at Initialization}\label{sec:spectral_init}

In this section, we prove that $K^l \left(K^l\right)^\transpose$ has eigenvalues that are close to each other, at least at initialization. With effective gradient sparsity measured by $\norm{\eta^l}_2^2$, this spectral concentration allows us to approximate the first term in RHS of \cref{eq:main} with $\lambda r \left(\eta^l\right)^\transpose \eta^l$, which is almost directly effective gradient sparsity measured in $L_2$ norms. 

To reach this goal, recall Marchenko–Pastur distribution in \cref{theorem:singular_value_of_product_of_random_matrices} that reveals the asymptotic spectral distribution of random matrices' product.
Applying \cref{theorem:singular_value_of_product_of_random_matrices} to $K^l\left(K^l\right)^\transpose$ initialized by Xavier or Kaiming initialization, we obtain \cref{theorem:initial_spectral_properties_of_kkT}.

\begin{theorem}[Initial spectral concentration of $\kkT$]\label{theorem:initial_spectral_properties_of_kkT}
    Assume $n \neq d$. Let $K^l \in \reals^{n \times d}$ be the weight matrix initialized by (Gaussian, uniform, or other distribution-based) Xavier or Kaiming initialization. When $n, d \to \infty$ with $d / n = 1 / c$, the ratio between the largest and smallest \emph{non-zero} eigenvalues of $\kkT$ converges weakly to
    \begin{align}
        \frac{\lambda_1\left(\kkT\right)}{\min_{k: \lambda_k\left(\kkT\right) > 0} \lambda_k\left(\kkT\right)} \le \frac{\left(1 + \sqrt{c}\right)^2}{\left(1 - \sqrt{c}\right)^2}.
    \end{align}
    Regarding zero eigenvalues, if $d > n$, there is no zero eigenvalue, and if $d \le n$, the expected portion of zero eigenvalues is $1 - \frac{1}{c} = 1 - \frac{d}{n}$.
    
\end{theorem}
\begin{proof}
    The initialization methods utilize centered distribution, and thus there is $\ex{K^l_{i, j}} = 0$.

    $\kkT$ differs from $S^p$ of \cref{theorem:singular_value_of_product_of_random_matrices} in 1) the shared standard variance of entries not being $1$, and 2) the scaling factor $\frac{1}{d}$. Since we are only interested in the ratio between eigenvalues, these differences of simultaneous scaling can be ignored.

    By \cref{eq:mp_density}, we can see that the support of eigenvalues is restricted to $[a, b] \cup \set{0}$. As a result, non-zero eigenvalues can only be found in $[a, b]$.

    When $c < 1$, i.e., $d > n$, the support degenerates to $[a, b]$. 
    When $c \ge 1$, the probability to pick a zero eigenvalue is $F(0) = \lim_{u \to 0^+} \int_{0}^{u} f(x) \mathrm{d}x = \lim_{u \to 0^+} \int_{0}^{u} \left(1 - \frac{1}{c}\right) \delta(x) \mathrm{d} x = 1 - \frac{1}{c}$ .
\end{proof}

Note that \cref{theorem:initial_spectral_properties_of_kkT} applies to uniform or other base initialization distribution as long as it is centered and entrywisely independent with the same variance. Since $n, d$ are generally large, even in small model sizes like Small and Base, we believe this lemma applies to common practice. In Base-sized Transformers, it is usually the case where $n = 3072, d = 768$, indicating $1 - \frac{768}{3072} = \frac{3}{4}$ of eigenvalues are $0$, while the rest of them varies up to the ratio of $\frac{\left(1 + \sqrt{\frac{1}{4}}\right)^2}{\left(1 - \sqrt{\frac{1}{4}}\right)^2} = 9$. This is a surprisingly small value compared to the number of dimensions.

Effective gradient sparsity patterns have a great affinity to $\kkT$, allowing \cref{theorem:initial_spectral_properties}.

\begin{theorem}[Implication of spectral concentration of $\kkT$ at initialization]\label{theorem:initial_spectral_properties}
    Assume $d \neq n$ and they are sufficiently large. $K^l \in \reals^{n \times d}$ is initialized as in \cref{theorem:initial_spectral_properties_of_kkT}. Let $M^l = g^l_K \left(g^l_K\right)^\transpose$, $\gamma^l$ and $\eta^l$ be those defined previously. 
    
    If $d > n$ then there is
    \begin{align}
        \left(\gamma^l\right)^\transpose \left(\left(K^l \left(K^l\right)^\transpose\right) \hadamard M^l\right) \gamma^l
        =& \left(\eta^l\right)^\transpose \kkT \eta^l\\
        \ge&    \lambda_{n}\left(\kkT\right) \cdot \norm{\eta^l}_2^2,
    \end{align}
    where $n$-th eigenvalue $\lambda_n\left(\kkT\right)$ is moderate and cannot be arbitrarily small because
    \begin{align}
        \frac{\lambda_1}{\lambda_{n}} \le \left(\frac{1 + \sqrt{c}}{1 - \sqrt{c}}\right)^2,
    \end{align}
    where $c = d / n$.

    If $n > d$, let $K^l = U \Sigma V^\transpose$ be the singular value decomposition of $K^l$. The result is restricted to the projection to the subspace expanded by $\left(U^\transpose\right)_{1:d}$, i.e.,
    \begin{align}
        \left(\gamma^l\right)^\transpose \left(\left(K^l \left(K^l\right)^\transpose\right) \hadamard M^l\right) \gamma^l
        =& \left(\eta^l\right)^\transpose \kkT \eta^l\\
        \ge&    \lambda_{d}\left(\kkT\right) \cdot \norm{\left(U^{T}\right)_{1: d}\eta^l}_2^2,
    \end{align}
    where $d$-th eigenvalue $\lambda_d\left(\kkT\right)$ satisfies
    \begin{align}
        \frac{\lambda_1}{\lambda_{d}} \le \left(\frac{1 + \sqrt{c}}{1 - \sqrt{c}}\right)^2,
    \end{align}
    where $c = d / n$.

    For demonstration, when $\set{n, d} = \set{3072, 768}$, the ratio upperbound is $9$.
\end{theorem}
\begin{proof}
    The proof is straightforward after \cref{theorem:initial_spectral_properties_of_kkT}, by noting that when $n > d$ there are exactly $\left(1 - \frac{1}{c}\right) n = n - d$ zero eigenvalues in $\kkT$, or equivalently $d$ non-zero eigenvalues in $\kkT$.
\end{proof}

There are still gaps between $\lambda \left(\eta^l\right)^\transpose \eta^l$ and current practices where $d < n$ and there are a lot of zero eigenvalues in $\kkT$, but \cref{theorem:initial_spectral_properties} is perfectly useful for wide MLPs where $d > n$. Therefore, we propose a drastic architectural modification called wide MLP where $d > n$, i.e., the model dimension is larger than the hidden dimension in MLP blocks. Aside from more powerful motivation toward sparsity, wide MLPs also allow rows in $K^l$ to be mutually orthogonal and permit perfect sparsity, which is impossible when $n > d$.
 
In non-wide MLPs, we believe that since $\kkT$ are randomly initialized and samples are randomly selected, there are moderate projections of $\eta^l$ into the non-null subspace of $\kkT$. Another supporting intuition is that although $\sigma'$ is not identity or linear, the derivatives of common activation functions are often monotonically increasing and form an approximation to its inputs. This approximation is better when part of the activation derivatives are linear, as in the case of Squared-$\relu$\citep{primer} and our $\jrelu$. Therefore, taking $\sigma'$ to $K^l x + b^l$, which already falls near the subspace expanded by $\left(U^\transpose\right)_{1: d}$, does not deviate far from the subspace. 
\cref{obs:memorizing_eff} also supports this moderate projection dynamically because if $\eta^l$ is in the null space, it will be borne into every column of the key matrix and next time it will have non-zero projections if $\eta^l$ does not change too much after one epoch and the column memory is not blurred too severely. Repeatedly memorizing different $\eta^l$ will make the non-null space of $\kkT$ a mixture of the majority $\eta^l$s provided by the training data set. 
The empirical evidence for this is that $\left(\eta^l\right)^\transpose \kkT \eta^l$, according to the derivation in \cref{lemma:adversarial_and_sparsity}, is actually the norm of gradients back propagated to shallower layers. Extreme cases where $\eta^l$ are contained only in the null space of $\kkT$ result in zero gradients for shallower layers, which rarely happens. If no residual connection is involved this insight strongly augments the spectral explanation. The detailed and formal analysis of the zero eigenvalues especially when there are residual connections is left for future empirical and theoretical works. For now, we can simply cover the gap with wide MLPs.

\subsection{Spectral Concentration during Stochastic Training}\label{sec:spectral_training}

In this subsection, we discuss how spectral concentration re-emerges during later stochastic training. 

First recall the Marchenko-Pastur distribution in \cref{theorem:singular_value_of_product_of_random_matrices}. The condition of random centered matrices in \cref{theorem:singular_value_of_product_of_random_matrices} invites another randomness other than initialization to the party, i.e., stochastic gradient noise (SGN) brought by stochastic optimizers.
After $t$ updates, $K^{l}$ can be written as the sum of random initialization, stochastic gradient noises and full-batch gradients that are not as stochastic as the two former terms, i.e.
\begin{align}
    K^{l, t} = \underbrace{K^{l, 0} - \sum_{i=1}^{t} U^{i}_{K^l}}_{\text{stochastic, centered}} - \sum_{i=1}^t \derivatives{\loss(\theta^t)}{K^l}
\end{align}
As discussed in \cref{sec:flat_minima}, $U^{i}_{K^l}$ is by definition centered. If it can be assumed that the $\sas$ SGN with large variance and noise norm shadows full-batch gradient, then $K^{l, t}$ is the sum of two centered random matrix with a slight non-stochastic bias, to which Marchenko-Pastur distribution would approximately apply if further entries in $U^{i}_{K^l}$ shared similar variance and were sampled independently. \citet{relax_mp_distribution} and works cited by them have tried to relax the independence condition of \cref{theorem:singular_value_of_product_of_random_matrices}, but it is still far from applying relaxed Marchenko-Pastur distribution here. Aside from waiting for this mathematical progress, we build empirical basis in \cref{sec:t_exp:spectral_concentration} where at all steps, in $\kkT$ of ViT and the decoder of T5, there is a stable portion of near-zero eigenvalues as in \cref{theorem:initial_spectral_properties} across all layers, and the majority of non-zero ones, with significant gap with near-zero ones, vary up to a ratio of $<100$ for most of the time. It is not surprising that this effect empirically sustains and even becomes stronger at the end of training because the model is well-learned by then and the full-batch gradient is of a smaller norm.

\input{theory/marchenko-pastur.tex} 

Finally, we have discussed the spectral concentration of $\kkT$. To put everything together, assuming moderate projection of $\eta^l$ into non-null space of $\kkT$, and considering the increasing trace of $\kkT$ and spectral concentration in $\kkT$, the only way to suppress 
\begin{align}
    \trace{\hessian[\theta_D]} \ge \left(\gamma^l\right)^\transpose \left(\kkT \hadamard M^l\right) \gamma^l = \left(\eta^l\right)^\transpose \kkT \eta^l \ge \lambda r \left(\eta^l\right)^\transpose \eta^l
\end{align}
is to reduce $\norm{\eta^l}_2^2 = \norm{g^l_K \hadamard \gamma^l}_2^2$. To see gradient-sparsity-induced activation sparsity's emergence, given empirically that $\trace{M^l} = \trace{\left(g^l_K\right)^\transpose g^l_K} = \norm{g^l_K}_2^2$ will not decrease during training, the only way to suppress $\norm{\eta^l}_2^2$ is to decrease $\gamma^l$, at least at entries where $g^l_K$ has large magnitudes.

%% file: theory/marchenko-pastur.tex
\NewDocumentCommand{\DKDKT}{O{t}}{\DK{#1}\left(\DK{#1}\right)^\transpose}
\NewDocumentCommand{\DKTDK}{O{t}}{\left(\DK{#1}\right)^\transpose \DK{#1}}

To have a more satisfying discussion, we propose an extended version of Marchenko-Pastur distribution and find a re-directed view on stochastic gradients to apply it. To this end, what conditions and assumptions can stochastic training provide must be figured out first.

Observing the structure of $\mlp$ or $\dbmlp$ layers, the most essential operation involving weight matrix $K^l$ is
\begin{align}
    z^l \defeq K^l x^l,
\end{align}
where $x$ is abused to represent anything that is multiplied with $K^l$, abstracting $x^l$ or $x^l + d^l$, while $z^l$ is the vector passed to the vanilla bias, activation function or later layers. This structure gives birth to the update of a sample $(x_s, y_s)$ to $K^l$ that writes
\NewDocumentCommand{\DK}{m}{\Delta K^{l, #1}}
\begin{align}
    \DK{s}
    =&  -\eta_{\mathrm{lr}} \cdot \derivatives{\loss(\theta, (x_s, y_s))}{z^{l, s}} \times (x^{l, s})^\transpose
    =  -\eta_{\mathrm{lr}} \cdot \eta^{l, s} \left(x^{l, s}\right)^\transpose,
\end{align}
where $\eta_{\mathrm{lr}}$ is the learning rate, assuming no scheduling is used.
At step $t$ with batch $B_t$, the update on $K^l$ averages these samplewise differences, i.e.
\begin{align}
    \DK{t}
    \defeq& - \frac{\eta_{\mathrm{lr}}}{\size{B_t}} \sum_{s \in B_t} \eta^{l, s} \left(x^{l, s}\right)^\transpose
    =  - \frac{\eta_{\mathrm{lr}}}{\size{B_t}} H^t \left(X^t\right)^\transpose,
\end{align}
where $\Eta^t \in \reals^{n \times \size{B_t}}$ (capitalized ``$\eta$'') is the matrix consisting of column vectors $\eta^{l,s}$ for sample $s$ in the batch $B_t$, and $X^t \in \reals^{d \times \size{B_t}}$ is similarly constructed with $x^{l,s}, s \in B_t$. 
Note that $X^t$ and $\Eta^t$ are random matrices because samples are independently randomly selected and gradients are also random variables as functions of variables.
Taking a similar view throughout the training, there is
\begin{align}
    K^{l, T} - K^{l, 0}
    =&  -\eta_{\mathrm{lr}} \sum_{t=1}^T \frac{1}{\size{B_t}} \sum_{s \in B_t} \eta^{l, s} \left(x^{l, s}\right)^\transpose
    =  -\frac{\eta_{\mathrm{lr}}}{b} \Eta^{1: T} \left(X^{1:T}\right)^\transpose \label{eq:update_and_large_matrices},
\end{align}
where $T$ is the number of batches, $b$ is batch size, and $\Eta^{1:T} \defeq \begin{bmatrix}  \Eta^1 & \cdots & \Eta^t & \cdots & \Eta^T  \end{bmatrix}$, and $X^{1:T} \defeq \begin{bmatrix} X^1 & \cdots & X^t & \cdots & X^T \end{bmatrix}$.
Another product of large random matrices emerges in the empirical covariance matrix of the difference, i.e.,
\begin{align}
    \left(K^{l, T} - K^{l, 0}\right) \left(K^{l, T} - K^{l, 0}\right)^\transpose = \frac{\eta_{\mathrm{lr}}^2}{b^2} \Eta^{1: T} \left(X^{1:T}\right)^\transpose X^{1:T} \left(\Eta^{1:T}\right)^\transpose,
\end{align}
where $X^{1:T}$ and $\Eta^{1:T}$ are random matrices in the sense that samples or gradient vectors in each batch are independently randomly sampled, if conditioned on the model state.

Since we are interested in spectral distribution and that cycling a matrix product does \emph{not} change non-zero eigenvalues, a more desirable form is
\begin{align}
    \left(\Eta^{1:T}\right)^\transpose \Eta^{1:T} \left(X^{1:T}\right)^\transpose X^{1:T},
\end{align}
and we intend to separately investigate the spectral distributions of 
\begin{align}
    &\text{$\left(\Eta^{1:T}\right)^\transpose \Eta^{1:T}$ or spectrally equivalent $\Eta^{1:T} \left(\Eta^{1:T}\right)^\transpose$},\\
    \text{and, }&\text{$\left(X^{1:T}\right)^\transpose X^{1:T}$ or spectrally equivalent $X^{1:T} \left(X^{1:T}\right)^\transpose$}.
\end{align}

After these transforms and dividing-and-conquering, the empirical covariance matrices look ready for Marchenko-Pastur law. However, there are dependencies between previous and later batches through model states, hindering the direct application of independence conditions of Marchenko-Pastur law. Fortunately, there is still conditional independence \emph{within} a batch. This mixture of dependence and independence is captured by \cref{def:batch_model}.
\begin{definition}[Batch Dependence Model]\label{def:batch_model}
    Let $U^{1: T} \in \reals^{p \times (b T)}$ be a random matrices. Decompose it into blocks with batch size $b$, i.e., 
        \begin{align}
            U^{1: T} = 
                \begin{bmatrix}
                    U^1 & \cdots & U^t & \cdots & U^{T-1} & U^{T}
                \end{bmatrix},
        \end{align}
    where $U^t \in \reals^{p \times b}$.
    If the dependence between elements can be described by SCMs
    \begin{align}
        \set{u^t_k \defeq g^t\left(U^{1}, U^{2}, \dots, U^{t-1}, \epsilon^t_k\right): t \in [1, T], k \in [1, b]}
    \end{align}
    or SCMs that resemble the notions of samples and model state
    \begin{align}
        \set{u^t_k \defeq g^t\left(m^{t-1}, \epsilon^t_k\right) : t \in [1, T], k \in [1, b]} \cup \set{m^t \defeq h^t\left(m^{t-1}, U^{t}\right)}
    \end{align}
    where $u^t_k$s are columns in $U^t$, $m^{t}$ is the model state (parameters, momentum, etc.) after step $t$, $\epsilon^{t}_k$ is I.I.D. random noises, then $U^{1: T}$ is a random matrix with batch dependence.
    \begin{remark}
        Within each batch (i.e., when conditioned on all previous batches), samples are I.I.D. sampled to simulate batch sampling. However, previous samples have trained the parameters and will shift the distribution of shallow layers's output as well as back-propagated gradients. Therefore, the current batch depends on previous batches.
    \end{remark}
\end{definition}

There are works re-establishing Marchenko-Pastur law with independence conditions relaxed to martingale conditions \citep{mp_martingale}, but some conditions in it require entrywise conditional independence. There are also Marchenko-Pastur laws for time series \citep{mp_linear_time_series1,mp_linear_time_series2}, but restricted to linear dependence. 
We adapt proofs by \citet{mp_quadratic_form}, and use anisotropy condition in all samples or gradients to extend Marchenko-Pastur distribution under the batch dependence in $X^{1:T}$ and $\Eta^{1:T}$, leading to \cref{theorem:spectral_of_accumulated}.
In later formal definitions, theorems and proofs, $X$ is abstractly used to represent both $\Eta^{1: T}$ or $X^{1:T}$, $p$ indicates the height of $X^{1:T}$ or $\Eta^{1: T}$, i.e., the hidden dimensions $d$ or $n$, while $b$ and $T$ keep their meanings as batch size and total number of batches.

\def\StateSpectralOfAccumulated{display}
\input{theory/theorems/ext_mp.tex}

Its proof is left in \cref{proof:spectral_of_accumulated}. This concludes that spectral concentration also happens in $\Eta^{1: T} \left(\Eta^{1:T}\right)^\transpose$ and $X^{1:T} \left(X^{1:T}\right)^\transpose$, as long as the samples in one batch are independently sampled and all model-state-specific samples involved during training are diverse enough to have low anisotropy and norm bounds. To see that the conditions are satisfied, note the only hard conditions are that of continuity, which is hard to verify and thus simply assumed, as well as the choice of $v$ and $v_0$. We assume enough steps have been trained so $c$ can be very small, for example effectively $\frac{768}{32 \times 50,042} \approx 4.8 \times 10^{-4}$ under strong weight decay (it is even smaller when weight decay weakens or disappears) as we shall see in latter paragraphs. So $v_0 = 10^{-3} > 2 c$ can be chosen. In experiments in \cref{sec:t_exp:anisotropy} we will see $\beta / p$ is always smaller than $1$ and $\alpha / p$ can be less than $0.05$ when weight decay is strong and dimension $p$ is large enough. Under these conditions, it can be verified that LHS of \cref{eq:hard_condition} is larger than RHS by at least $0.061$. So the applicability of the theorem depends on how good the bound is. To empirically verify the applicability of \cref{theorem:spectral_of_accumulated}, one needs to compute anisotropy $\ex{\sqrt{p \trace{\left(T^{p} - I\right)^\transpose\left(T^{p} - I\right)}}}$ and $\alpha$ from $x^p$'s marginal distribution, i.e., by mixing hidden features or back-propagated gradients, which is conducted in \cref{sec:t_exp:anisotropy}.

In contrast with conventional asymptotic results with $p \to \infty$ and $p / b T \to c$ in RMT, our theorem gives bounds for non-limiting scenarios. One of its benefits in the context of machine learning is that one often is more interested in training behaviors when the total number $b T$ of training samples increases as the training proceeds while $p$ is held still. 
Nevertheless, the co-increasing scenarios are also of interest, especially in the era of large models \citep{scaling_law}.
In our result, the spectral concentration is measured by the expected fraction between eigenvalues and the expected eigenvalue, but with an extra shadowing parameter $v$ that may shadow small eigenvalues and decreasing $v$ to suppress the disturbance loosens the bound. Fortunately, most $v$ as dominators show up together with $c$ as numerators, the ratio between hidden dimension and number of training samples used which is often extremely small. If $\frac{\alpha}{p}$ is also small and decreases as $p$ is enlarged, as we will show in the experiments, then the bound will be controlled.

It is frustrating to see that the bound diverges as the training step increases due to factor $\frac{1}{c} = \frac{b T}{p}$. However, we consider weight decay as an alleviation to this issue. Weight decay effectively introduces a sliding window which reduces the effective $T$, because vectors out of the window are exponentially decayed and they have little contribution to the sample covariance matrix. For example, let $w$ be the parameter of weight decay, then each column in $\Eta^{1: T}$ or $X^{1:T}$ becomes $\left(\sqrt{1 - \eta_{\mathrm{lr}} w}\right)^{T - t} \eta_{t, l}$ or $\left(\sqrt{1 - \eta_{\mathrm{lr}} w}\right)^{T - t} u_{t, l}$, where $\eta_{\mathrm{lr}}$ is learning rate. Setting $r = 1 - \eta_{\mathrm{lr}} w$, if we consider the tail whose weights' sum is smaller than a threshold $\tau$ as those out of the window, then the window size $k$ needs to satisfy $\sum_{i={k+1}}^{\infty} r^i = r^{k+1} / (1 - r) \le \tau$, where $r$ is used to obtain tighter bounds instead of $\sqrt{r} = \sqrt{1 - \eta_{\mathrm{lr}} w}$ by arguments in \cref{appendix:effective_window_size}. One sufficient condition for this is $k \ge \frac{\ln \tau (1 - r)}{\ln r} - 1$. When $\tau=10^{-3}, \eta=10^{-3}, w = 10^{-1}$, the effective window size is about $161,172$. When $w$ increases to $0.3$, the effective window size becomes $50,057$. As a result, $\frac{1}{c}$ is upperbounded by a constant as the training proceeds if weight decay presents.

There are still gaps between our results and the spectral distribution of $\kkT$ during stochastic training. For example, spectral concentration of $\Eta \Eta^\transpose$ and $X X^\transpose$ hints similar phenomena in their interlaced product $\left(K^{l, T} - K^{l, 0}\right) \left(K^{l, T} - K^{l, 0}\right)^\transpose$, which, however, is not yet formally proved. Moreover, applying \cref{theorem:spectral_of_accumulated} assumes SGD instead of adaptive optimizers is used due to \cref{eq:update_and_large_matrices}. Filling these gaps is left for future works because we have established spectral concentration's empirical supports in \cref{sec:t_exp:spectral_concentration}.

%% file: theory/theorems/ext_mp.tex
\ifdefstring{\StateSpectralOfAccumulated}{display}{

\begin{restatable}[Spectral concentration of accumulated steps]{theorem}{SpectralOfAccumulated}
    \label{theorem:spectral_of_accumulated}
    Let $X^p = X^{p, b T} \in \reals^{p \times b T}$ be a random matrix that forms a Batch Dependence Model as in \cref{def:batch_model} with batch size $b$ and step count $T$, whose columns are $x^p_j = X^{p, b T}_{\cdot, j} \in \reals^{p}$. 
    Columns in $X^p = X^{p, b T} \in \reals^{p \times b T}$ are \emph{not} necessarily independent.
        
    Let $x_k^p \defeq X^{p}_{\cdot, k}$ be the $k$-th column of $X^p$ and $x_{t, l}^p \defeq X^{t}_{\cdot, l}$ be the $l$-th column of batch $t$ or equivalently the $k=\left((t-1)*b + l\right)$-th column $x_k$ in $X^{p}$. Superscription $p$ may be dropped for convenience.

    Let $S^{p} \defeq \frac{1}{b T} X^p \left(X^p\right)^\transpose = \frac{1}{b T} \sum_{k=1}^{b T} x_k x_k^\transpose$ be the empirical covariance matrix of all random vectors, and $I_p$ be the compatible identity matrix.
    Assume $x_{t, l}$s' norm is bounded, say by $1$, and scale it with 
    \begin{align}
        u^p_{t, l} \defeq \sqrt{\alpha} x^p_{t, l},
    \end{align}
    obtaining $U^p = U^{p, b} \defeq \begin{bmatrix} U^1 & \cdots &  U^t & \cdots  & U^T \end{bmatrix} = \sqrt{a} \cdot X^p$, where $a \defeq \frac{\trace{S^p}}{\trace{S^p S^p}}$. Let 
    \begin{align}
        T^{p} \defeq \frac{1}{b T} U^p \left(U^p\right)^\transpose = \frac{1}{b T} \sum_{k=1}^{b T} u_{k} u_{k}^\transpose = a \cdot S^{p}
    \end{align}
    be the empirical covariance matrix of all $u^p_{k}$s. 

    Assume $a$ is bounded by $\alpha(p)$ and $\ex{\sqrt{p \trace{\left(T^{p} - I^p\right)\left(T^{t} - I^p\right)}}}$ is also upperbounded by $\beta(p)$.

    Further assume that the following function of $z \in \positivecomplex$ and $p, b, T \in \nats^+$
    \begin{align}
        \ex{\sum_{i} \frac{1}{\lambda_i\left(U U^\transpose\right) - z}}
    \end{align}
    is always continuous w.r.t. $z$ for any $p, b, T$.

    If the above assumptions are satisfied, the non-zero eigenvalue concentrates. To be more specific, let $\overline{\lambda^{>0}}$ be the mean of non-zero eigenvalues of $\frac{1}{b T} U^p \left(U^p\right)^\transpose$ and use $\ex{\overline{\lambda^{>0}}}^2$ to represent the overall situation of non-zero eigenvalues. Then there is
    \begin{align}
        \ex{\frac{\ex{\overline{\lambda^{>0}} / \sqrt{v}}^2}{\left(\frac{\lambda}{\sqrt{v}}\right)^2 + v}}
        \le&    \frac{\sqrt 2}{c \sqrt{v}} \frac{\alpha^2}{v \cdot \min(b T, p)^2}  \sqrt{c + \frac{\left(2 \sqrt{2} + 2\right) c \alpha}{v p} + \frac{c \beta}{v p} + \frac{c}{v p}} \label{eq:im_bound},\\
        \ex{\frac{\ex{\overline{\lambda^{>0}}}}{\lambda + \frac{v^2}{\lambda}}}
        \le&    \frac{\sqrt 2}{c \sqrt{v}} \frac{\alpha}{\min(b T, p)}  \sqrt{c + \frac{\left(2 \sqrt{2} + 2\right) c \alpha}{v p} + \frac{c \beta}{v p} + \frac{c}{v p}} \label{eq:re_bound}.
        \end{align}
    for any $v \ge v_0$, where $\lambda$ is a randomly selected eigenvalue of $T^p \defeq \frac{1}{b T} U^p \left(U^p\right)^\transpose$ and $c = p / b T \in [0, 1]$, and $v_0 \ge 2 c$ satisfying 
    \begin{align}
        \frac{v_0 + (1 - c)}{\sqrt{2}} > \frac{\tau}{v_0} + 2 \sqrt{c v_0} + 2 \sqrt{\tau} \label{eq:hard_condition}
    \end{align}
    with $\tau \defeq \frac{c}{p} \left(1 + \beta + 2\left(\sqrt{2} + 1\right) \alpha\right)$.
\end{restatable}

}{

% \arxivonly{\SpectralOfAccumulated*}
\begin{proof}[Proof of \cref{theorem:spectral_of_accumulated}]\label{proof:spectral_of_accumulated}

The proof is adapted from \citet{mp_quadratic_form} where independence conditions are replaced with Batch Dependence model and new regularities.

Cauchy-\sti{} transform method is used. 
When applied to empirical spectral density, by definition there is
\begin{align}
    s^{F^A}(z) = \trace{A - z I}^{-1} / p \defeq \trace{\left(A - z I\right)^{-1}} / p.
\end{align}
for positive semi-definite $A \in \reals^{p \times p}$.
Specifically, $\frac{1}{b T} U^p \left(U^p\right)^\transpose = \frac{1}{b T} U U^\transpose$'s \sti{} transform is
\begin{align}
    s_p(z) = \trace{\frac{1}{b T} U U^\transpose - z I}^{-1} / p = b T / p \trace{U U^\transpose - z b T I}^{-1}.
\end{align}
% By \sti{} continuity theorem \citep{RMT_book}, it is sufficient to show that $s_p(s) \asto s(z)$ for all $z \in \positivecomplex$, where $s$ is the \sti{} transform of Marchenko-Pastur distribution with parameter $p$ and $n$. To this end, typical steps include the following steps \citep{mp_proof_sketch}:
% \begin{itemize}
    % \item $s_p(z) - \ex{s_p(z)} \asto 0$, by a martingale argument;
    % \item $\ex{s_p(z)} \to s(z)$.
% \end{itemize}

\NewDocumentCommand{\boundedby}{m}{\varXi\left(#1\right)}
To ease presentation, we define $\boundedby{g}$ to indicate (complex) functions whose magnitudes are bounded by positive real function $g$, i.e.,
\begin{align}
    h \in \boundedby{g} \iff \forall x, y, \abs{h(x, y)} \le g(x),
\end{align}
where $y$ indicates variables other than $x$ that $h$ relies.
$\boundedby{\cdot}$ will be used combined with ``$=$'' imitating $O(\cdot)$. Since $\boundedby{\cdot}$ does not hide constant scaling factors and biases in it, unlike $O(\cdot)$ it can be freely added, averaged, multiplied and divided, i.e.,
\begin{align}
    \boundedby{g_1} + \boundedby{g_2} \in& \boundedby{g_1 + g_2},
    \frac{1}{n} \sum_{i=1}^n \boundedby{g_i} \in \boundedby{\frac{1}{n}\sum_{i=1}^n g_i},\\
    \boundedby{g_1} \cdot \boundedby{g_2} \in& \boundedby{g_1 \cdot g_2},
    \frac{\boundedby{g_1}}{g_2} \in \boundedby{\frac{g_1}{g_2}}.
\end{align}

Consistent with final conclusion, fix $z = 0 + v i$ ($v \in \reals^+$) such that $v \ge v_0$ throughout the proof.
Define $A^p \defeq \sum_{k} \uut$.
Sample an auxiliary vector $u_{T, b+1} = u_{b T  + 1} \in \reals^p$ so that it is sampled from the conditional distribution given the first $T-1$ batches but it is conditionally independent with other samples in $U^T$, i.e., an extra sample for the last batch. This dependence relation can be expressed by only adding edges $U^{1: T-1} \to u_{T, b+1}$ to the SCMs of the Batch Dependence Model. With the auxiliary vector, define $B^p \defeq A^p +  \uut[T, b+1]$. %Define $C^b_t \defeq B^p - u^t \left(u^t\right)^\transpose$

By \cref{lemma:3.1_from_mp_quadratic_form}(1), $B^p - z b T I$ is non-degenerate and
\begin{align}
    p 
    =&  \trace{\left(B^p - z b T I\right) \left(B^p - z b T I\right)^{-1}}\\
    =&   \sum_{t=1}^{T} \sum_{l=1}^{b + \indic{t = T}} u_{t, l}^\transpose \left(B^p - z b T I\right)^{-1} u_{t, l} - z b T \trace{B^p - z b T I}^{-1}.
\end{align}
Taking expectations and using the exchangeability within each batch give
\begin{align}
    p = \sum_{t=1}^{T} (b + \indic{t = T}) \ex{u_t^\transpose \left(B^p - z b T I\right)^{-1} u_t} - z b T \ex{\trace{B^p - z b T I}^{-1}} \label{eq:a3}.
\end{align}

Define $S_p(z) \defeq \trace{A^p - z b T I}^{-1}$ and note that $S_p(z) = (p / b T) s_p(z)$. 

By \cref{lemma:3.1_from_mp_quadratic_form}(2), there is
\begin{align}
    \ex{\trace{B^p - z b T I}^{-1}} =& \ex{S_p(z)} + \boundedby{1/ v b T} = \ex{S_p(z)} + \boundedby{c / v p} \label{eq:a1}.
\end{align}

\NewDocumentCommand{\approxmatrix}{O{\boundedby} O{2}}{#1{\frac{#2 \sqrt{2} c \alpha}{v p}}}
We now prove
\begin{align}
    \frac{1}{T} \sum_{t=1}^{T} \ex{u_t^\transpose \left(B^p - z b T I\right)^{-1} u_t} = \frac{\ex{S_p(z)}}{1 + \ex{S_p(z)}} + t \label{eq:claim}, 
\end{align}
where $\abs{t}$ is bounded by a function of $c, \alpha, \beta, v, p$.

\NewDocumentCommand{\approxfunction}{O{\boundedby}}{#1{1}}
A complex function $\frac{x}{1 + x} = 1 - \frac{1}{x + 1}$ emerges many times. We will approximate it to the first order so its complex derivative should be computed and bounded.
\begin{align}
    \abs{\left(\frac{x}{1 + x}\right)'}
    =&  \abs{\frac{1}{(x+1)^2}} 
    = \frac{1}{\abs{x + 1}^2}
\end{align}
Therefore, if $x_1, x_2$ both stay away from $-1$, then $\abs{\left(\frac{x'}{1 + x'}\right)'} = \approxfunction$ on the line connecting $x_1, x_2$ and we can approximate $\frac{x_2}{1 + x_2}$ by $\frac{x_1}{1 + x_1} + \boundedby{1} \cdot \Delta x = \frac{x_1}{1 + x_1} + \boundedby{\Delta x}$, where $\Delta x = x_2 - x_1$.
In latter application, $x$, both the start and the end of approximation, is often of form $\frac{1}{n} \sum_{i=1}^n \ex{u_i^\transpose \left(C - z b T I\right)^{-1} u_i}$ possibly with averaging or expectation missing, where $C$ is real symmetric positive semi-definite and $u_i$ is a real vector. The eigenvalues in $\left(C - z b T I\right)^{-1}$ are
\begin{align}
    \frac{1}{\lambda_i(C) - v b T i}
    =&  \frac{\lambda_i(C) + v b T i}{\lambda_i(C)^2 + (v b T)^2},
\end{align}
whose real part is
\begin{align}
    \rpart{\frac{1}{\lambda_i(C) - v b T i}}
    =&  \frac{\lambda_i(C)}{\lambda_i(C)^2 + (v b T)^2} \ge 0.
\end{align}
As a result, the real part of inner products is always non-negative and $x$ stays away from $-1$, and the magnitude of derivatives is $\approxfunction$.

Another approximation is done between $C^p_k$ and $A_p$, whose difference is the outer products of a constant number of random vectors, and it should be minor considering there are $b T$ of them. Formally, for real symmetric positive semi-definite $C$ with eigenvalue decomposition $C = V \Lambda V^\transpose$ by real matrices $V$ and $\Lambda$, $(C - z I)$ can be decomposed to $(C - z I) = V \left(\Lambda - z I\right) V^\transpose$, and non-degenerate $\left(C - z I\right)^{-1}$ to $\left(C - z I\right)^{-1} = V \left(\Lambda - z I\right)^{-1} V^\transpose \defto V \Sigma \Sigma V^\transpose$ where $\Sigma \defeq \sqrt{\left(\Lambda - z I\right)^{-1}}$. Let $S \defeq V \Sigma V^\transpose$ to have $S^\transpose S = S S = \left(C - z I\right)^{-1}$. After that, there is
\begin{align}
    &   \abs{y^\transpose \left(C + x x^\transpose - z I\right)^{-1}y - y^\transpose \left(C - z I\right)^{-1} y}
    =  \abs{y^\transpose \left(\left(C + x x^\transpose - z I\right)^{-1} - \left(C - z I\right)^{-1} \right) y}\\
    =&  \abs{\frac{
            y^\transpose \left(C - z I\right)^{-1} x x^\transpose \left(C - z I\right)^{-1} y
        }{1 + x^\transpose \left(C - z I\right)^{-1} x}}
    =   \abs{\frac{
            \left(y^\transpose S^\transpose S x\right) \left(x^\transpose S^\transpose S y\right)
        }{1 + x^\transpose \left(C - z I\right)^{-1} x}}
    =  \abs{\frac{
            \left(a^{\transpose} \bar{b}\right) \left(\bar{b}^\transpose a\right)
        }{1 + x^\transpose \left(C - z I\right)^{-1} x}}\\
    =&   \frac{
            \abs{a^* b} \abs{a^* b}
        }{\abs{1 + x^\transpose \left(C - z I\right)^{-1} x}}
    \le \frac{
            \norm{a^*}_2 \norm{b}_2 \norm{a^*}_2 \norm{b}_2
        }{\abs{1 + x^\transpose \left(C - z I\right)^{-1} x}}
    =   \frac{
            \abs{a^* a} \abs{b^* b}
        }{\abs{1 + x^\transpose \left(C - z I\right)^{-1} x}}\\
    =&  \frac{
            \abs{\trace{y y^\transpose S^* S}} \abs{b^* b}
        }{\abs{1 + x^\transpose \left(C - z I\right)^{-1} x}}
    \le \frac{
            \norm{y y^\transpose S^* S}_1 \abs{b^* b}
        }{\abs{1 + x^\transpose \left(C - z I\right)^{-1} x}}
    \le \frac{
            \norm{y y^\transpose}_1 \norm{S^* S}_\infty \abs{b^* b}
        }{\abs{1 + b^\transpose b}}
    =   \frac{\norm{y}_2^2}{\ipart{z}} \frac{
            \abs{b^* b}
        }{\abs{1 + b^\transpose b}},
\end{align}
where $a \defeq S y, b \defeq \bar{S} \bar{x} = \bar{S} x$, the second step is from Sherman-Morrison formula, and the second last inequality is due to \cref{lemma:abs_trace_and_schatten_1}. The fact, that $S^* S$ is positive semi-definite whose largest eigenvalue is smaller than the upperbound $\frac{1}{v}$ of $S^\transpose S$'s eigenvalue magnitude, is also used.  To bound the fraction between $\abs{b^* b}$ and $\abs{1 + b^\transpose b}$, recall the eigenvalue decomposition on $\left(C - z I\right)^{-1}$ 
\begin{align}
    \left(C - z I\right)^{-1} = V \Sigma \Sigma^\transpose V^\transpose
\end{align}
and $S = V \Sigma^\transpose V^\transpose$. Then 
\begin{align}
    b^\transpose b &= v^\transpose \Sigma \Sigma v,
    b^* b = v^\transpose \bar{\Sigma} \Sigma v,
\end{align}
where $v \defeq V^\transpose x$ is a real vector. Notice that $\Sigma \Sigma = \diag{\frac{1}{\lambda_i(C) - v i}} = \diag{\frac{\lambda_i(C) + v i}{\lambda_i(C)^2 + v^2}}$ where both real and imaginary parts are non-negative, and that $\Sigma^* \Sigma = \diag{\frac{\abs{\lambda_i(C) + v i}}{\lambda_i(C)^2 + v^2}}$. With this, the inner products are simplified to
\begin{align}
    b^\transpose b &= \sum_{i} \frac{v_i^2 \lambda_i(C)}{\lambda_i(C)^2 + v^2} + i \sum_{i} \frac{v_i^2 v}{\lambda_i(C)^2 + v^2},
    b^* b = \sum_{i} \abs{\frac{v_i^2}{\lambda_i(C)^2 + v^2} \lambda_i(C) + i \frac{v_i^2 v}{\lambda_i(C)^2 + v^2}}
\end{align}
Representing complex numbers by 2-dimensional vectors $w_i \defeq \begin{bmatrix} \frac{v_i^2 \lambda_i(C)}{\lambda_i(C)^2 + v^2} & \frac{v_i^2 v}{\lambda_i(C)^2 + v^2} \end{bmatrix}^\transpose$, there are
\begin{align}
    \abs{b^\transpose b} &= \norm{\sum_i w_i}_2,
    \abs{b^* b} = \sum_i \norm{w_i}_2.
\end{align}
Noting that all entries of $w_i$'s are non-negative, there is
\begin{align}
    \abs{b^* b}
    &=  \sum_i \norm{w_i}_2
    \le \sum_{i} \norm{w_i}_1 
    =   \norm{\sum_i w_i}_1 
    \le \sqrt{2} \norm{\sum_i w_i}_2 = \sqrt{2} \abs{b^\transpose b}.
\end{align}
So $\frac{\abs{b^* b}}{\abs{b^\transpose b}} \le \sqrt{2}$. Given that the real part of $b^\transpose b$ is non-negative, adding $1$ will only increase its magnitude. As a result, there is
\begin{align}
    &   \abs{y^\transpose \left(C + x x^\transpose - z I\right)^{-1}y - y^\transpose \left(C - z I\right)^{-1} y}
    \le  \frac{\sqrt{2} \norm{y}_2^2}{\ipart{z}},
\end{align}

When $z b T$ is substituted, there is
\begin{align}
    &   \abs{y^\transpose \left(C + x x^\transpose - z b T I\right)^{-1}y - y^\transpose \left(C - z b T I\right)^{-1} y}
    =  \frac{\sqrt{2} \norm{y}_2^2}{v b T}.
\end{align}
In later use, $y$ is instantiated by $u_k$ and there is $\norm{u_k}_2^2 = a \norm{x}_2^2 \le \alpha$ for any $t$, so by assumption the approximation error is always bounded by 
\begin{align}
    \abs{u_k^\transpose \left(C + x x^\transpose - z b T I\right)^{-1} u_k - u_k^\transpose \left(C - z b T I\right)^{-1} u_k} \le \approxmatrix[\boundedby][].
\end{align}

\NewDocumentCommand{\innersum}{}{\frac{1}{b T}\sum_{k}}
\NewDocumentCommand{\outersum}{}{}
\NewDocumentCommand{\innerapproximator}{O{\left(A^p - z b T I\right)^{-1}} O{k}}{ u_{#2}^\transpose #1 u_{#2} }
\NewDocumentCommand{\innerouterproduct}{O{}}{u_{k} #1 u_{k}^\transpose}
\NewDocumentCommand{\biasapproximator}{}{\frac{\ex{\innersum \innerapproximator}}{1 + \ex{\innersum \innerapproximator}}}
\NewDocumentCommand{\diffapproximator}{}{\innerapproximator - \ex{\innersum \innerapproximator}}
With these two approximation techniques, we first approximate the LHS of \cref{eq:claim}.
To this end, let $C^p_k \defeq B^p - u_k u_k^\transpose$ and by Sherman-Morrison formula there is
\begin{align}
    &   u_k^\transpose \left(B^p - z b T I\right)^{-1} u_k
    =   u_k^\transpose \left(C^p_k + u_k u_k^\transpose - z b T I\right)^{-1} u_k\\
    =&  u_k^\transpose \left(\left(C^p_k - z b T I\right)^{-1} - \frac{\left(C^p_k - z b T I\right)^{-1} u_k u_k^\transpose \left(C^p_k - z b T I\right)^{-1}}{1 + u_k^\transpose\left(C^p_k - z b T I\right)^{-1} u_k}\right) u_k\\
    =&  \frac{u_k^\transpose \left(C^p_k - z b T I\right)^{-1}u_k}{1 + u_t^\transpose\left(C^p_k - z b T I\right)^{-1} u_k}
    =  \frac{u_k^\transpose \left(A^p - z b T I\right)^{-1}u_k}{1 + u_k^\transpose\left(A^p - z b T I\right)^{-1} u_k} + \approxfunction \cdot \approxmatrix.
\end{align}
After that, there is
\begin{align}
    &   \frac{1}{T} \sum_{t=1}^T \ex{u_t^\transpose (B^p - z b T I)^{-1} u_t}\\
    =&  \frac{1}{T} \sum_{t=1}^T \frac{1}{b} \sum_{l=1}^b \ex{\frac{u_{t, l}^\transpose \left(A^p - z b T I\right)^{-1} u_{t, l}}{1 + u_{t, l}^\transpose \left(A^p - z b T I\right)^{-1} u_{t, l}}} + \approxfunction \cdot \approxmatrix\\
    =&  \outersum \innersum  \ex{\biasapproximator} + \approxmatrix \\
        &+ \outersum \innersum \ex{\approxfunction \abs{\diffapproximator}}\\
    =&  \outersum \biasapproximator + \approxmatrix\\ 
        &+ \approxfunction  \outersum \innersum \ex{\abs{\diffapproximator}}\\
    =&  \outersum \biasapproximator \label{eq:term1}\\ 
        &+ \approxfunction  \ex{\abs{u_r^\transpose \left(A^p - z b T I\right)^{-1} u_r - \ex{u_r^\transpose \left(A^p - z b T I\right)^{-1} u_r}}}  \label{eq:term2} \\&+ \approxmatrix,
\end{align}
where $r$ in the last line is a uniformly randomly selected index from $\set{1, \dots, b T}$ independently to the training process.

\NewDocumentCommand{\approxbias}{O{\boundedby}}{#1{\frac{c \beta}{v p}}}
Note that $S_p(z) = \trace{A^p - z b T I}^{-1}, \ex{S_p(z)} = \ex{\trace{A^p - z b T I}^{-1}}$. So for the term in \cref{eq:term1} we proceed by proving $\frac{1}{b} \sum_{l=1}^{b} \ex{u_{t, l}^\transpose \left(A^p - z b T I\right)^{-1} u_{t, l}}$ approximates $\ex{\trace{A^p -z b T I}^{-1}}$. For convenience let $D \defeq b T\left(A^p - z b T I\right)^{-1}$ be an alias to it, whose spectral norm satisfies $\norm{D}_{\infty} \le b T \frac{1}{ v b T} = \frac{1}{v}$, then
\begin{align}
    &       \abs{\ex{\innersum \innerapproximator} - \ex{S_p(z)}}\\
    =&      c\abs{\frac{\ex{\innersum \innerapproximator[b T \left(A^p - z b T I\right)^{-1}]}}{p} - \frac{ \ex{b T S_p(z)}}{p}}\\
    =&      \frac{c}{p} \abs{\innersum \ex{\innerapproximator[D]} - \ex{\trace{D}}}
    =      \frac{c}{p} \abs{\ex{\trace{\left(\innersum \innerouterproduct - I\right) D}}}\\
    \le&    \frac{c}{p} \ex{\norm{\left(\innersum \innerouterproduct - I\right) D}_1}
    \le    \frac{c}{p} \ex{\norm{\left(\innersum \innerouterproduct - I\right)}_1 \norm{D}_{\infty}}\\
    \le&    \frac{c}{v p}\ex{\norm{\left(\innersum \innerouterproduct - I\right)}_1}
    \le    \frac{c}{v p}  \ex{\sqrt{p \trace{\left(T^p_t - I\right)^\transpose \left(T^p_t - I\right)}}}\\
    =&      \approxbias,
\end{align}
where the last inequality is because
\begin{align}
    \norm{A}_1 =& \sum_{i=1}^{p} \abs{\lambda_i(A)} = \norm{\begin{bmatrix}
        \lambda_1(A) & \cdots & \lambda_i(A) & \cdots \lambda_p(A)
    \end{bmatrix}^\transpose}_1\\
    \le&    \sqrt{p} \norm{\begin{bmatrix}
        \lambda_1(A) & \cdots & \lambda_i(A) & \cdots \lambda_p(A)
    \end{bmatrix}^\transpose}_2\\
    =&  \sqrt{p} \norm{A}_2 = \sqrt{p \trace{A^\transpose A}},
\end{align}
given that $A = \left( \left(\innersum \innerouterproduct - I\right) \right)$ is symmetric so that its singular values are absolute eigenvalues.
With approximation on complex function $\frac{x}{1 + x}$, this $\approxbias$-boundedness implies $\approxfunction \cdot \approxbias = \approxbias$ approximation of in \cref{eq:term1}.

\NewDocumentCommand{\approxvariance}{O{\boundedby}}{#1{\frac{c \alpha}{v p}}}
\NewDocumentCommand{\utdu}{}{u_t^\transpose D u_t}
\NewDocumentCommand{\xtdx}{}{x_r^\transpose D x_r}
For the difference term in \cref{eq:term2}, we prove its diminishment by $\frac{1}{v}$-bounded variance of $u_{t, l}^\transpose (A^p - z b T I)^{-1} u_{t, l}$, or formally
\begin{align}
    \ex{\abs{X - \ex{X}}^2} - \ex{\abs{X - \ex{X}}}^2 =& \ex{\left(\abs{X - \ex{X}} - \ex{\abs{X - \ex{X}}}\right)^2} \ge 0\\
    \ex{\abs{X - \ex{X}}} \le&  \sqrt{\ex{\abs{X - \ex{X}}^2}} = \sqrt{\var{X}},
\end{align}
and 
\begin{align}
    &   \var{\innerapproximator[\left(A^p - z b T I\right)^{-1}][r]}
    =  \frac{c^2}{p^2} \var{\innerapproximator[D][r]}\\
    =&  \frac{c^2}{p^2} \left(\ex{\trace{\innerapproximator[D][r] \innerapproximator[D][r]}} - \ex{\trace{\innerapproximator[D][r]}}^2\right)
    \le    \frac{c^2}{p^2} \alpha^2 \left(\ex{\trace{\xtdx \xtdx}}\right)\\
    \le&   \frac{c^2 \alpha^2}{p^2} \ex{\trace{\xtdx \xtdx}}
    \le    \frac{c^2 \alpha^2}{p^2} \ex{\norm{x^p}_2^4 \norm{D}_\infty^2}
    =  \frac{c^2 \alpha^2}{v^2 p^2},
\end{align}
where the last step follows that $x^p$'s norm is bounded and that $\norm{D}$ is also uniformly bounded. 

\NewDocumentCommand{\approxlargest}{O{\boundedby}}{#1{\approxmatrix[] + \approxbias[] + \approxvariance[]}}
To sum up, we have obtained 
\begin{align}
    \frac{1}{T} \sum_{t=1}^T \ex{u_t^\transpose \left(B^p - z b T I\right)^{-1} u_t} = \frac{\ex{S_p(z)}}{1 + \ex{S_p(z)}} + t,
\end{align}
where $\abs{t} = \approxlargest$.

\NewDocumentCommand{\approxall}{O{\boundedby}}{#1{\approxlargest[] + \frac{c}{v p} + \frac{c \alpha }{v p}}}

With \cref{eq:claim}, \cref{eq:a1}, one can reduce \cref{eq:a3} to
\begin{align}
    p =& T (b + O(1)) \left(\frac{\ex{S_p(z)}}{1 + \ex{S_p(x)}} + t\right) - z b T \left(\ex{S_p(z)} + \boundedby{c / v p}\right),
\end{align}
and
\begin{align}
    \frac{\ex{S_p(z)}}{1 + \ex{S_p(x)}} - z \ex{S_p(z)} =& \frac{p}{b T} + s = c + s,
\end{align}
where $s = \approxall$.

$\ex{S_p(p)}$ always have a non-negative real part because real parts of eigenvalues of $\left(U U^\transpose - z b T I\right)$ are always non-negative by an argument similar to previous ones. Since $\alpha, \beta, c$ and $p$ depend only on $p, b, T$ instead of $v$, $s = \approxall$ is bounded by $\tau / v$ which satisfies $\tau$ is constant w.r.t $v$, $\frac{v_0 + (1 - c)}{\sqrt{2}} > \frac{\tau}{v_0} + 2 \sqrt{c v_0} + 2 \sqrt{\tau}$ and $v \ge v_0 \ge 2 c$. Given $c \in [0, 1]$, by \cref{lemma:bound_of_sti}, there is
\begin{align}
    &   \abs{\ipart{\ex{S_p(v i)}}} \le \abs{\ex{S_p(v i)}} \\
    \le& \approxquadratic{\approxall[]}
\end{align}
for $v \ge v_0$. Bounds using the real part are similarly obtained.

The similar bound for $s_p(z) = (b T / p) S_p(z) = \frac{1}{c} S_p(z)$ is
\begin{align}
    \abs{\ipart{\ex{s_p(v i)}}}
    \le& \frac{1}{c}\approxquadratic{\approxall[]}.
\end{align}

The expected mean $\ex{\overline{\lambda^{ > 0}}}$ of $\frac{1}{b T} U^p \left(U^p\right)^\transpose$'s non-zero eigenvalue is
\begin{align}
    &   \ex{\frac{\trace{\frac{1}{b T} U^p \left(U^p\right)^\transpose}}{\min(b T, p)}}
    =  \frac{\frac{1}{b T}\sum_{k=1}^{b T}\ex{\trace{u_{k} u_k^\transpose}}}{\min(b T, p)}
    =  \frac{\frac{1}{b T}\sum_{k=1}^{b T}\ex{\norm{u_k}_2^2}}{\min(b T, p)}
    \le   \frac{\alpha}{\min(b T, p)}.
\end{align}

Finally, the desired conclusion is obtained through \cref{lemma:sti_and_eigenvalue_ratio} by
\begin{align}
        \ex{\frac{\ex{\overline{\lambda^{>0}} / \sqrt{v}}^2}{\left(\frac{\lambda}{\sqrt{v}}\right)^2 + v}}
    \le    \frac{1}{v} \ex{\overline{\lambda^{>0}}}^2 \abs{\ipart{s_p(v i)}},
    % \le&    \frac{1}{v c} \frac{\alpha^2}{\min(b T, p)^2}  \approxquadratic{\approxall[]},\\
        \ex{\frac{\ex{\overline{\lambda^{>0}}}}{\lambda + \frac{v^2}{\lambda}}}
    \le    \ex{\overline{\lambda^{>0}}} \abs{\rpart{s_p(v i)}}.
    % \le&    \frac{1}{c} \frac{\alpha}{\min(b T, p)}  \approxquadratic{\approxall[]}.
\end{align}

\end{proof}
}

%% file: algo/refined.tex
\section{Refined Zeroth Biases}\label{sec:refine}

After the analyses in \cref{sec:theory}, the use of zeroth biases proposed in \cref{sec:illustration} can be further refined. 
\cref{theorem:main_with_hidden_vectors_and_layernorm} and \cref{theorem:main_with_effective_duplication} emphasize the role of LayerNorm layers in sparsity. To push sparsity further, it can be beneficial to enforce a lowerbound of elementwise scaling factors of LayerNorms.

The two theorems also hint at conflicts of LayerNorm with zeroth biases and elementwise biases in LayerNorm. 
Therefore, one can turn off the elementwise biases in LayerNorm layers and then confine the norm of columns in zeroth biases. 
We choose to simply clamp the parameter of zeroth biases. To utilize the potential rising of elementwise scaling factors, the upperbound of absolute values of zeroth biases are determined relatively to the elementwise scaling factors.
Putting things together, we execute \cref{algo:refined} after every update during training from scratch.

\begin{algorithm}
    \caption{Refined LayerNorm and Zeroth Biases}\label{algo:refined}
    \NewDocumentCommand{\listofLayerNorms}{}{\mathcal{L}}
    \NewDocumentCommand{\listofZerothBiases}{}{\mathcal{D}}
    \begin{algorithmic}[1]
        \Require $\listofZerothBiases$: the list of all zeroth biases; $\listofLayerNorms$: the list of \emph{unbiased} LayerNorm layers before zeroth biases, ordered so that $\listofZerothBiases[i]$ lies right after $\listofLayerNorms[i]$; $c$: the factor that controls zeroth biases relatively to the LayerNorm scaling factors
        \Procedure{RestrictLayerNorm}{$\listofLayerNorms$} 
            \For{$L \in \listofLayerNorms$}
                \State $L.\text{weight} \gets \text{clamp}(L.\text{weight},\text{min}=1.0)$
            \EndFor
        \EndProcedure
        \\
        \Procedure{RestrictZerothBiases}{$\listofZerothBiases, \listofLayerNorms, c$}\label{algo:restricted-zb}
            \For{$l = 0 \dots \size{\listofZerothBiases} - 1$}
                \State $D \gets \listofZerothBiases[l]$
                \State $L \gets \listofLayerNorms[l]$
                \State $s \gets c * \text{unsqueeze}(\text{abs}(L.\text{weight}), \texttt{dim}=-2)$ 
                \State \Comment{In Transformers, scaling factors in LayerNorms are vectors.}
                \State \Comment{It needs to be expanded along the token-stacking dimension}
                \State \Comment{to adapt the shape of $D$}
                \State $D.\text{biases} \gets \text{clamp}(D.\text{biases}, \text{elementwise-min}=-s, \text{elementwise-max}=s)$
            \EndFor
        \EndProcedure
    \end{algorithmic}
\end{algorithm}

In sparsity finetuning, it is not likely that the loaded checkpoints have scaling-restricted and unbiased LayerNorms so we freeze LayerNorms' biases and gradually uplift the absolute value of scaling factors while retaining their signs after each update according to \cref{algo:refined_finetuning} instead of \mathsc{RestrictLayerNorm}. It is also possible that changing to $\jrelu$ without any adaptation harms the performance, so we mix $\jrelu$ and $\relu$ linearly and increase the portion of $\jrelu$ linearly after every step, starting from $0$. \cref{sec:f_experiments} shows that the refinement performs well during finetuning.
\begin{algorithm}
    \caption{Refined LayerNorm in finetuning for sparsity in training from scratch}\label{algo:refined_finetuning}
    \NewDocumentCommand{\listofLayerNorms}{}{\mathcal{L}}
    \NewDocumentCommand{\listofZerothBiases}{}{\mathcal{D}}
    \begin{algorithmic}[1]
        \Require $\listofLayerNorms$: the list of \emph{unbiased} LayerNorm layers before zeroth biases; $T_{\text{uplifting}}$: the number of steps of uplifting; $t$: the index of the current step.
        \Procedure{UpliftLayerNorm}{$\listofLayerNorms, T_{\text{uplifting}}, t$}
            \For{$L \in \listofLayerNorms$}
                \State $s \gets (L.\text{weight}.\text{sign}() + 0.1).\text{sign}()$ \Comment{Uplift zeros towards the positive direction}
                \State $\text{UpliftedAbsoluteValue} \gets \text{clamp}(L.\text{weight}.\text{abs}(),\text{min}= \min(t / T_{\text{uplifting}}, 1))$
                \State $L.\text{weight} \gets s * \text{UpliftedAbsoluteValue}$
            \EndFor
        \EndProcedure
    \end{algorithmic}
\end{algorithm}

%% file: experiments/validational.tex
\section{Experiments for Verification}\label{sec:v_experiments}

The experiments for verification are designed to rule out other possible explanations and show that gradient sparsity is the main cause of activation sparsity, at least in $\dbmlp$.

Although comprehensively demonstrated by \citet{observation}, the sparsity emerging in MLP blocks has still a lot of confusion. 
This is because a lot of coincidences exist in $\relu$. For example, $\relu$ and its derivatives share the same zero-valued intervals, and they are both monotonically increasing.
Previous works \citep{observation,sharpness_aware} consider that it is the value or norm of activations that are reduced to produce sparsity. There are more possible explanations, such as the absolute value of activations as well as all similar explanations based on pre-activations.
We show none of them is the essential source of observed activation sparsity but the gradient sparsity, by designing weird activation functions, using which activations or pre-activations are not consistently decreased, increased or moved towards zero, but the norm of derivatives is.

\NewDocumentCommand{\weird}{}{\operatorname{Weird-J-SquaredReLU}}
\NewDocumentCommand{\halfwidth}{}{w_{\textrm{half}}}
Our construction includes an activation function $\weird$ and 4 shifting operations. $\weird$ is defined as
\begin{align}
    \weird(x) \defeq \begin{cases}
        -\frac{1}{2} ((x + \halfwidth - 1)^2 - 1)  & x \le -\halfwidth,\\
        0 & -\halfwidth < x < \halfwidth,\\
        \frac{1}{2} ((x - \halfwidth + 1)^2 - 1)  & x \ge \halfwidth,
    \end{cases}
\end{align}
where $\halfwidth > 0$ is a hyperparameter.
\arxivonly{$\weird$ and its derivatives are illustrated in \cref{fig:weird}.}
We shift it around along $x$- or $y$-direction, i.e., use $\weird(\cdot - \Delta x) + \Delta y$ in MLP layers, to decorrelate between activations and pre-activations, values and absolute values or norms. 
If gradient sparsity is indeed essential, then the activations and pre-activations will concentrate around the flat area instead of zero point or any half-planes and will follow it well if it is shifted in another way.
To have enough space for all pre-activations, we choose $\halfwidth=1.5$, i.e., the width of the flat area is $3$, which is approximately the width of majority pre-activations in Fig. B.2(c) of \citet{observation}'s experimental observations.
To exclude flat area from zero point, we shift $\weird$ horizontally and vertically by $(\Delta x, \Delta y) \in \set{+1.6, -1.6} \times \set{+1.6, -1.6}$. This family of shifted $\weird$ is illustrated in \cref{fig:shifted_weird0-}-\cref{fig:shifted_weird-0}. From theory, we expect that most pre-activations and activations will fall around $[\Delta x - \halfwidth, \Delta x + \halfwidth] \times \set{\Delta y}$.
\arxivonly{\begin{figure}
    \centering
    \resetHeight{}
    \begin{subfigure}[h]{0.3\textwidth}
        \centering
        \myincludegraphics[width=\textwidth]{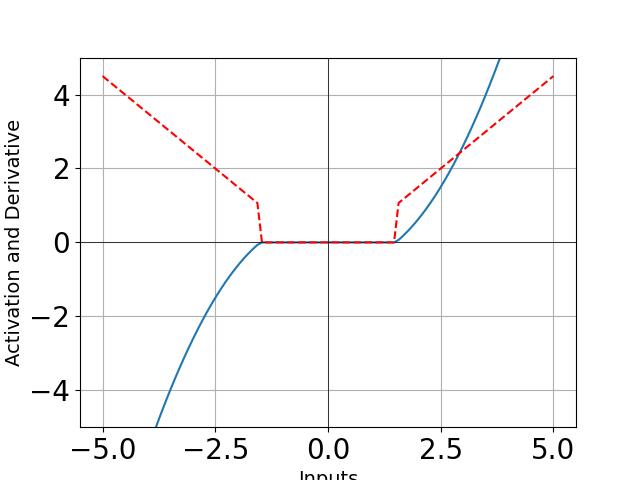}
        \caption{\scriptsize Base $\weird$.}\label{fig:weird}
    \end{subfigure}
    \begin{subfigure}[t]{0.65\textwidth}
        \centering
        \foreach \i in {,-}{
        \foreach \j in {-,}{
            \begin{subfigure}[t]{0.48\textwidth}
                \centering
                \myincludegraphics[width=\textwidth]{pic/activation/shifted_wired_jsrelu_\j1.6_\i1.6.jpg}
                \caption{\scriptsize $\Delta x = \j1.6, \Delta y = \i1.6$}\label{fig:shifted_weird\i0\j}
            \end{subfigure}
        }}
    \end{subfigure} 
    \caption{The illustration of shifted $\weird$ and their derivatives, indicated by blue lines and red dashed lines, respectively.} 
\end{figure}}
\jmlronly{\begin{figure}
    \centering
    \resetHeight{}
    \centering
    \foreach \i in {,-}{
    \foreach \j in {-,}{
        \begin{subfigure}{0.22\textwidth}
            \centering
            \myincludegraphics[width=\textwidth]{pic/activation/shifted_wired_jsrelu_\j1.6_\i1.6.jpg}
            \captionsetup{font=tiny}
            \caption{$\Delta x = \j1.6, \Delta y = \i1.6$}\label{fig:shifted_weird\i0\j}
        \end{subfigure}
    }}
    \caption{The illustration of shifted $\weird$ and their derivatives, indicated by blue lines and red dashed lines, respectively.} 
\end{figure}}

We train ViT-Base with $\dbmlp$ and $\weird$ on CIFAR-10 \citep{cifar10} for 100 epochs. Training details are listed in \cref{appendix:experimental_details}. The gradient sparsity during training, measured simply by the percentage \citep{observation} of activations not falling in $[\Delta x - \halfwidth, \Delta x + \halfwidth] \times [\Delta y - 10^{-6}, \Delta y + 10^{-6}]$, is illustrated in \cref{fig:validation_full}. 
\begin{figure}
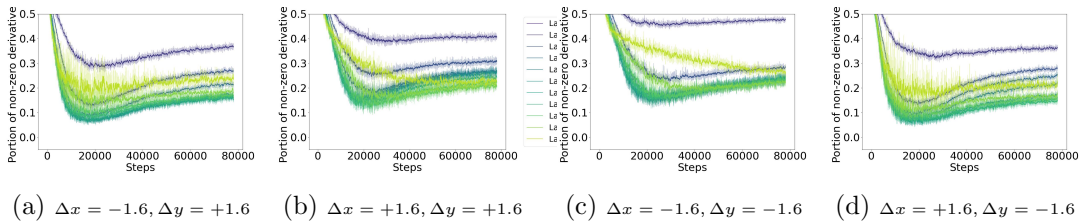

    \resetHeight{}
    \centering
    \foreach \i in {+,-}{
    \foreach \j in {-,+}{
        \begin{subfigure}[lt]{0.22\textwidth}
            \centering
            \myincludegraphics[width=\textwidth]{pic/results/dumps/validations/dx=\j1.6,dy=\i1.6/training.jpg}
            \caption{\tiny $\Delta x = \j1.6, \Delta y = \i1.6$}\label{fig:validational_full\i0\j}
        \end{subfigure}
    }}
    \caption{Gradient sparsity during training of ViT-Base with $\dbmlp$ and differently shifted $\weird$ on CIFAR-10.}\label{fig:validation_full}
\end{figure}
In \cref{fig:validation_full}, most layers have at least 70\% activations concentrating at the flat area as expected, which not only indicates that activation is currently not sparse, but also is a strong evidence that gradient sparsity dominates the sparsity among other potential factors. \cref{fig:validation_full} rules out the dominance of other potential activation or pre-activation explanations, because in \cref{fig:validational_full+0+}, \cref{fig:validational_full+0-} and \cref{fig:validational_full-0+} activation or pre-activation are not decreased, in \cref{fig:validational_full-0-}, \cref{fig:validational_full-0+} and \cref{fig:validational_full+0-} activations or pre-activations are not increased, while in all of them activations or pre-activations' norms are not optimized toward zero. Furthermore, in \cref{fig:validation_full}, where the spurious correlation between activation and gradient sparsity is broken, activation sparsity is gone while gradient sparsity survives and follows the expected area well, indicating that gradient sparsity is more stable and can be more essential than the activation one and thus deserves more research attention.

%% file: experiments/productive.tex
\section{Experiments for Productivity}\label{sec:p_experiments}

In this section, we conduct experiments with non-weird activation functions on larger data sets to examine the effectiveness our modifications and verify the theoretical guidance behind them.
Aside from modified models, vanilla ViTs are used as baselines. Since it is hard to define (or to achieve) sparsity on $\gelu$ activations \citep{observation}, even in ``vanilla'' ViTs, $\relu$ are used throughout this section.

\subsection{Training from Scratch}\label{sec:from_scratch}

Although distinguished in \cref{sec:v_experiments}, gradient and activation sparsity should better coincide to allow aggressive dynamic neuron pruning during both training and inference. 
To show their effectiveness in such practical scenarios, we test our modifications of $\dbmlp$ and non-weird $\jrelu$ by training from scratch ViT-Base/16 \citep{vit} on ImageNet-1K \citep{imagenet1k} and T5 on C4 \citep{t5}, representing computer vision and natural language processing tasks as recommended by \citet{sparsity_handbook}. 

We generally follow PyTorch's recipe\citep{pytorch_recipe} for ViT-Base/16, except that we turn off model EMA to avoid \texttt{deepcopy} that bothers experiment tracking and parameter clamping. For major hyperparameters, we use learning rate $0.03$, weight decay $0.3$, batch size $2048$, cosine annealing scheduler and a variety of random data augmentation during the 300-epoch training.
To train T5, we follow major hyperparameters released by \citet{observation} such as learning rate of $0.01$, batch size of $256$ and inverse square root scheduling with linear warmup of 10,000 steps, while filling not mentioned ones from \citet{t5_recipe}. The training of T5 lasts for only 100,000 steps as well \citep{observation}, since it involves too much computation.
To make both models sparser, we modify them by plugging zero-initialized zeroth biases before MLP blocks and replacing all $\relu$ with $\jrelu$. We run \cref{algo:refined} after each step with $c=0.1$, i.e., $\sqrt{d} - c \ge 0.9 \sqrt{d}$.
Since a lot of statistics are computed which slows the training, for ViT the logging (including training losses, sparsity, etc.) happens after per-epoch evaluation or every 100 steps in training. The encoder-decoder architecture and larger token number of T5 cost more memory use, so half of a batch is used to compute the statistics. As compensation logging happens every 25 steps for T5.
T5 and ViT are both trained for only one trial due to our limited computation budget. More details of the experiments can be found in \cref{appendix:experimental_details}.
The training and testing sparsity of ViT-Base/16 and T5-Base are illustrated in \cref{figure:productive_vit} and \cref{figure:productive_t5}, and are summarized in \cref{table:productive_vit} and \cref{table:productive_t5}.
\begin{figure}
    \centering
    \resetHeight{}
    \begin{subfigure}[t]{0.24\textwidth}
        \myincludegraphics[width=\textwidth]{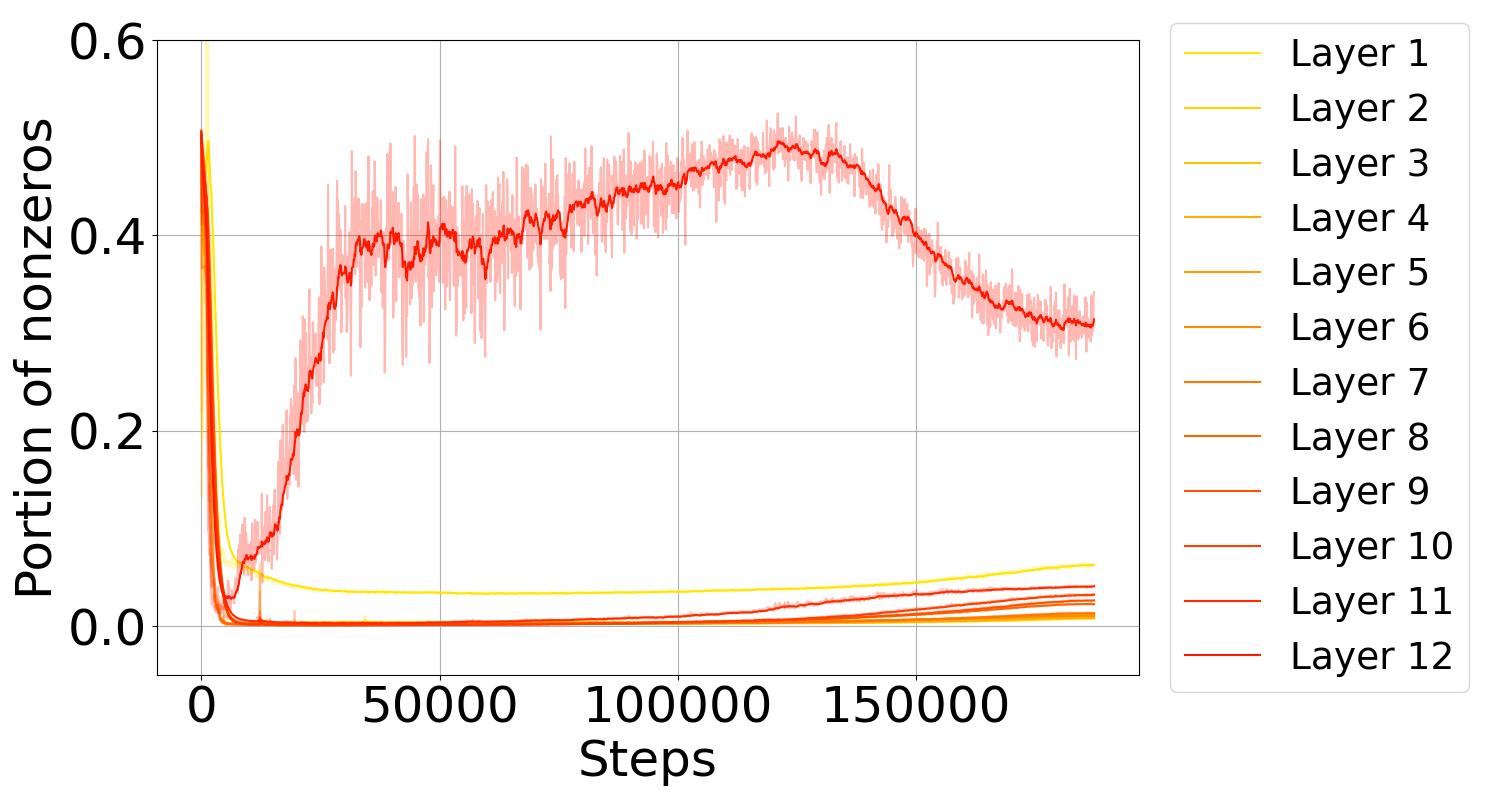}
        \caption{\tiny Training sparsity of modified ViT.}\label{figure:productive_vit_sparsified_training}
    \end{subfigure}
    \begin{subfigure}[t]{0.24\textwidth}
        \myincludegraphics[width=\textwidth]{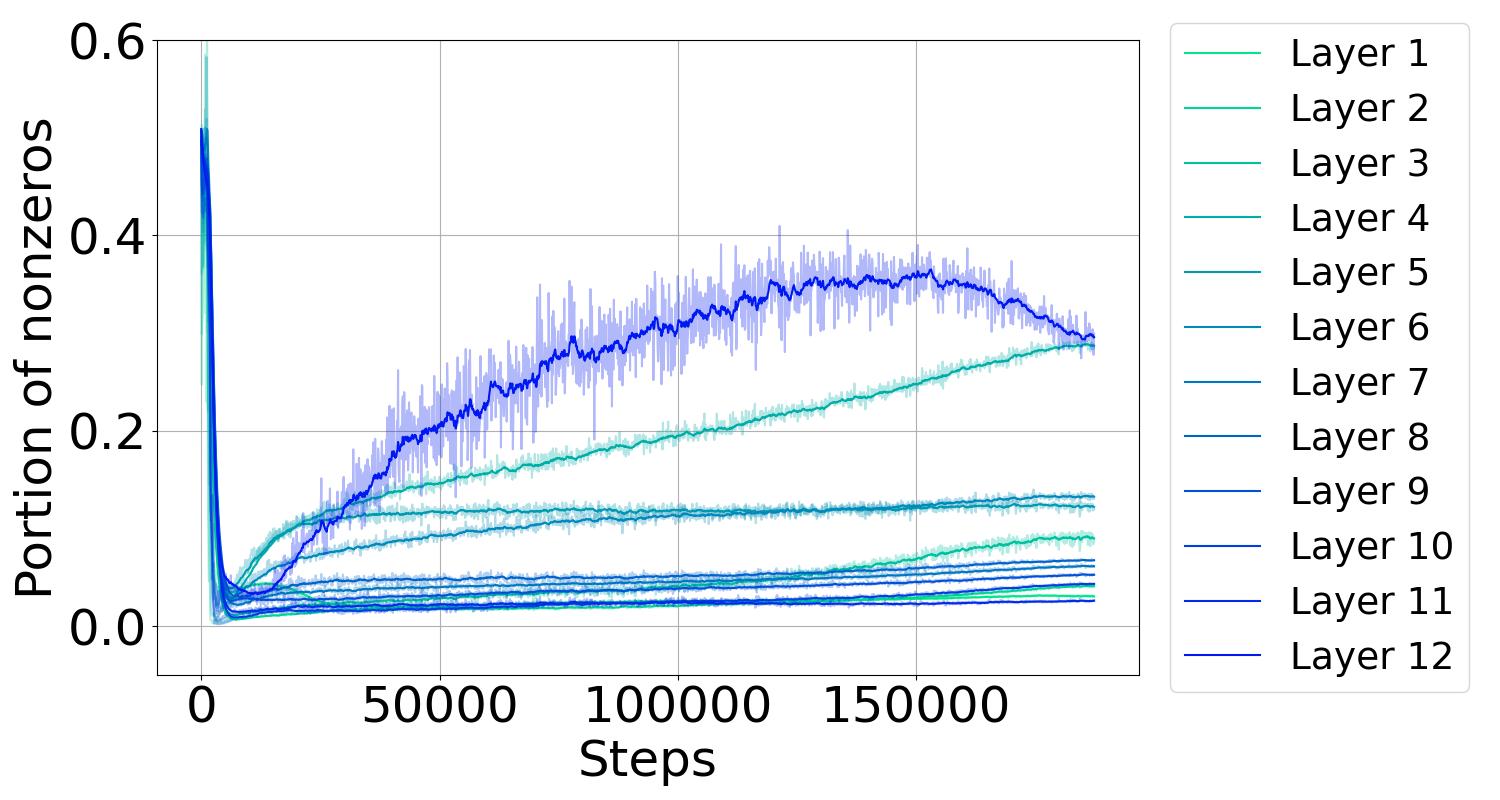}
        \caption{\tiny Training sparsity of vanilla ViT.}\label{figure:productive_vit_vanilla_training}
    \end{subfigure}
    \begin{subfigure}[t]{0.24\textwidth}
        \myincludegraphics[width=\textwidth]{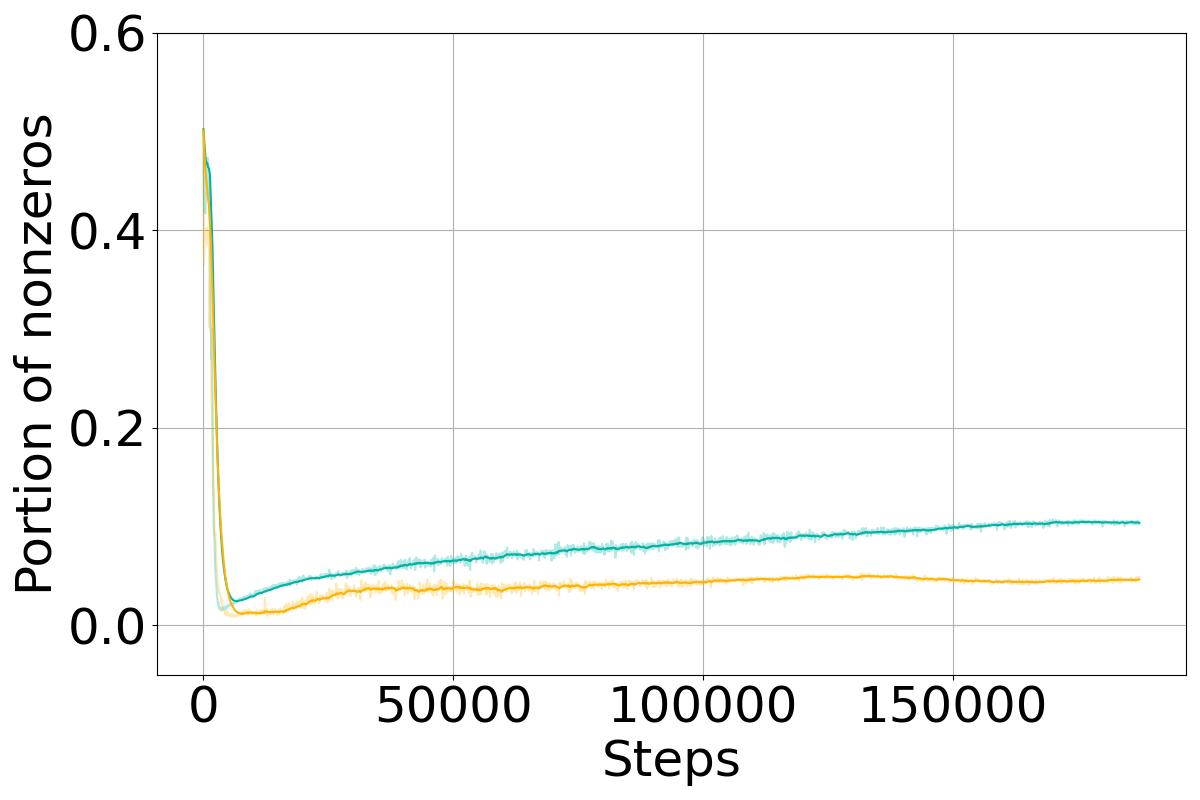}
        \caption{\tiny Comparison of layer-averaged training sparsity between vanilla and modified ViT.}\label{figure:productive_vit_average_training}
    \end{subfigure}
    \begin{subfigure}[t]{0.24\textwidth}
        \myincludegraphics[width=\textwidth]{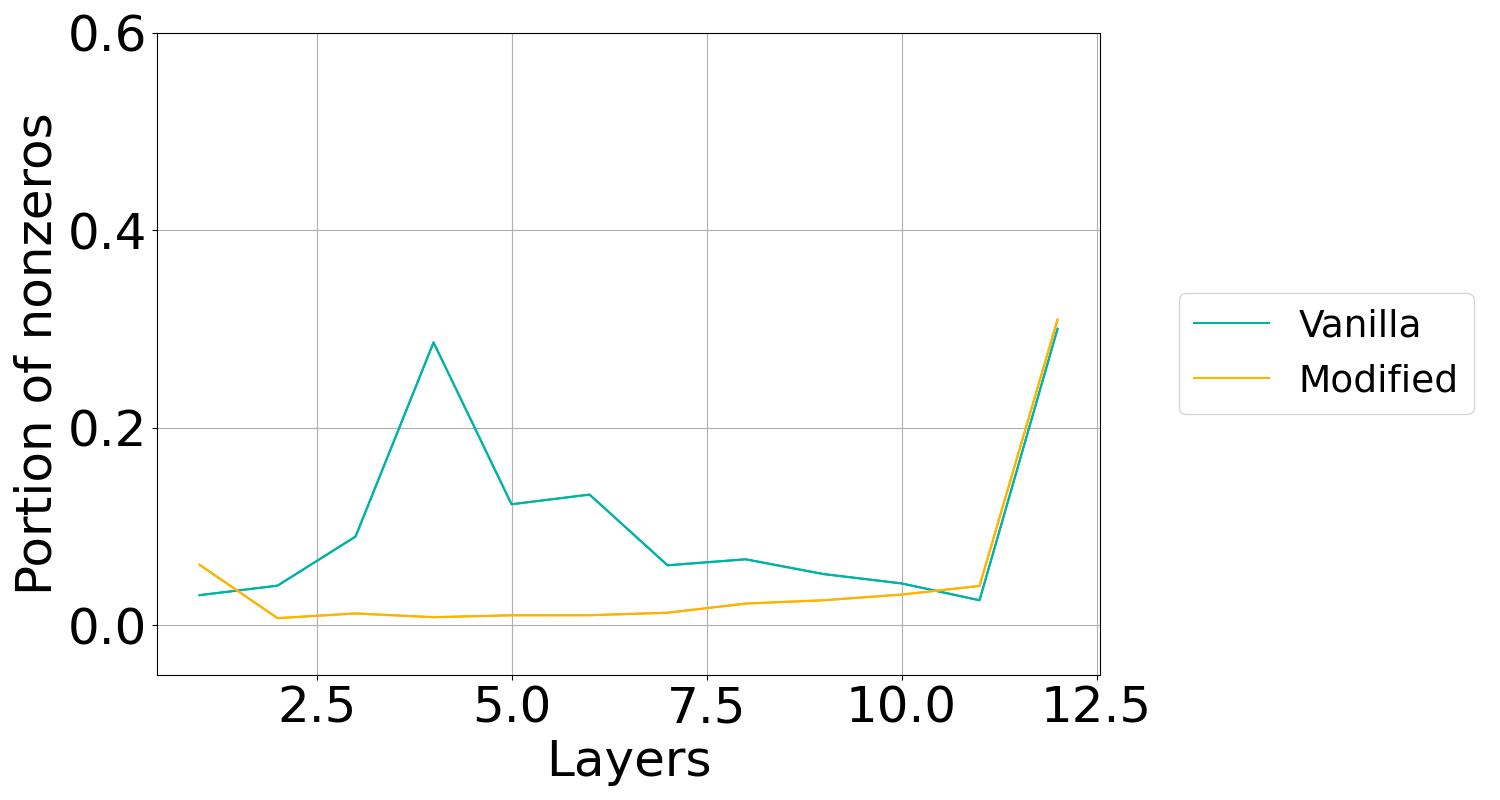}
        \caption{\tiny Comparison of layerwise training sparsity between vanilla and modified ViT during the last 100 steps.} \label{figure:productive_vit_end_training}
    \end{subfigure}
    \begin{subfigure}[t]{0.24\textwidth}
        \myincludegraphics[width=\textwidth]{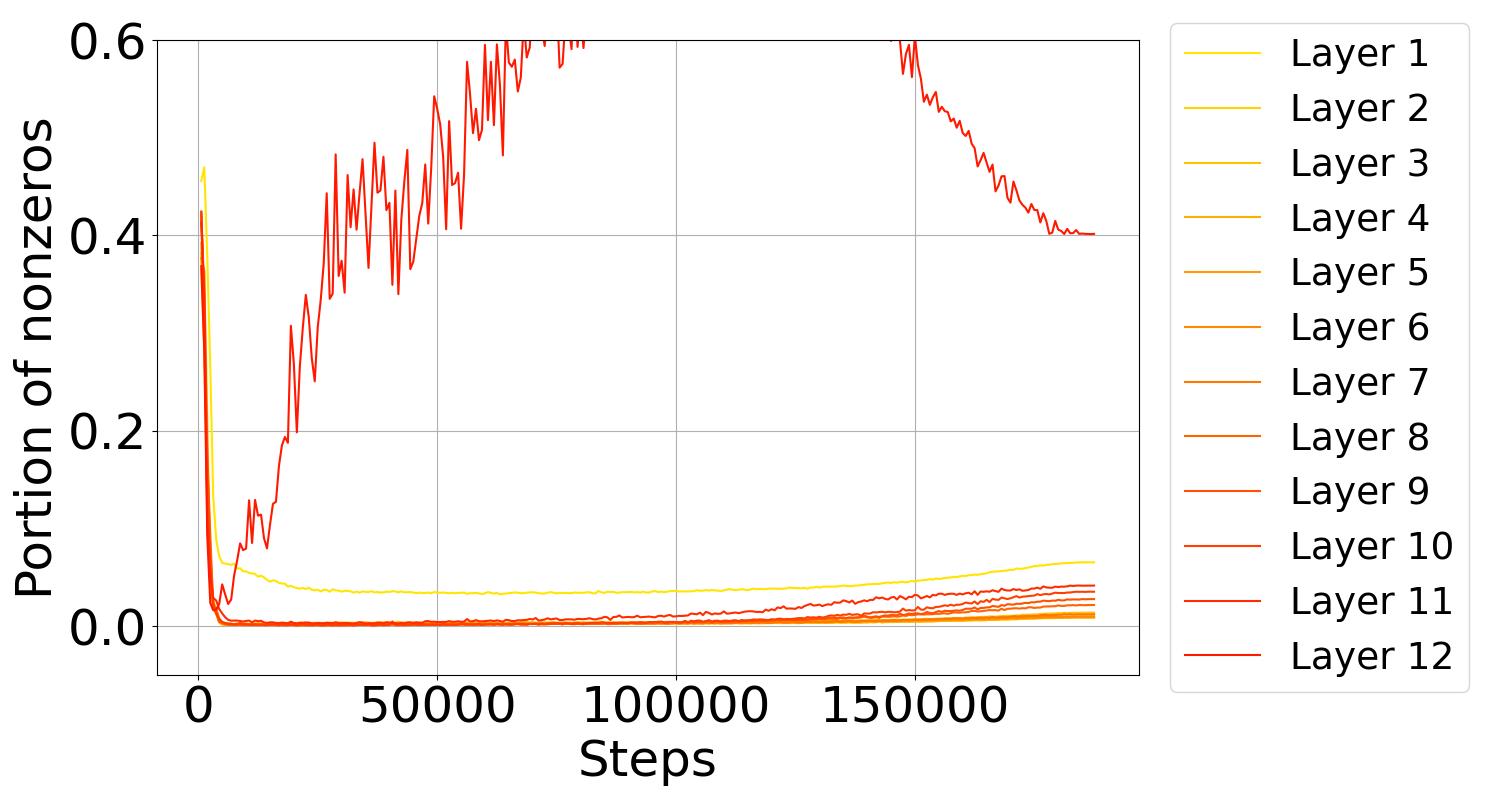}
        \caption{\tiny Testing sparsity of modified ViT.}\label{figure:productive_vit_sparsified_testing}
    \end{subfigure}
    \begin{subfigure}[t]{0.24\textwidth}
        \myincludegraphics[width=\textwidth]{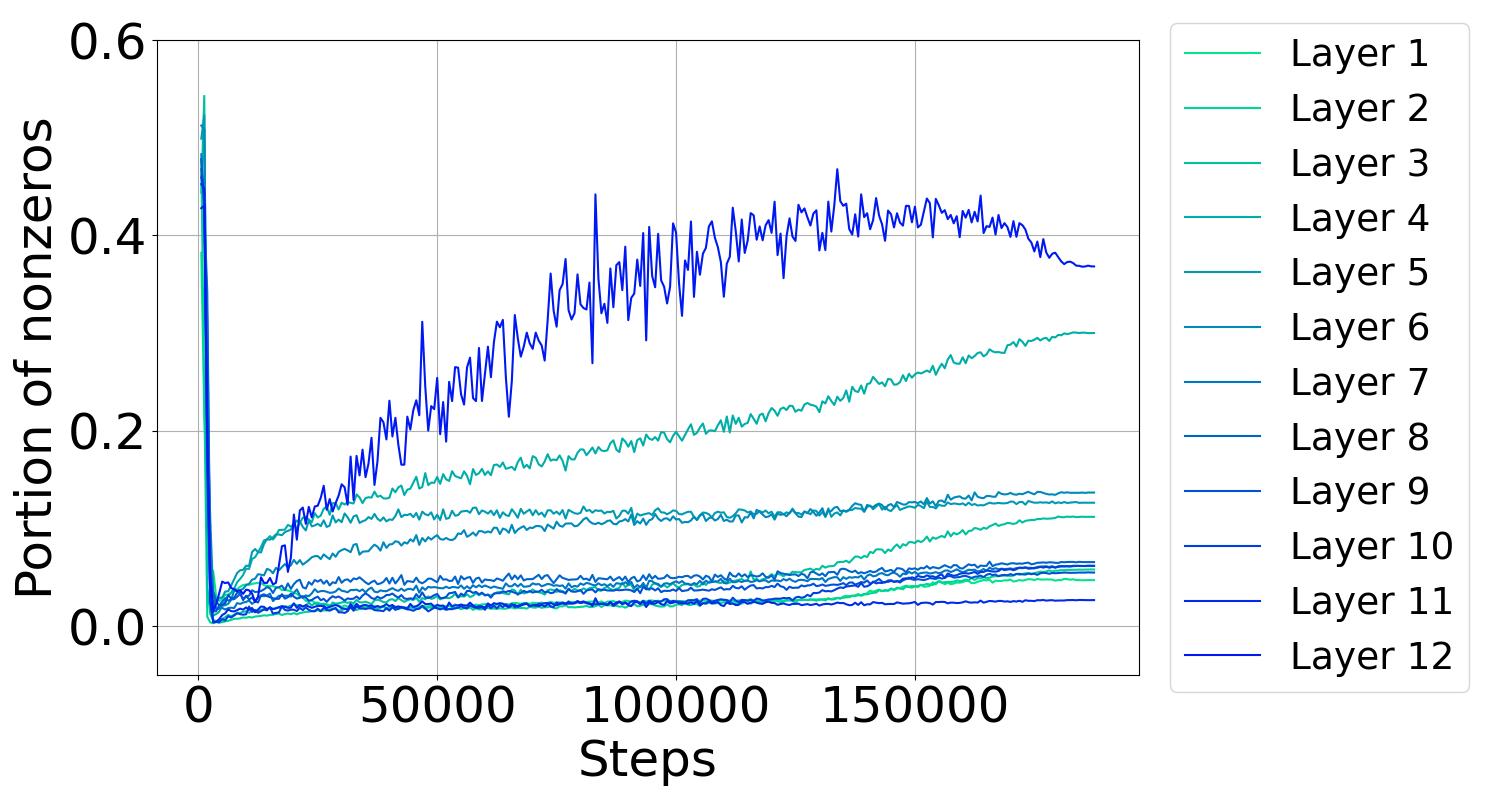}
        \caption{\tiny Testing sparsity of vanilla ViT.}\label{figure:productive_vit_vanilla_testing}
    \end{subfigure}
    \begin{subfigure}[t]{0.24\textwidth}
        \myincludegraphics[width=\textwidth]{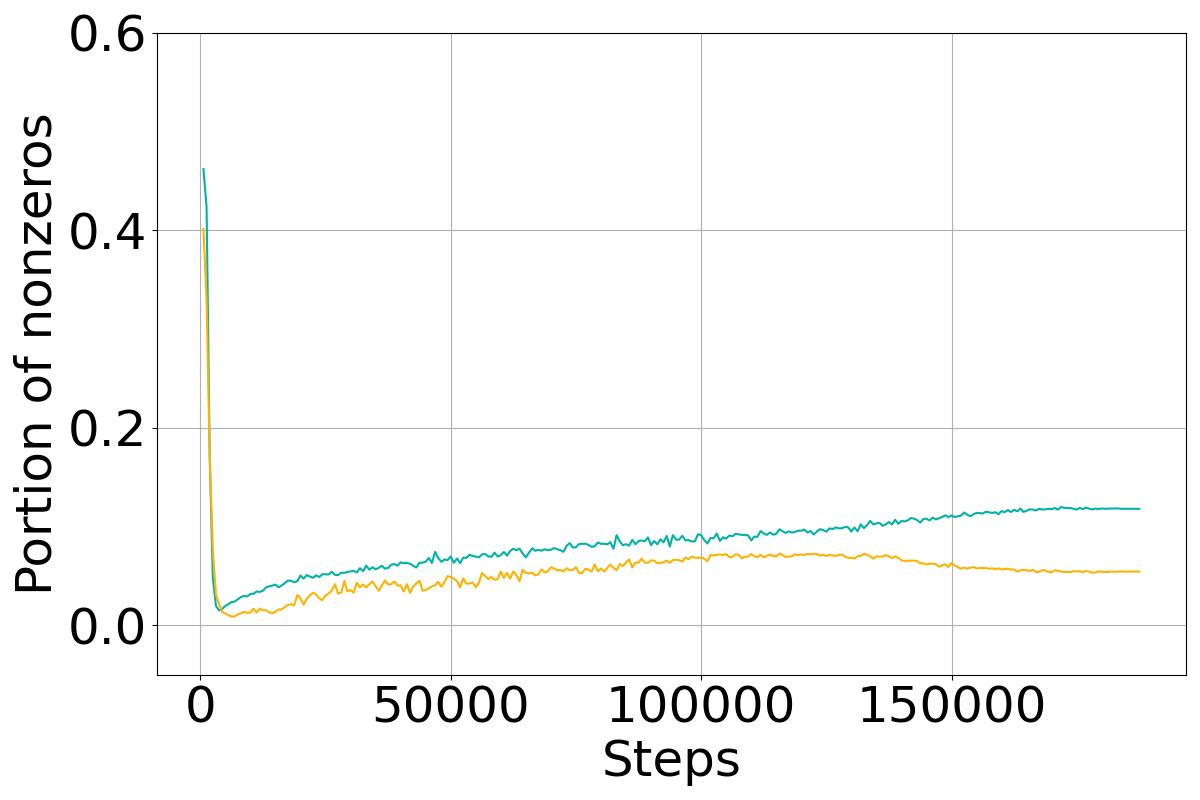}
        \caption{\tiny Comparison of layer-averaged testing sparsity between vanilla and modified ViT.}\label{figure:productive_vit_average_testing}
    \end{subfigure}
    \begin{subfigure}[t]{0.24\textwidth}
        \myincludegraphics[width=\textwidth]{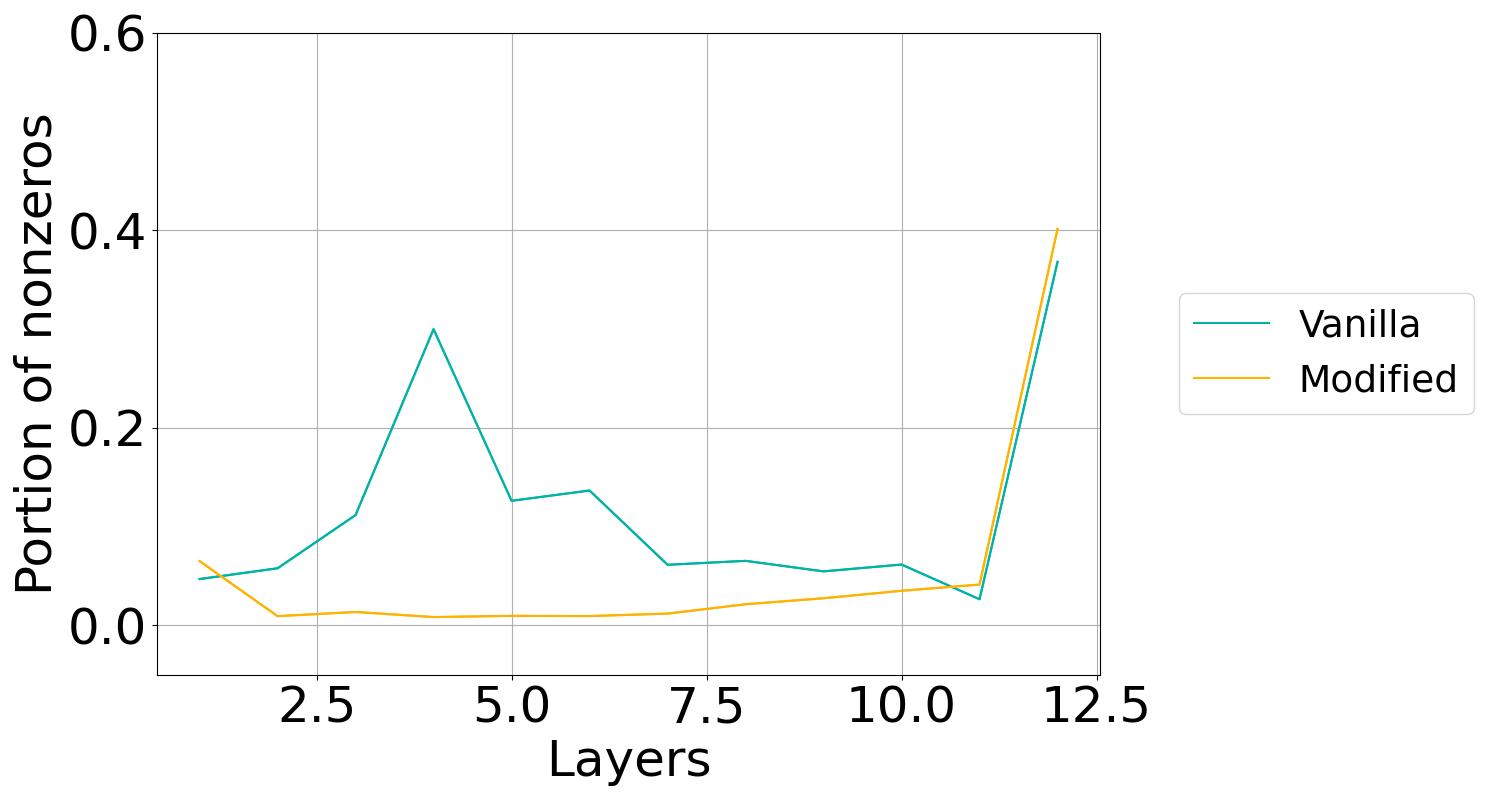}
        \caption{\tiny Comparison of layerwise testing sparsity between vanilla and modified ViT at the last testing.} \label{figure:productive_vit_end_testing}
    \end{subfigure}
    \caption{Training and testing sparsity during training of ViT-Base/16 on ImageNet-1K. Red and yellow are used for modified ViT while blue and green indicate vanilla ViT. 
        \cref{figure:productive_vit_sparsified_training}-\cref{figure:productive_vit_end_training} show \emph{training} sparsity while \cref{figure:productive_vit_sparsified_testing}-\cref{figure:productive_vit_end_testing} show \emph{testing} sparsity. 
        \cref{figure:productive_vit_sparsified_training}, \cref{figure:productive_vit_vanilla_training}, \cref{figure:productive_vit_sparsified_testing} and \cref{figure:productive_vit_vanilla_testing} show sparsity in a layerwise manner, while \cref{figure:productive_vit_average_training} and \cref{figure:productive_vit_average_testing} illustrate after averaging across layers in order to demonstrate the overall improvement.
        \cref{figure:productive_vit_end_training} and \cref{figure:productive_vit_end_testing} display the layerwise sparsity at the end of training to show improvements' tendency along the depth.
    }\label{figure:productive_vit}
\end{figure}
\begin{table}
    \centering
    \begin{tabular}{lrrrrrrrr}
        \toprule
                                & \multicolumn{2}{c}{Training Sparsity}                             &   \multicolumn{2}{c}{Testing Sparsity}                &   \multicolumn{1}{c}{Acc@1}   &   \multicolumn{1}{c}{Acc@5}\\
        \midrule    
        Vanilla                 &   $0.104$             &                                           &   $0.087$                 &                           &   $\mathbf{77.35\%}$          &   $\mathbf{93.50\%}$\\
        Modified                &   $\mathbf{0.046}$    &   $\downarrow 55.92\%$                    &   $\mathbf{0.055}$        &   $\downarrow 36.03\%$    &   $76.77\%$                   &   $93.30\%$\\
        \bottomrule
    \end{tabular}
    \caption{Summary of training ViT-Base from scratch on ImageNet-1K\citep{imagenet1k}. Training sparsity is computed by integrating \cref{figure:productive_vit_average_training}.}\label{table:productive_vit}
\end{table}

\begin{figure}
    \centering
    \resetHeight{}
    \begin{subfigure}[t]{0.24\textwidth}
        \myincludegraphics[width=\textwidth]{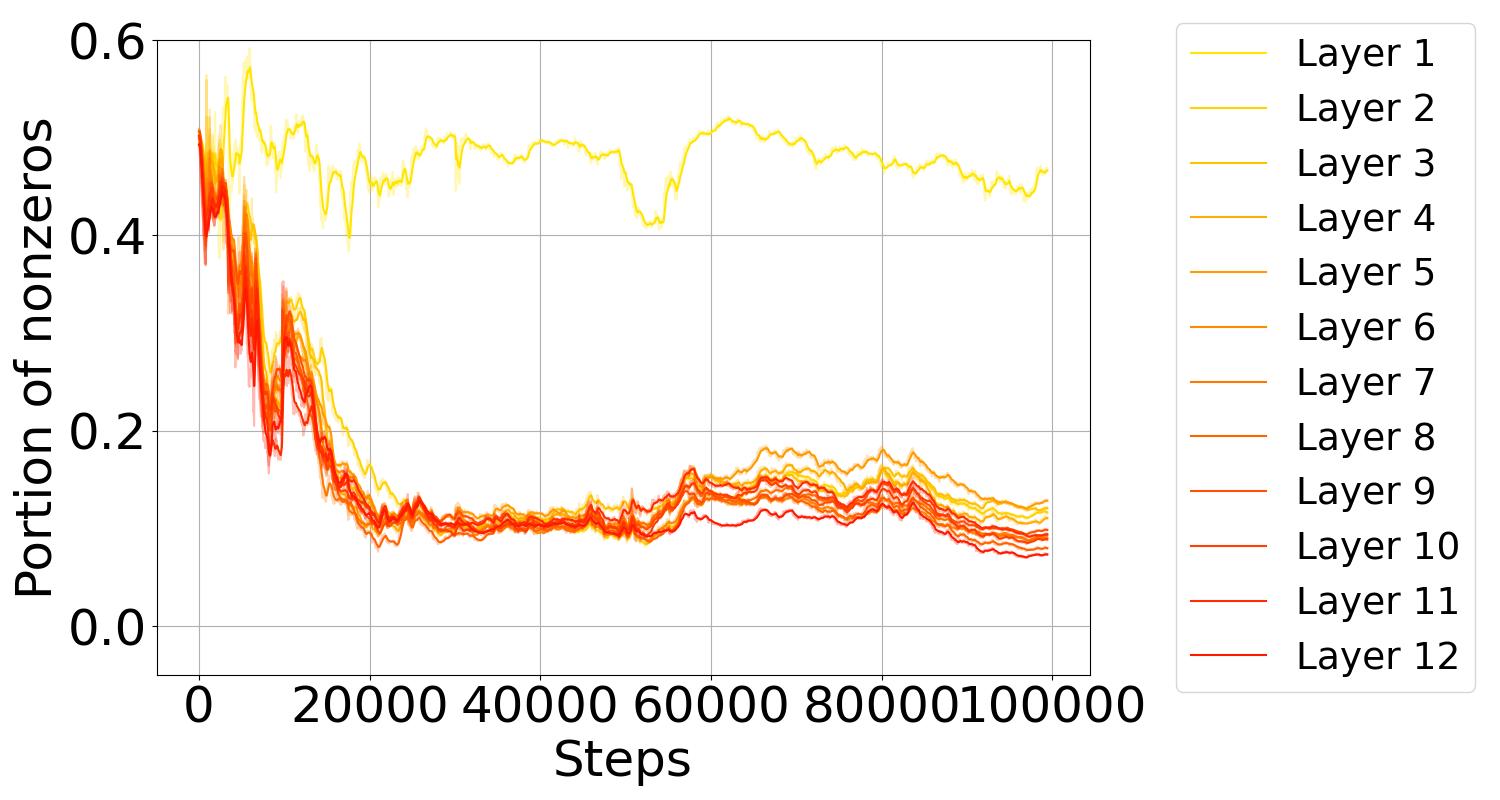}
        \caption{\tiny Training sparsity in the encoder of modified T5.}\label{figure:productive_t5_sparsified_training_encoder}
    \end{subfigure}
    \begin{subfigure}[t]{0.24\textwidth}
        \myincludegraphics[width=\textwidth]{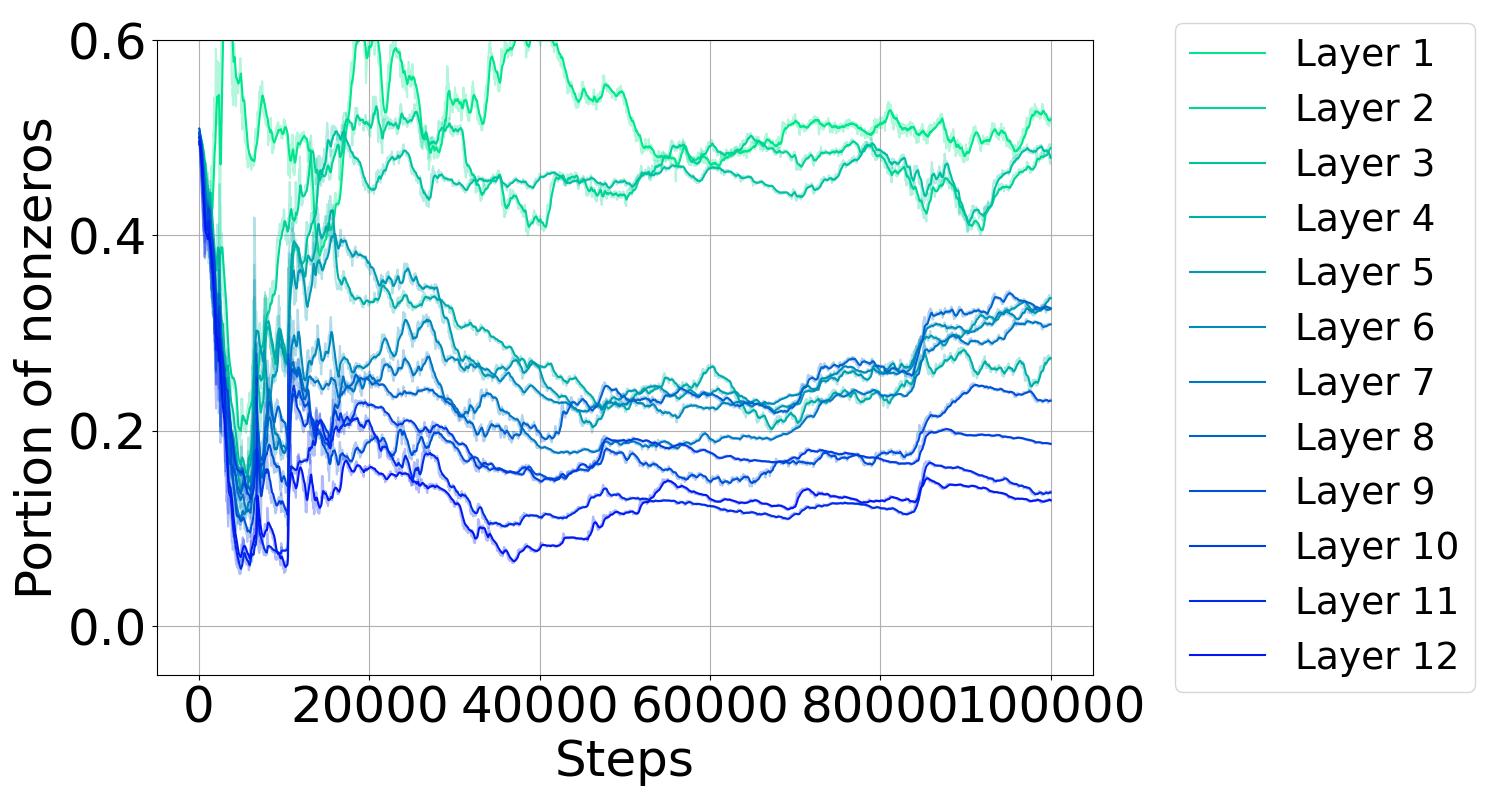}
        \caption{\tiny Training sparsity in the encoder of vanilla T5.}\label{figure:productive_t5_vanilla_training_encoder}
    \end{subfigure}
    \begin{subfigure}[t]{0.24\textwidth}
        \myincludegraphics[width=\textwidth]{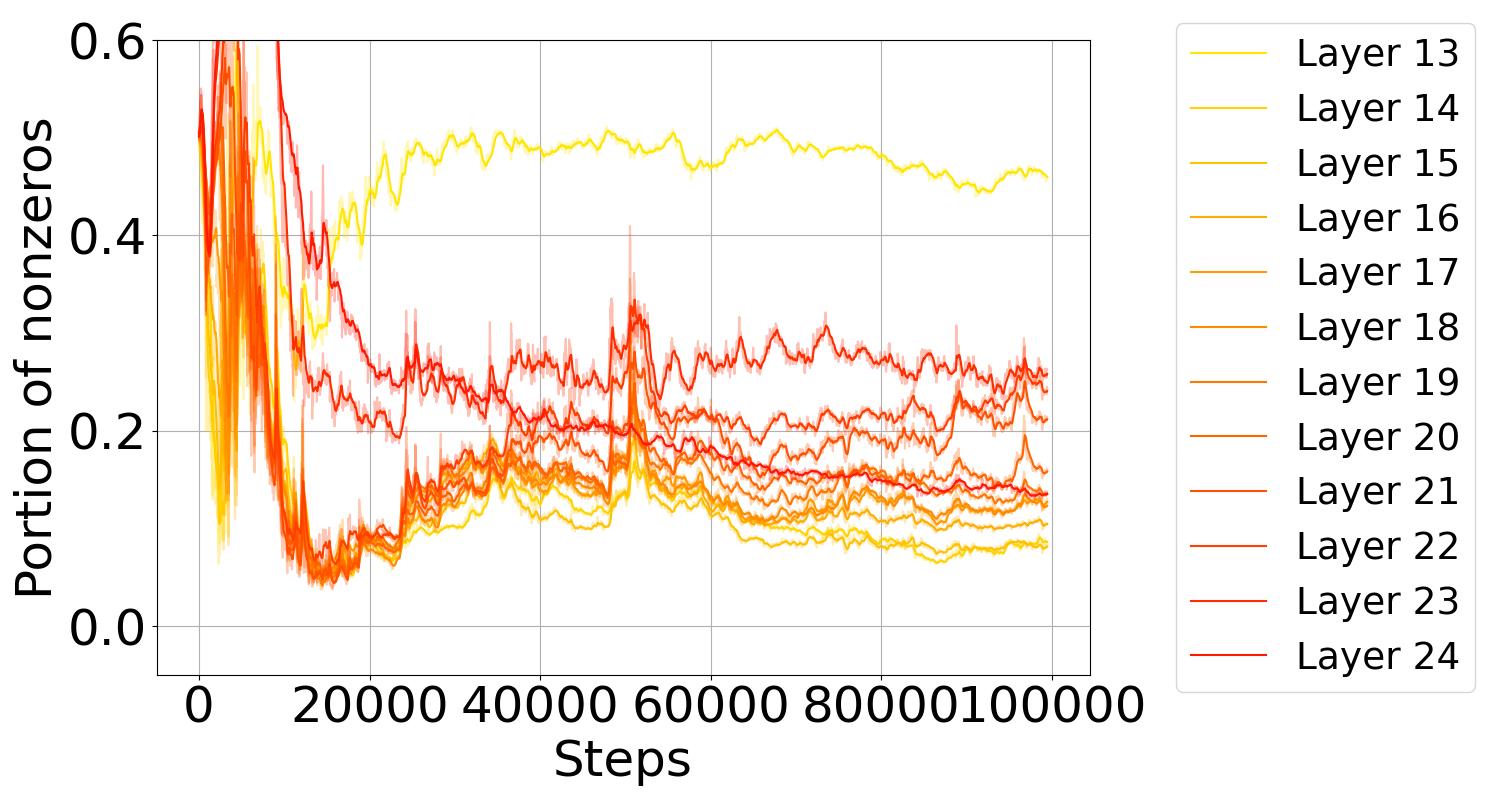}
        \caption{\tiny Training sparsity in the decoder of modified T5.}\label{figure:productive_t5_sparsified_training_decoder}
    \end{subfigure}
    \begin{subfigure}[t]{0.24\textwidth}
        \myincludegraphics[width=\textwidth]{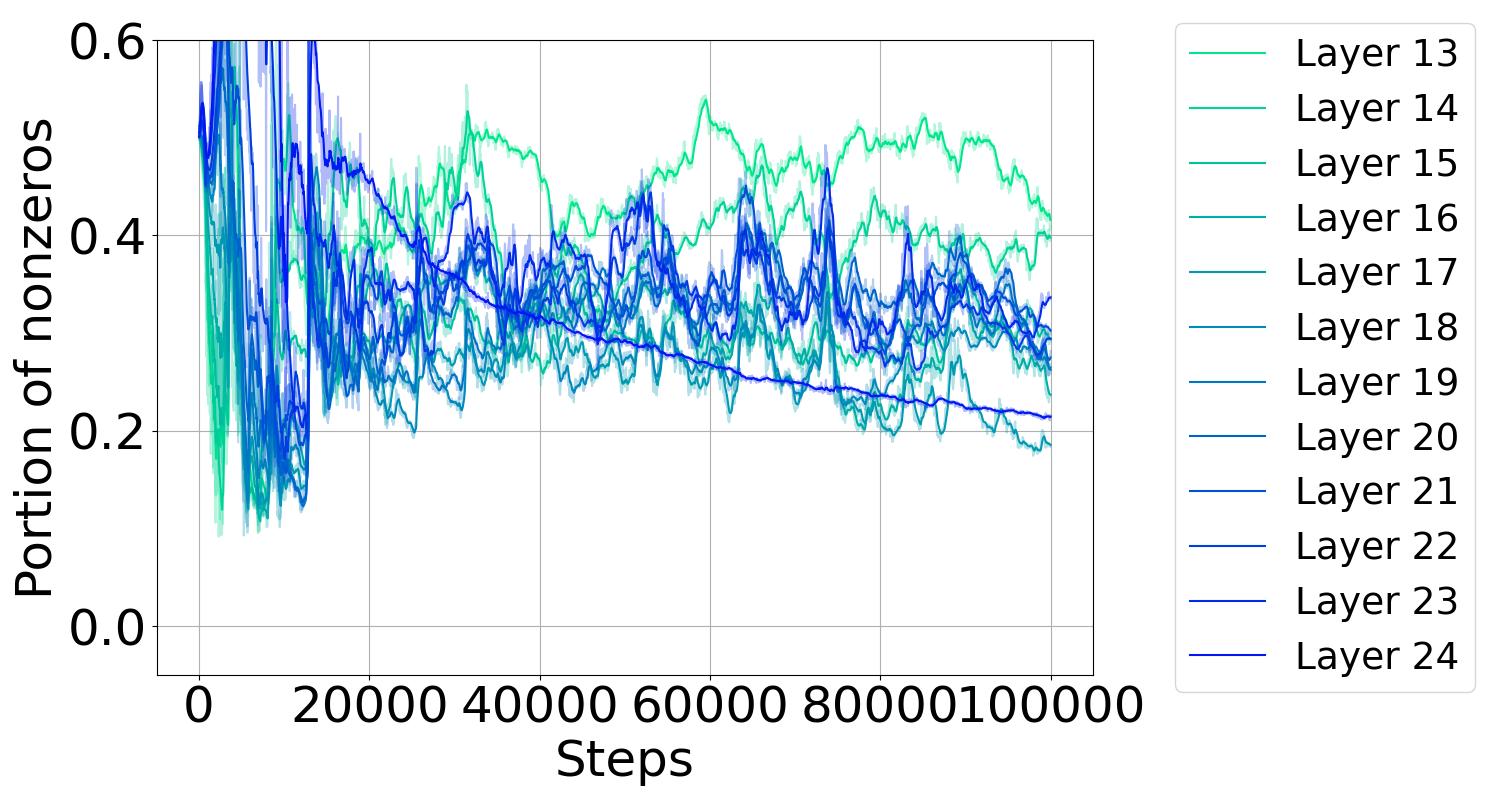}
        \caption{\tiny Training sparsity in the decoder of vanilla T5.}\label{figure:productive_t5_vanilla_training_decoder}
    \end{subfigure}
    \begin{subfigure}[t]{0.24\textwidth}
        \myincludegraphics[width=\textwidth]{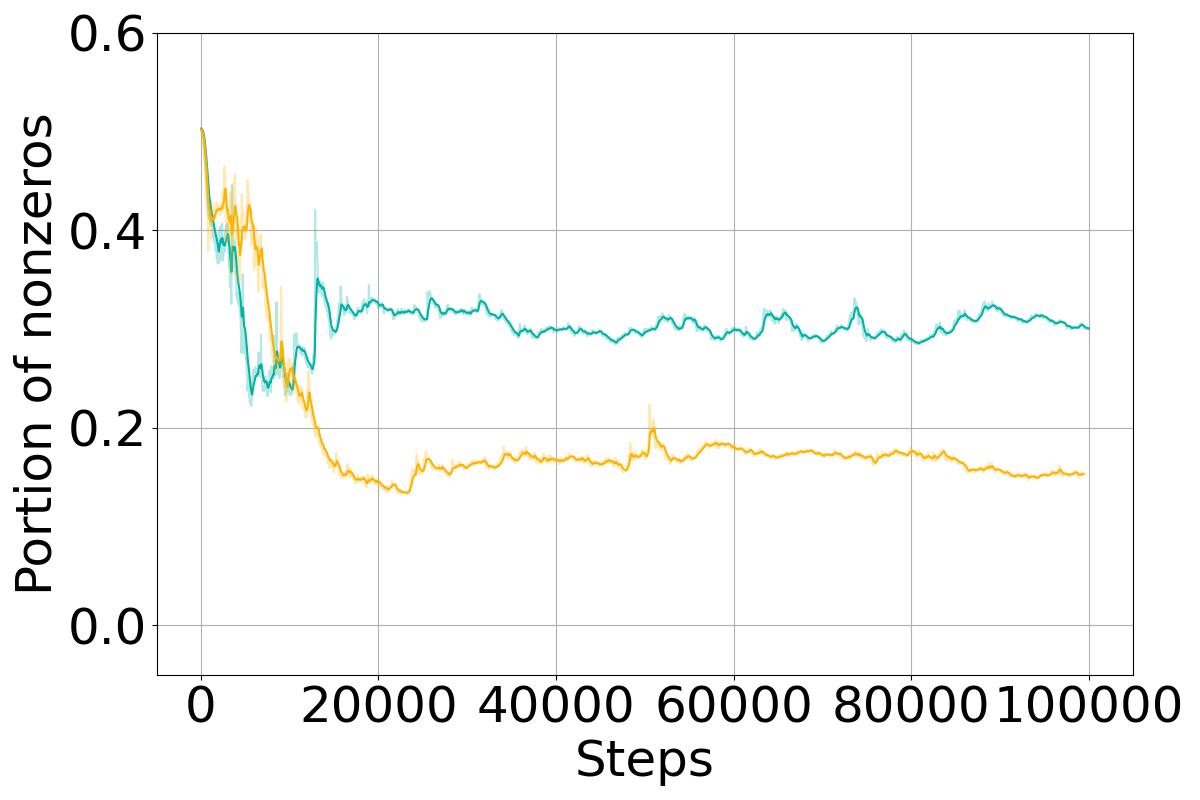}
        \caption{\tiny Comparison of layer-averaged training sparsity between vanilla and modified T5.}\label{figure:productive_t5_average_training}
    \end{subfigure}
    \begin{subfigure}[t]{0.24\textwidth}
        \myincludegraphics[width=\textwidth]{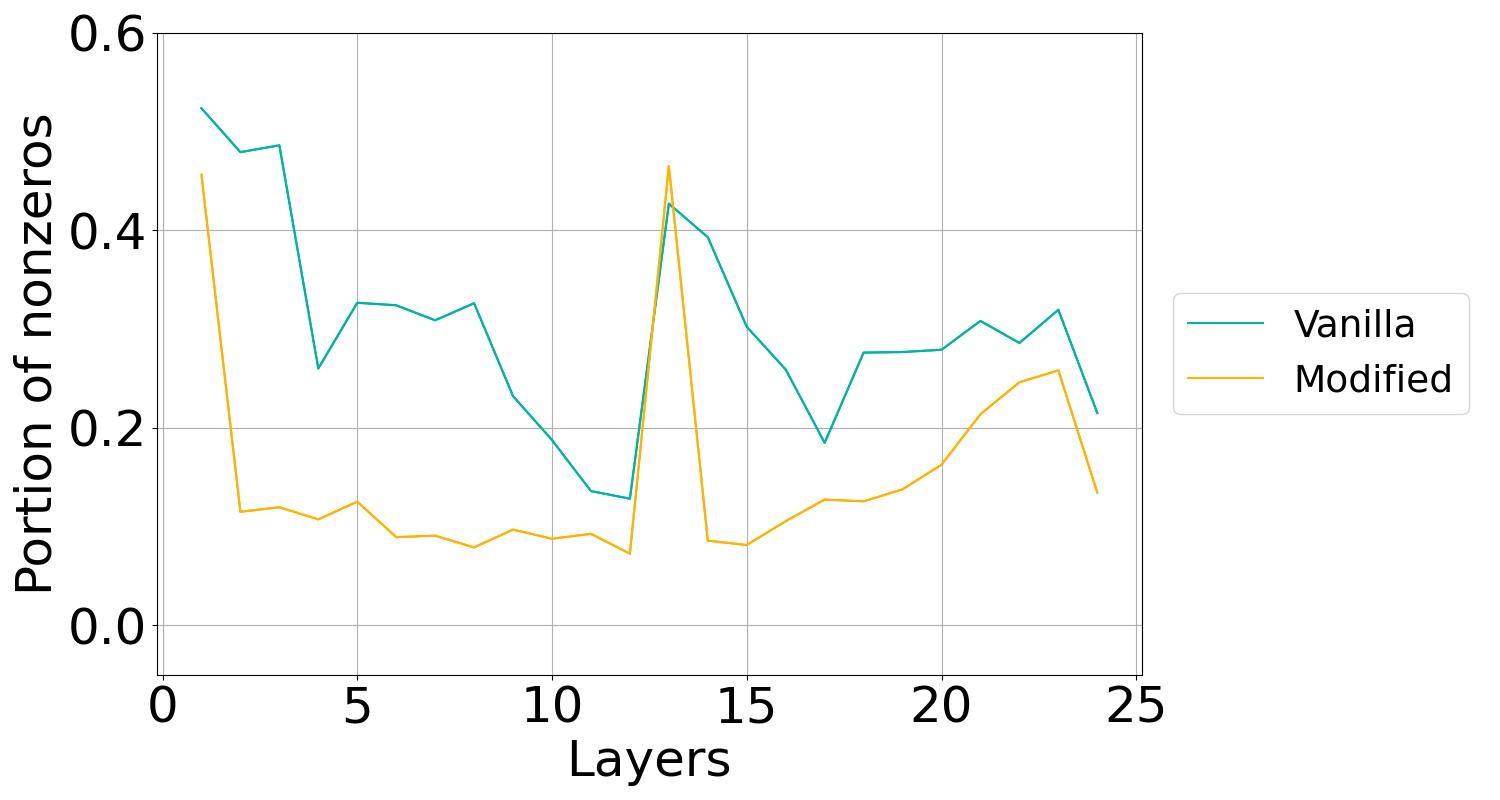}
        \caption{\tiny Comparison of layerwise training sparsity between vanilla and modified T5 during the last 100 steps.} \label{figure:productive_t5_end_training}
    \end{subfigure}
    \begin{subfigure}[t]{0.24\textwidth}
        \myincludegraphics[width=\textwidth]{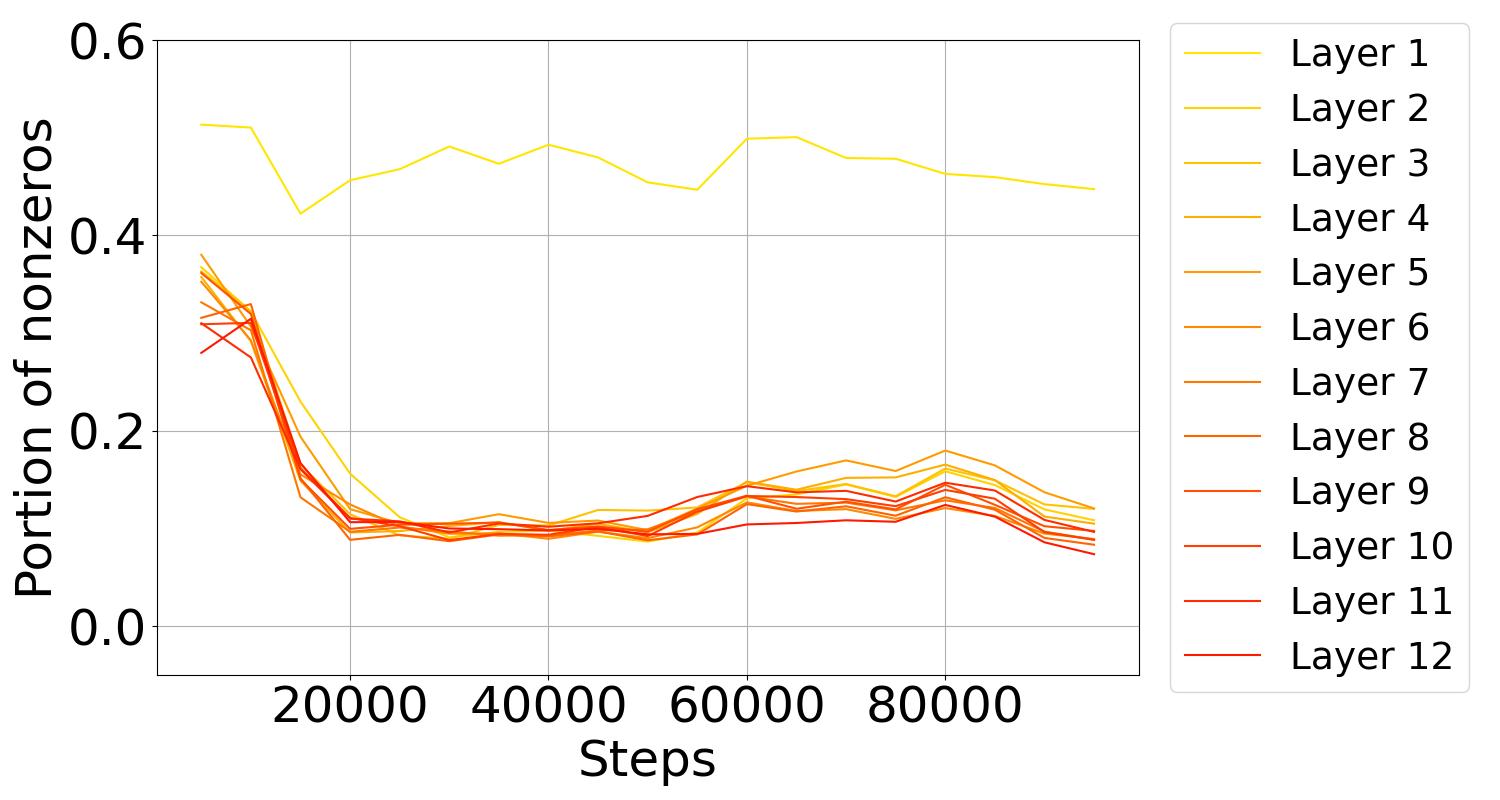}
        \caption{\tiny Testing sparsity in the encoder of modified T5.}\label{figure:productive_t5_sparsified_testing_encoder}
    \end{subfigure}
    \begin{subfigure}[t]{0.24\textwidth}
        \myincludegraphics[width=\textwidth]{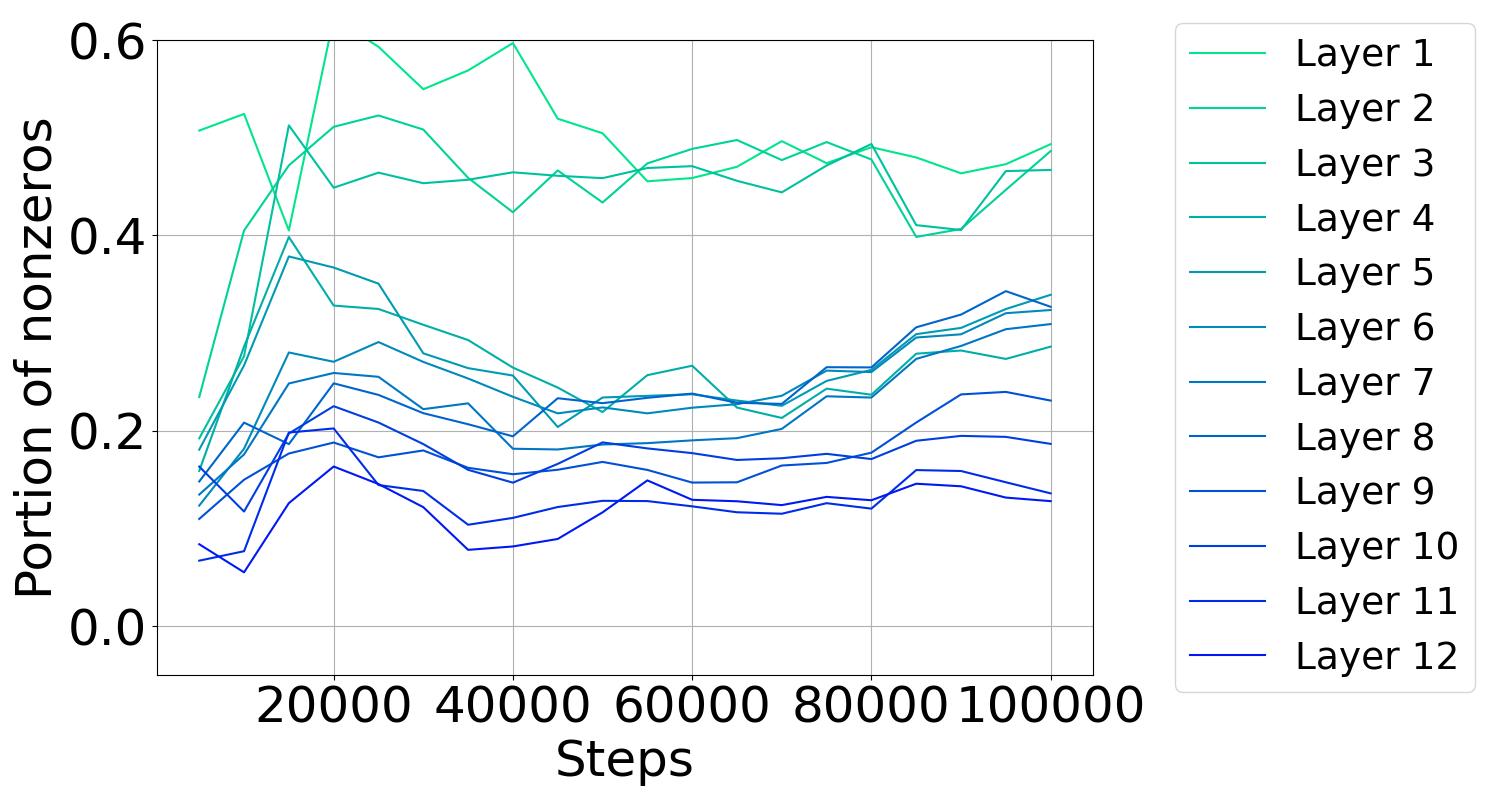}
        \caption{\tiny Testing sparsity in the encoder of vanilla T5.}\label{figure:productive_t5_vanilla_testing_encoder}
    \end{subfigure}
    \begin{subfigure}[t]{0.24\textwidth}
        \myincludegraphics[width=\textwidth]{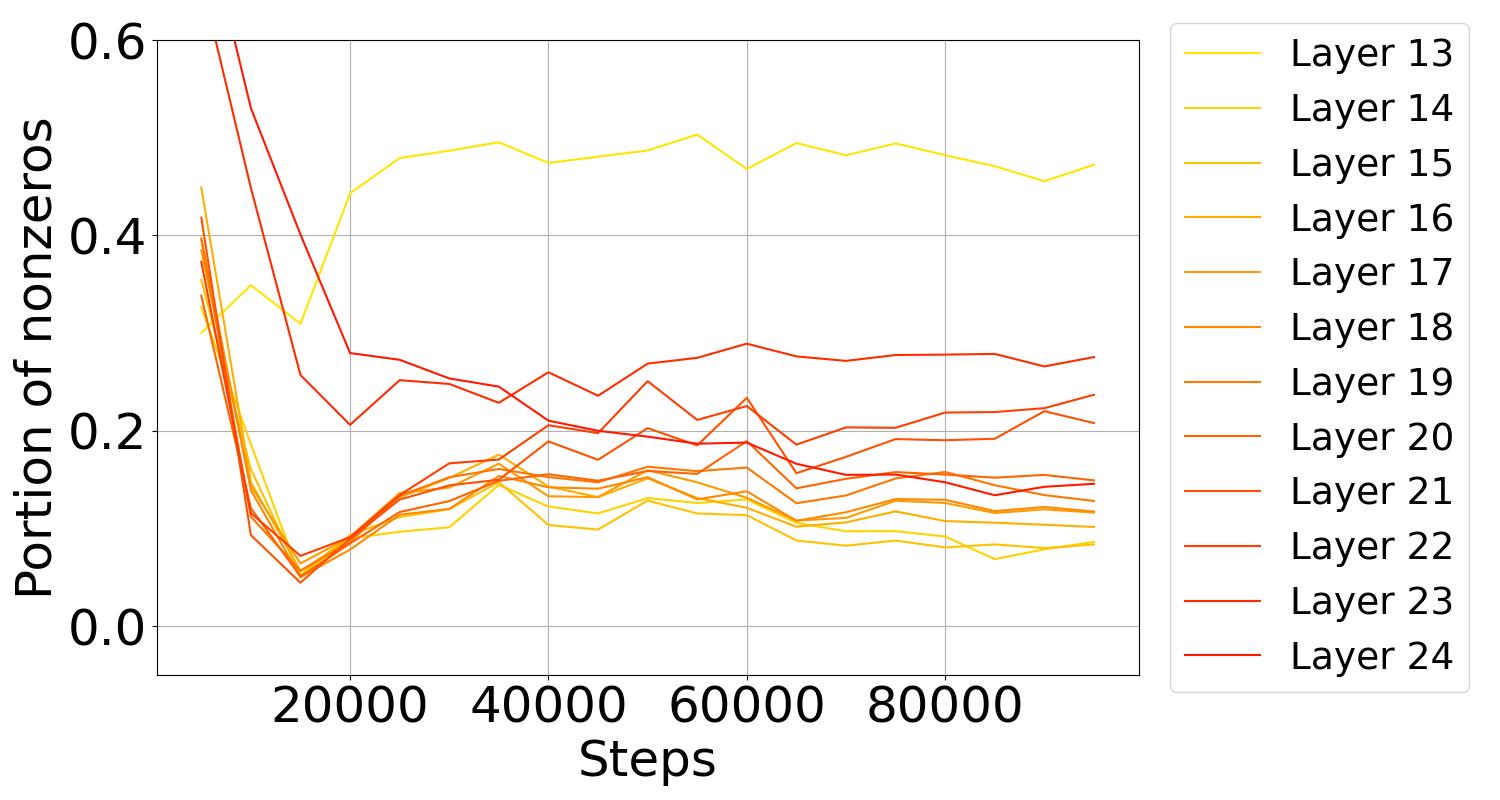}
        \caption{\tiny Testing sparsity in the decoder of modified T5.}\label{figure:productive_t5_sparsified_testing_decoder}
    \end{subfigure}
    \begin{subfigure}[t]{0.24\textwidth}
        \myincludegraphics[width=\textwidth]{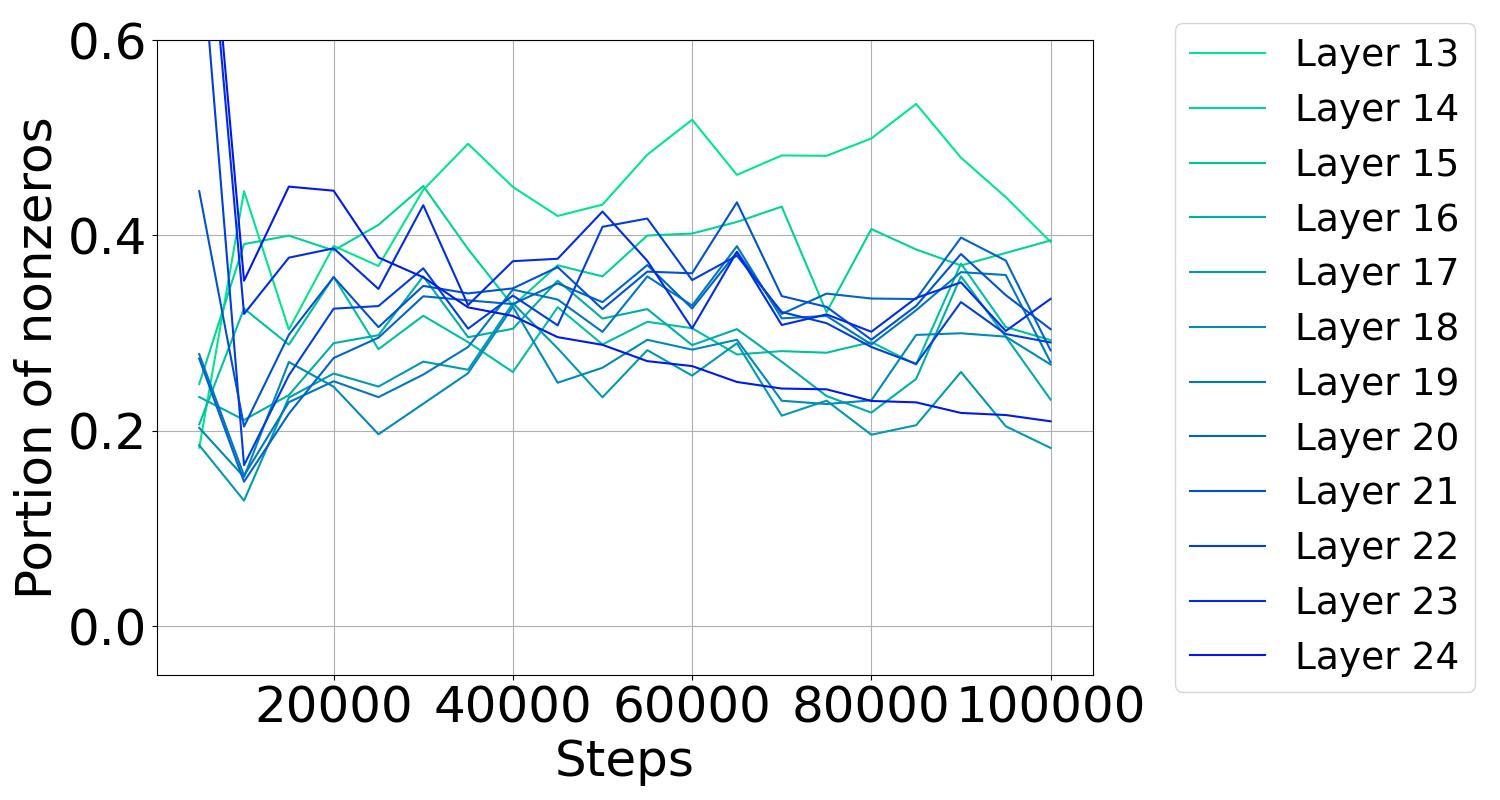}
        \caption{\tiny Testing sparsity in the decoder of vanilla T5.}\label{figure:productive_t5_vanilla_testing_decoder}
    \end{subfigure}
    \begin{subfigure}[t]{0.24\textwidth}
        \myincludegraphics[width=\textwidth]{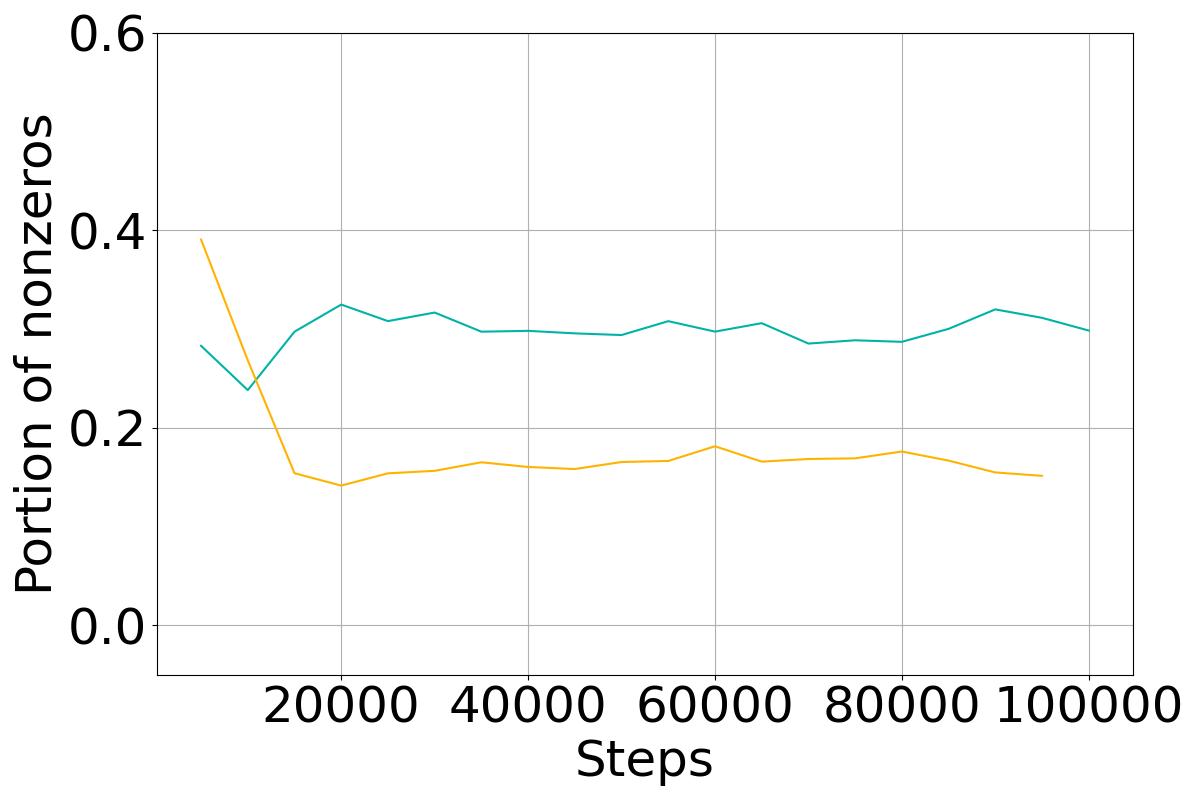}
        \caption{\tiny Comparison of layer-averaged testing sparsity between vanilla and modified T5.}\label{figure:productive_t5_average_testing}
    \end{subfigure}
    \begin{subfigure}[t]{0.24\textwidth}
        \myincludegraphics[width=\textwidth]{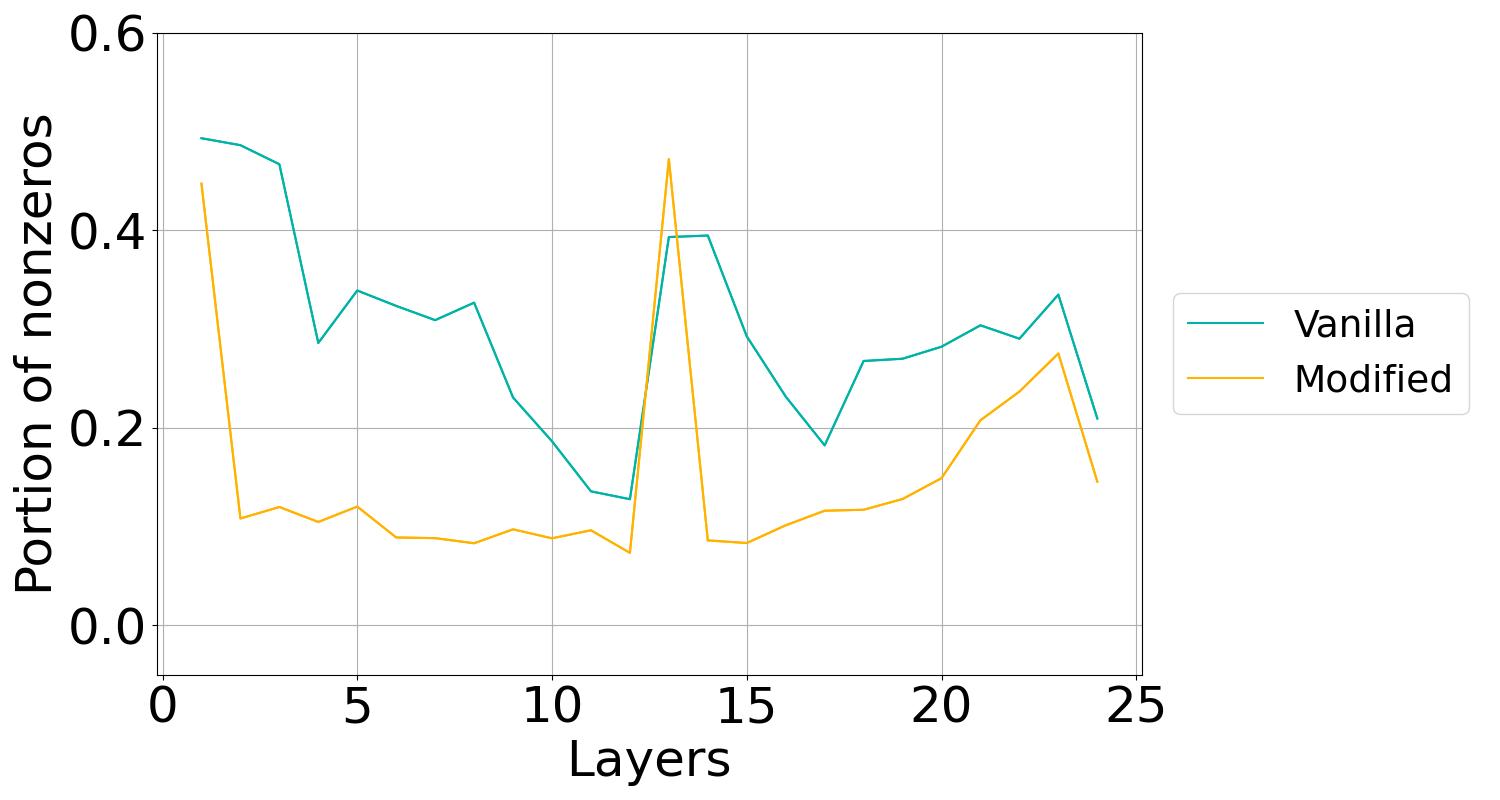}
        \caption{\tiny Comparison of layerwise testing sparsity between vanilla and modified T5.} \label{figure:productive_t5_end_testing}
    \end{subfigure}
    \caption{Training and testing sparsity during training of T5-Base on C4. Red and yellow are used for modified models while blue and green indicate vanilla ones. 
        \cref{figure:productive_t5_sparsified_training_encoder}-\cref{figure:productive_t5_end_training} show \emph{training} sparsity while \cref{figure:productive_t5_sparsified_testing_encoder}-\cref{figure:productive_t5_end_testing} show \emph{testing} sparsity. 
        \cref{figure:productive_t5_sparsified_training_encoder}-\cref{figure:productive_t5_vanilla_training_decoder} and \cref{figure:productive_t5_sparsified_testing_encoder}-\cref{figure:productive_t5_vanilla_testing_decoder} show sparsity in a layerwise manner with encoder layers numbered by $1$-$12$ and decoder layers numbered by $13$-$24$, while \cref{figure:productive_t5_average_training} and \cref{figure:productive_t5_average_testing} illustrate after averaging across layers in order to demonstrate the overall improvement.
        \cref{figure:productive_t5_end_training} and \cref{figure:productive_t5_end_testing} display the layerwise sparsity at the end of training to show improvements' tendency along the depth.
    }\label{figure:productive_t5}
\end{figure}
\begin{table}
    \centering
    \begin{tabular}{lrrrrrrrr}
        \toprule
                                & \multicolumn{2}{c}{Training Sparsity}                             &   \multicolumn{2}{c}{Testing Sparsity}                &   \multicolumn{1}{c}{Testing Loss}\\
        \midrule    
        Vanilla                 &   $0.302$             &                                           &   $0.299$                 &                           &   $4.88$\\
        Modified                &   $\mathbf{0.153}$    &   $\downarrow 49.25\%$                    &   $\mathbf{0.180}$        &   $\downarrow 39.64\%$    &   $\mathbf{4.78}$\\
        \bottomrule
    \end{tabular}
    \caption{Summary of training T5-Base from scratch on C4. Training sparsity is computed by integrating \cref{figure:productive_t5_average_training}.}\label{table:productive_t5}
\end{table}

\begin{figure}
    \centering
    \resetHeight{}
    \begin{subfigure}[t]{0.24\textwidth}
        \myincludegraphics[width=\textwidth]{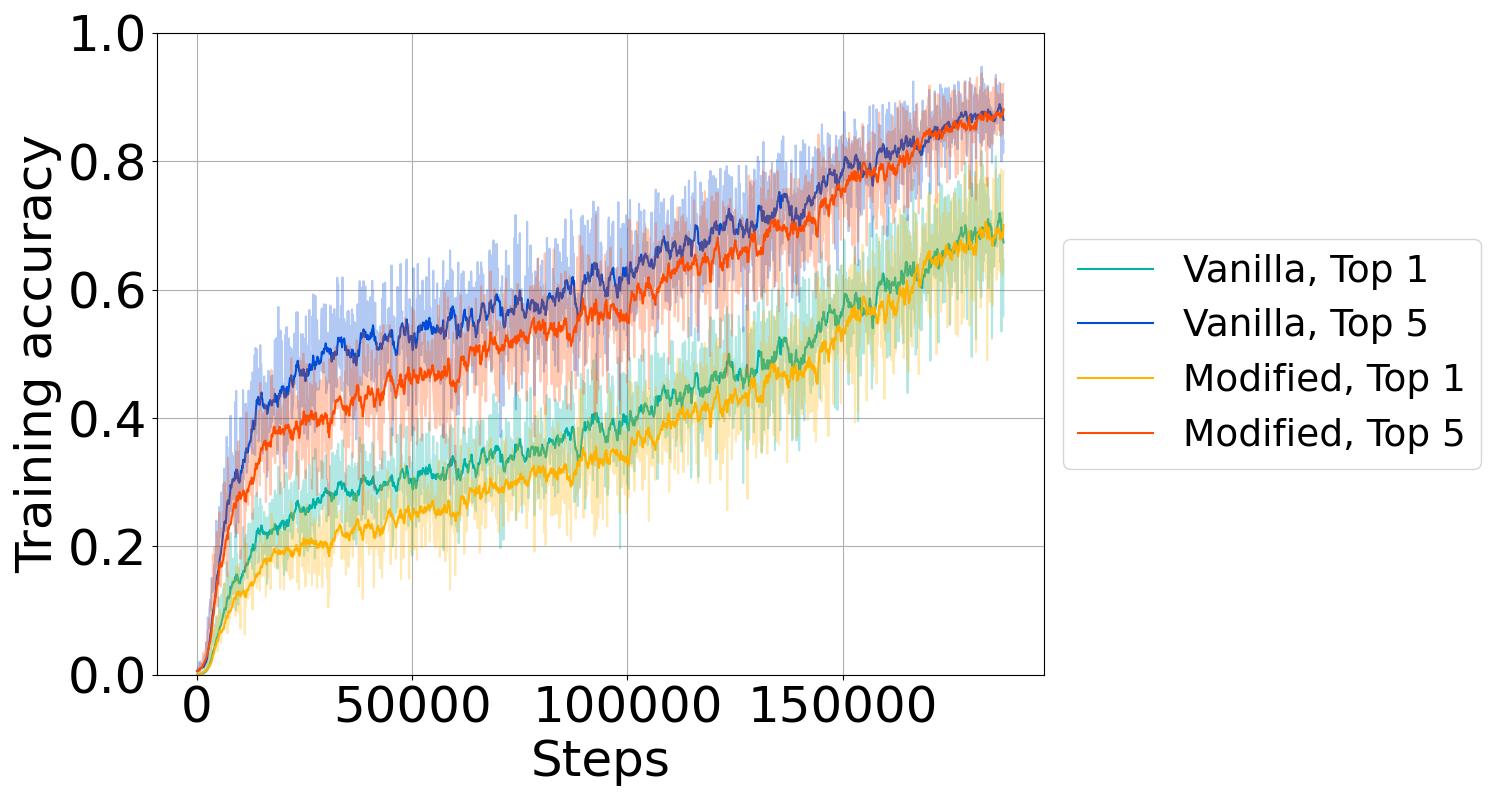}
        \caption{\tiny Training accuracy of ViT-Base/16 on ImageNet-1K.}\label{figure:productive_vit_acc_training}
    \end{subfigure}
    \begin{subfigure}[t]{0.24\textwidth}
        \myincludegraphics[width=\textwidth]{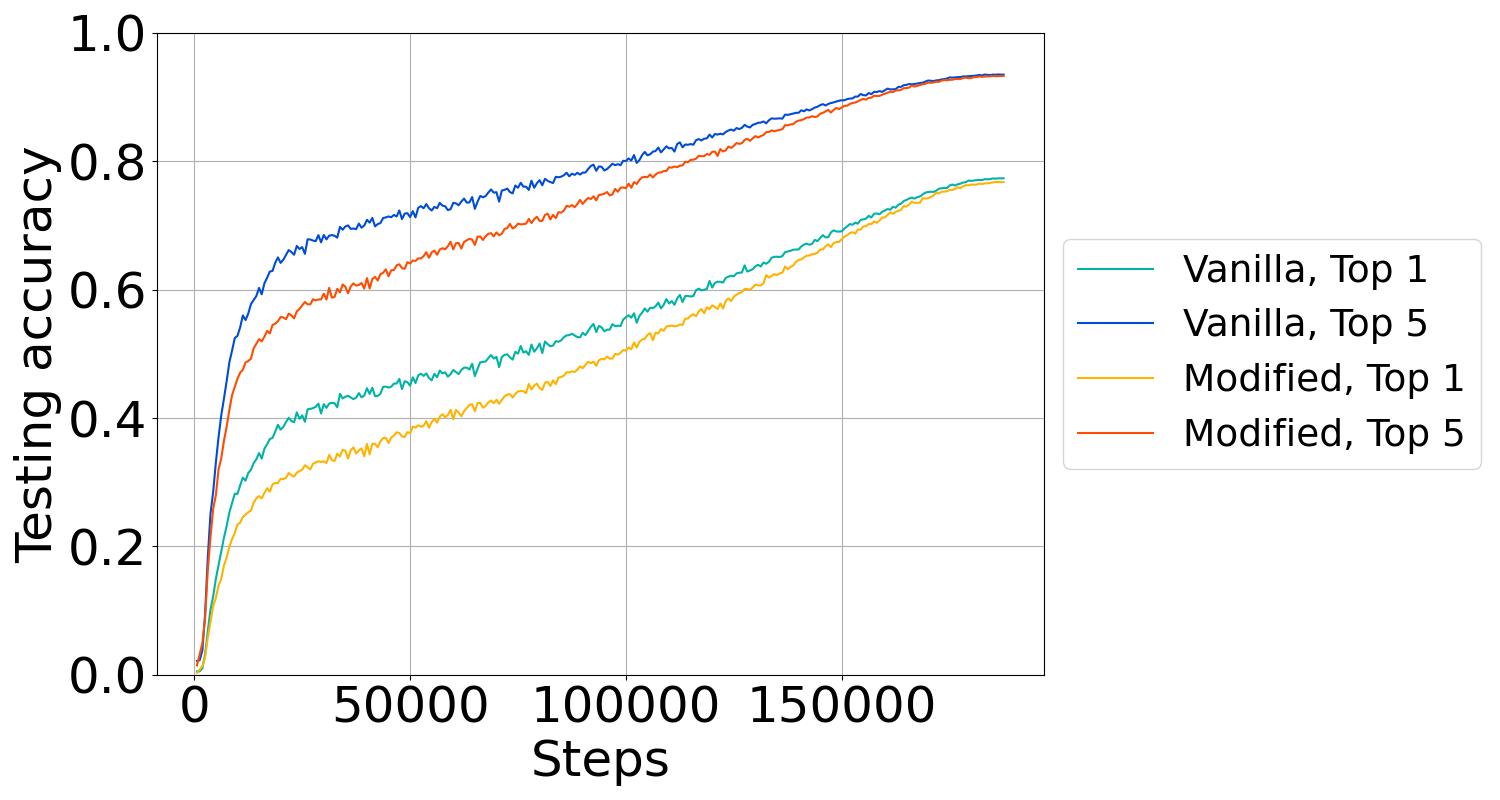}
        \caption{\tiny Testing accuracy of ViT-Base/16 on ImageNet-1K.}\label{figure:productive_vit_acc_testing}
    \end{subfigure}
    \begin{subfigure}[t]{0.24\textwidth}
        \includegraphics[width=\textwidth]{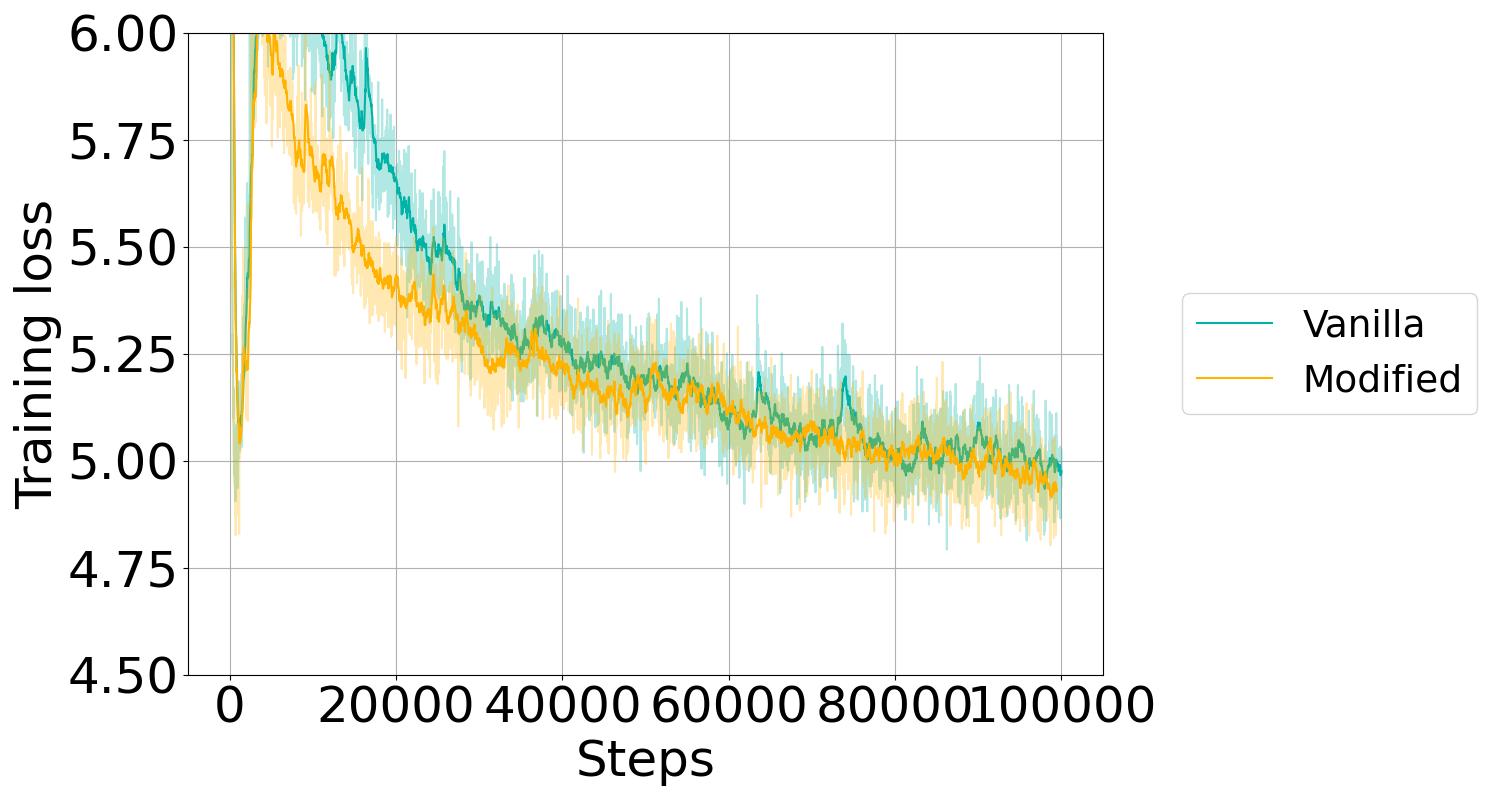}
        \caption{\tiny Training loss of T5-Base on C4.}\label{figure:productive_t5_acc_training}
    \end{subfigure}
    \begin{subfigure}[t]{0.24\textwidth}
        \includegraphics[width=\textwidth]{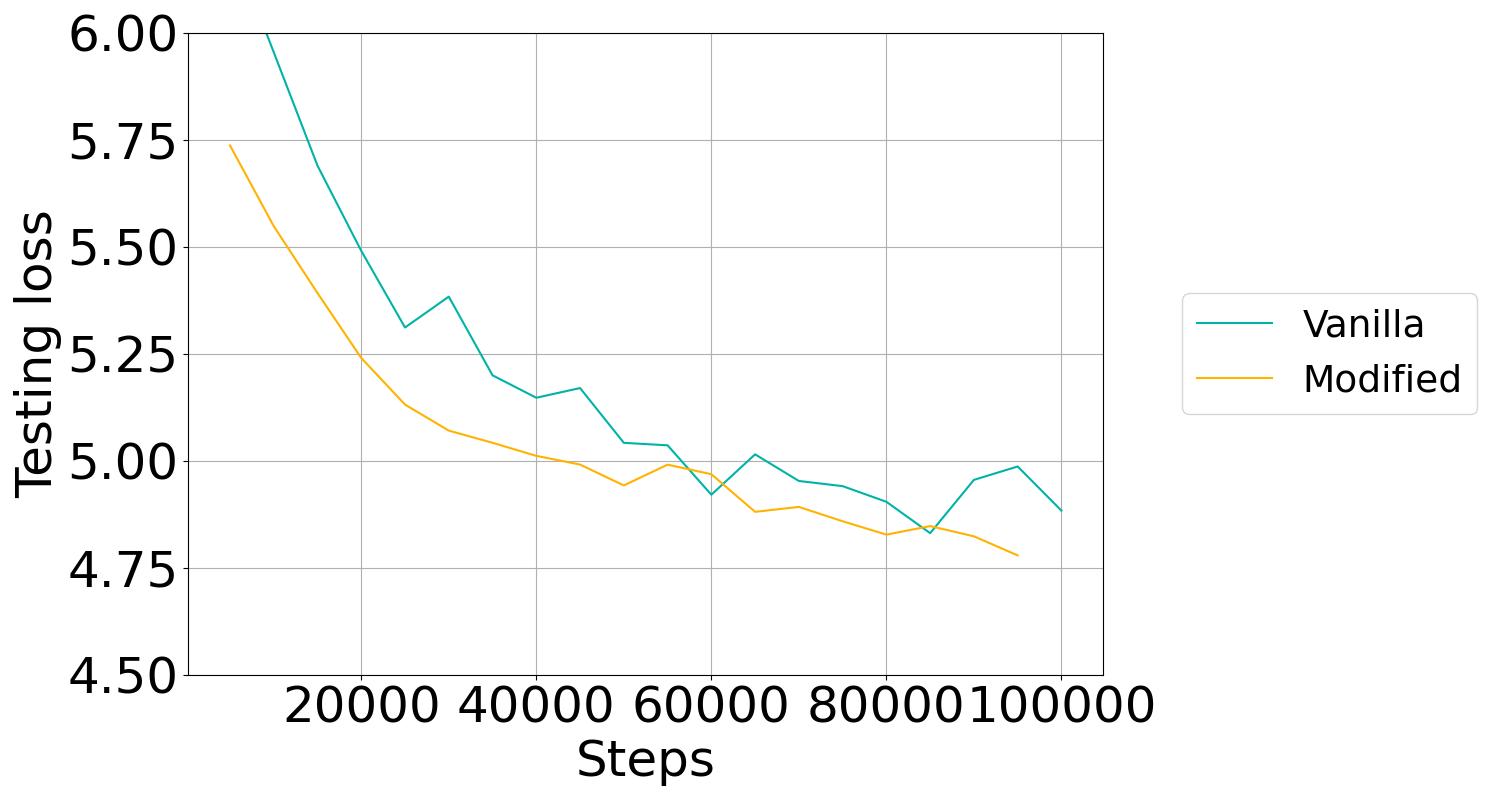}
        \caption{\tiny Testing loss of T5-Base on C4.}\label{figure:productive_t5_acc_testing}
    \end{subfigure}
    \caption{Accuracies or losses during training from scratch.}
\end{figure}

Visually obvious significant improvements in activation/gradient sparsity can be found in the experiment results. 
By integrating the training sparsity along steps, a training sparsity improvement of approximately 50\% is achieved in both CV and NLP tasks. 
By observing the testing sparsity of the last checkpoint, there is a testing improvement of relatively at least 36\%.
\cref{figure:productive_vit_end_training}, \cref{figure:productive_vit_end_testing}, \cref{figure:productive_t5_end_training} and \cref{figure:productive_t5_end_testing} suggest that sparsity improvements happen in shallow and middle layers of both encoder and decoder. We conjecture that it is because our improvement focuses on stimulating gradient sparsity, which is poorly trained since fewer parameters lie before shallow layers and implicit adversarial samples are insufficient. Zeroth biases can fix this insufficiency. In contrast, deep layers have sufficient implicit adversarial samples but the motivation directly toward activation sparsity is limited, which our modification hardly helps.
In \cref{table:productive_vit} and \cref{table:productive_t5}, in contrast to sparsity generalization brought by vanilla training one can observe sparsity overfitting after modification in both CV and NLP tasks. Fortunately, it is slight compared to the improvement.

The improvement in sparsity does not harm performance or generalization. 
According to \cref{table:productive_vit}, \cref{table:productive_t5}, \cref{figure:productive_vit_acc_testing} and \cref{figure:productive_t5_acc_testing}, testing accuracy of ViT only suffers an acceptable accuracy degeneration while the testing loss of T5 surprisingly improves after modification. 
Nevertheless, the training of the modified ViT indeed slows during the middle stage in \cref{figure:productive_vit_acc_training} and \cref{figure:productive_vit_acc_testing}, but the lagging diminishes at the end of training. The reason for this catching up remains mysterious and we conjecture it has something to do with the decayed learning rate. For T5 the modified version always have better training and testing losses. Notably, the modified T5 seems to have better generalization because in \cref{figure:productive_t5_acc_testing} modified T5 has significantly lower testing losses despite similar training losses with the vanilla one according to \cref{figure:productive_t5_acc_training}.

\subsection{Finetuning for Sparsity}\label{sec:f_experiments}

Apart from training from scratch, we finetune existing weights for sparsity to show a cheaper way to make them sparser, given that our modification is nearly plug-and-play as analyzed in \cref{sec:illustration}. 
We use the vanilla weights from \cref{sec:p_experiments} and finetune it for another 15 epochs or 10,000 steps after applying the modification. 
Since finetuning is relatively short and mainly serves inference, we only measure testing sparsity rather than training sparsity in this subsection. 

LoRA is used for finetuning to modify weight matrices while keeping learned knowledge, because we observe difficulties in recovering performance with full-parameter finetuning in tentative experiments. In both tasks, we use a relatively large $r=192$ for LoRA because we expect better sparsity requires massively rearranging the matrices. LoRA is applied at all matrices in MLP blocks as well as in self-attention blocks. 
For ViT we directly plug zero-initialized zeroth biases before MLP blocks and replace $\jrelu$ for $\relu$, which does not harm testing accuracy thanks to the derivative calibration of $\jrelu$ to imitate $\relu$. However we find T5 hard to directly adapt to $\jrelu$, so we mix two activation functions linearly and increase the portion of $\jrelu$ linearly after each step, starting from $0$ to $1$ during the first 3,000 steps.
In tentative experiments we find clamped scaling factors in LayerNorm layers critical to better sparsity, so we apply \cref{algo:refined_finetuning} after every step with $T_{\text{uplifting}}=3,000$ in both CV and NLP tasks. The finetuning lasts for 15 epochs for the CV task while 10,000 steps for the NLP task. All other hyperparameters are inherited from \cref{sec:from_scratch}.
The testing sparsity of modified models is compared in \cref{figure:finetuning_vit}, \cref{figure:finetuning_t5} and \cref{table:finetuning} with those before finetuning and those after vanilla finetuning. 
\begin{figure}
    \centering
    \resetHeight{}
    \begin{subfigure}[t]{0.24\textwidth}
        \myincludegraphics[width=\textwidth]{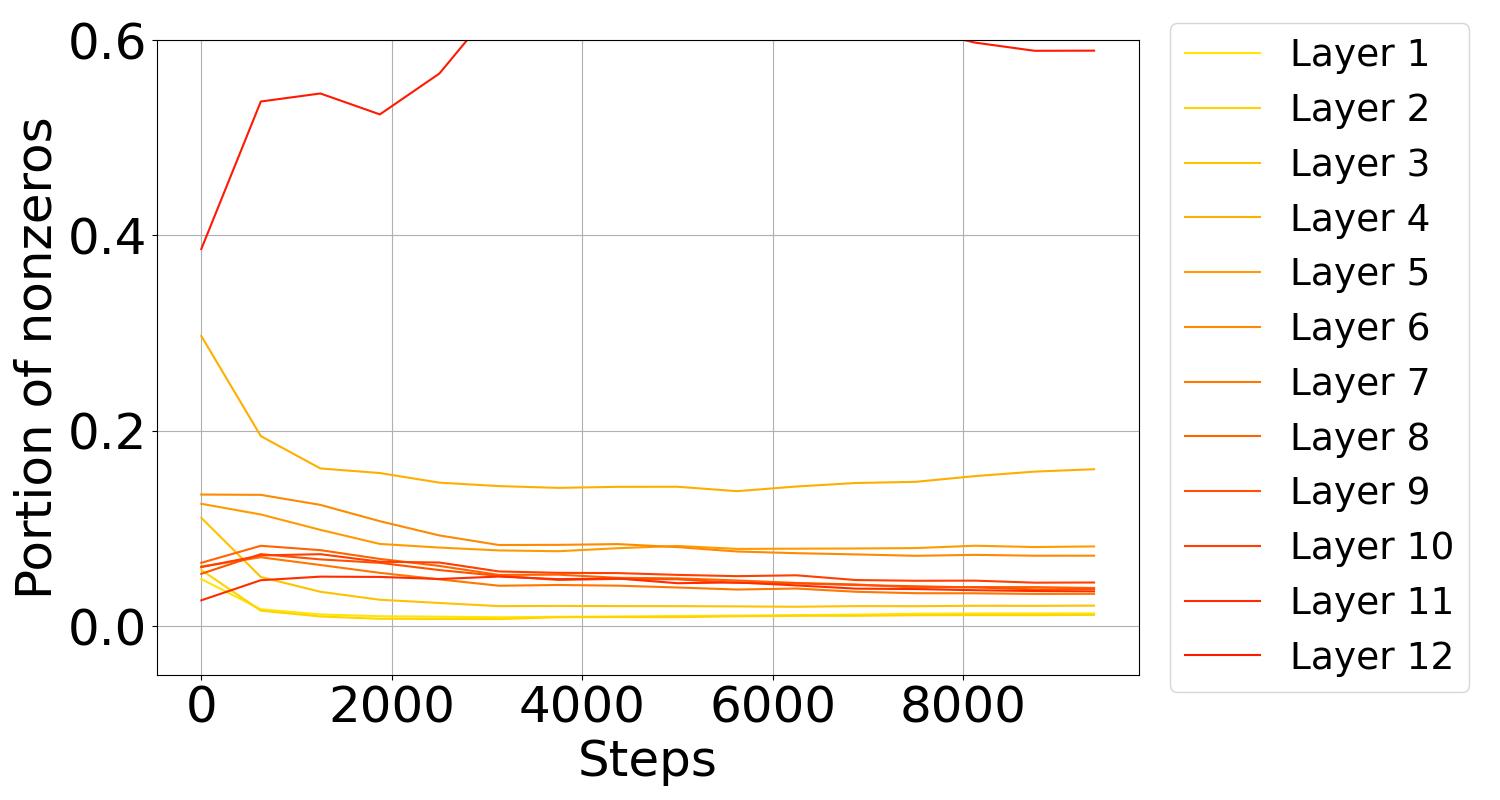}
        \caption{\tiny Modified ViT-Base.}\label{figure:finetuning_vit_sparsified}
    \end{subfigure}
    \begin{subfigure}[t]{0.24\textwidth}
        \myincludegraphics[width=\textwidth]{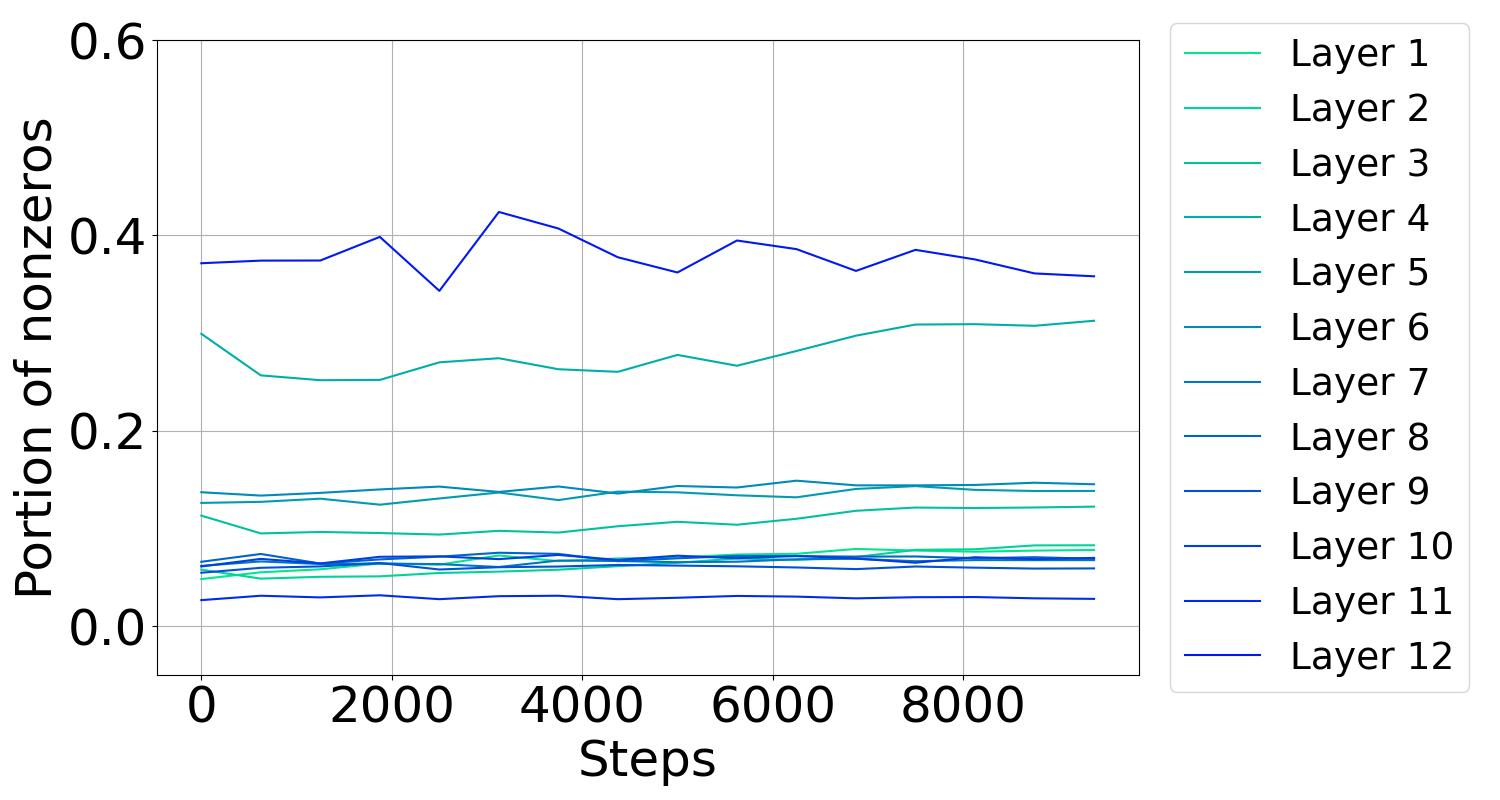}
        \caption{\tiny Vanilla ViT-Base.}\label{figure:finetuning_vit_vanilla}
    \end{subfigure}
    \begin{subfigure}[t]{0.24\textwidth}
        \myincludegraphics[width=\textwidth]{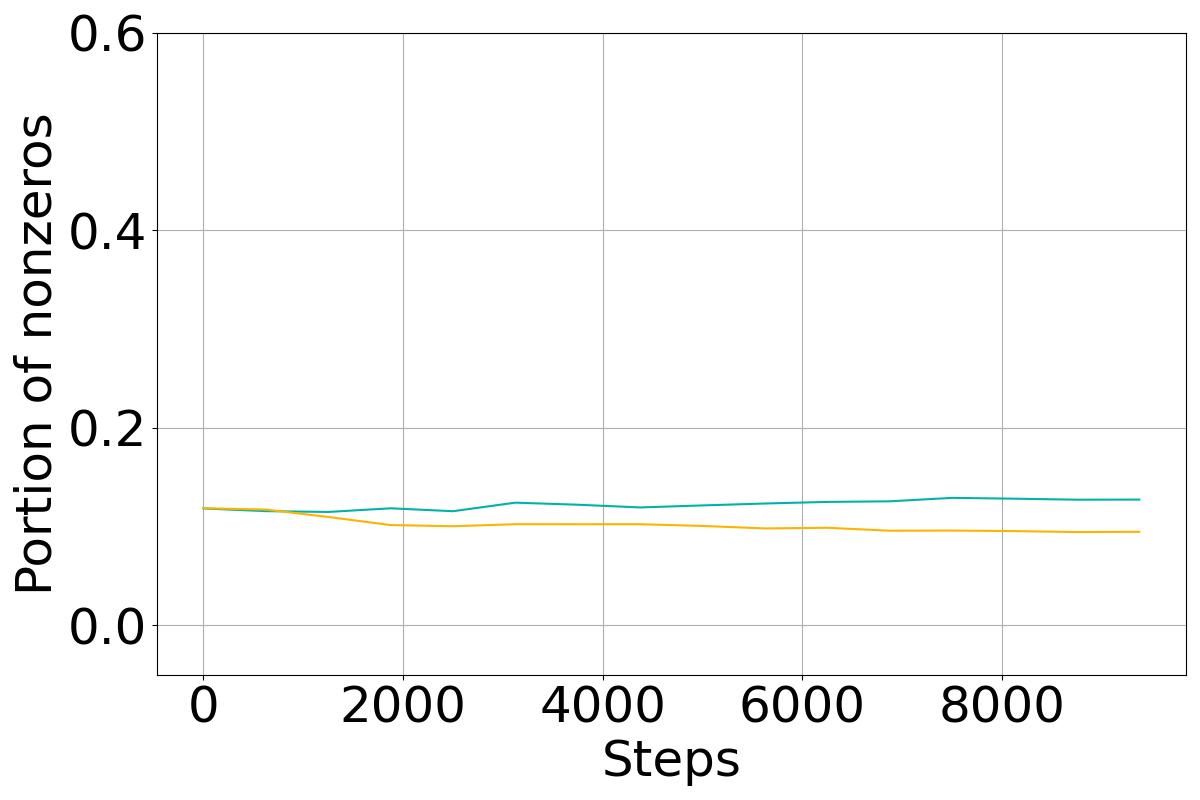}
        \caption{\tiny Testing sparsity averaged across layers during finetuning.}\label{figure:finetuning_vit_average}
    \end{subfigure}
    \begin{subfigure}[t]{0.24\textwidth}
        \myincludegraphics[width=\textwidth]{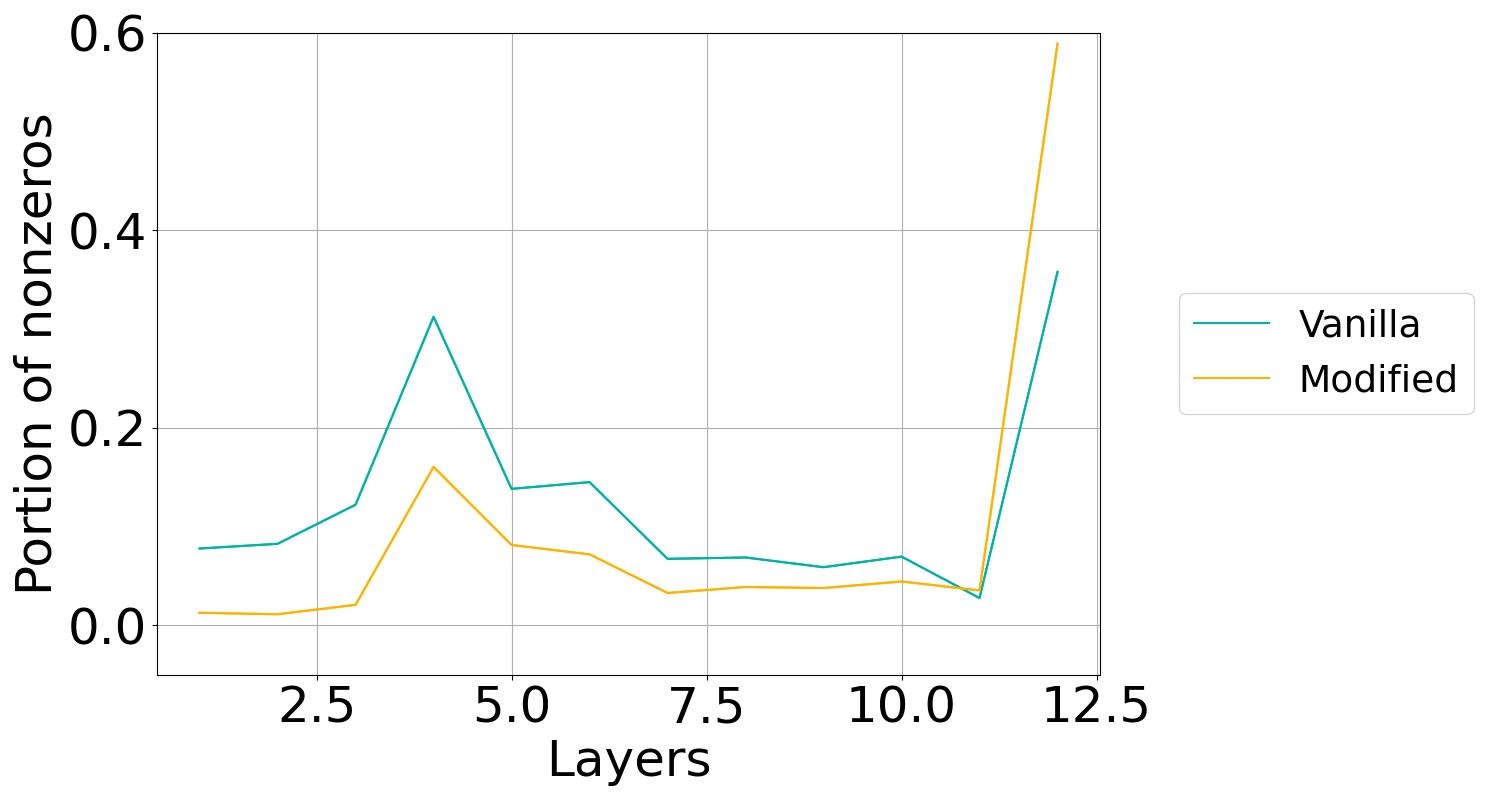}
        \caption{\tiny Layerwise testing sparsity during the last 100 steps.}\label{figure:finetuning_vit_end}
    \end{subfigure}
    \caption{Testing sparsity during finetuning ViT-Base/16 for sparsity on ImageNet-1K. Red and yellow are used for the modified model while blue and green indicate the vanilla one. }\label{figure:finetuning_vit}
\end{figure}
\begin{figure}
    \centering
    \resetHeight{}
    \begin{subfigure}[t]{0.24\textwidth}
        \myincludegraphics[width=\textwidth]{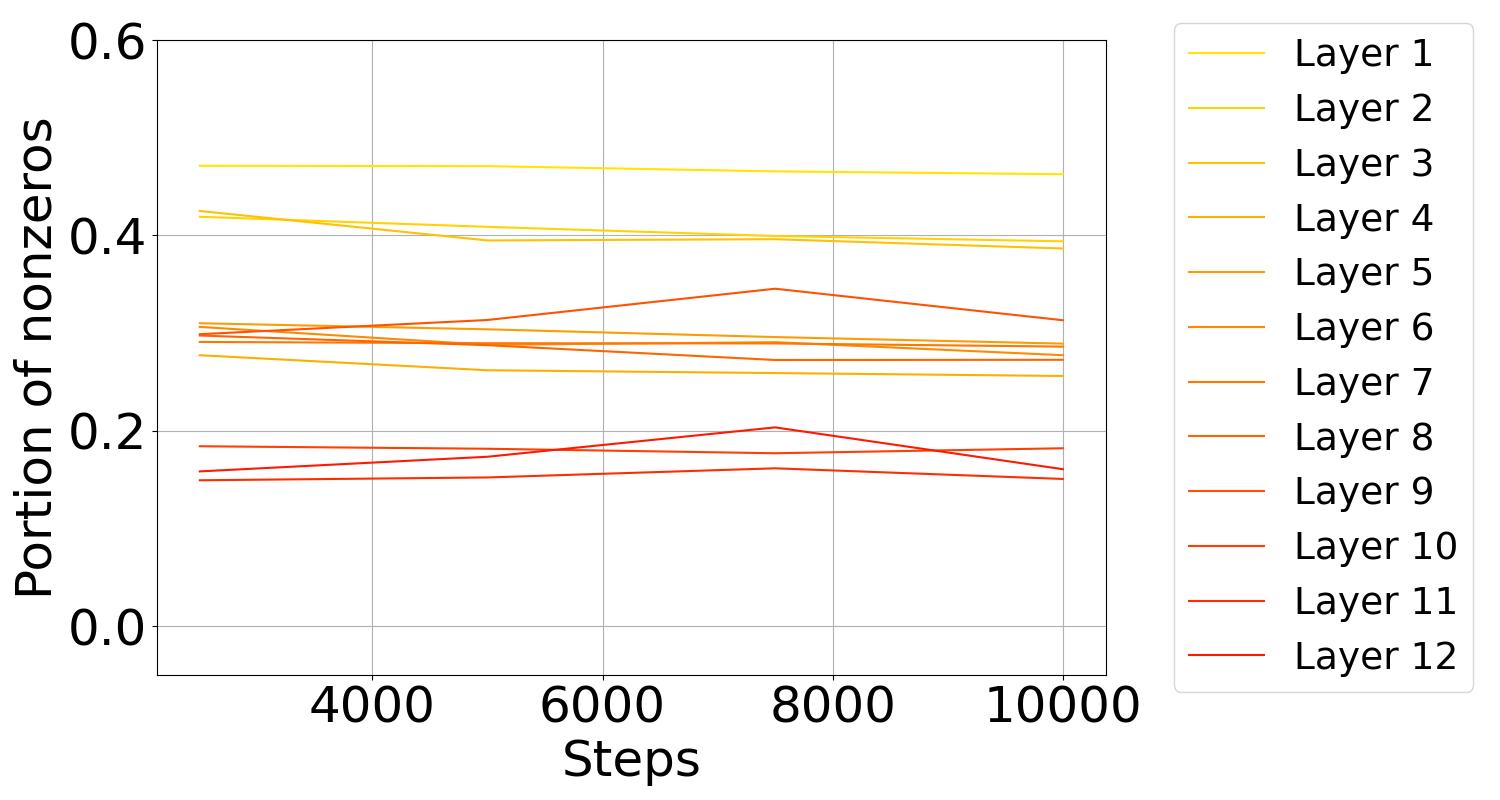}
        \caption{\tiny Encoder of modified T5.}\label{figure:finetuning_t5_sparsified_encoder}
    \end{subfigure}
    \begin{subfigure}[t]{0.24\textwidth}
        \myincludegraphics[width=\textwidth]{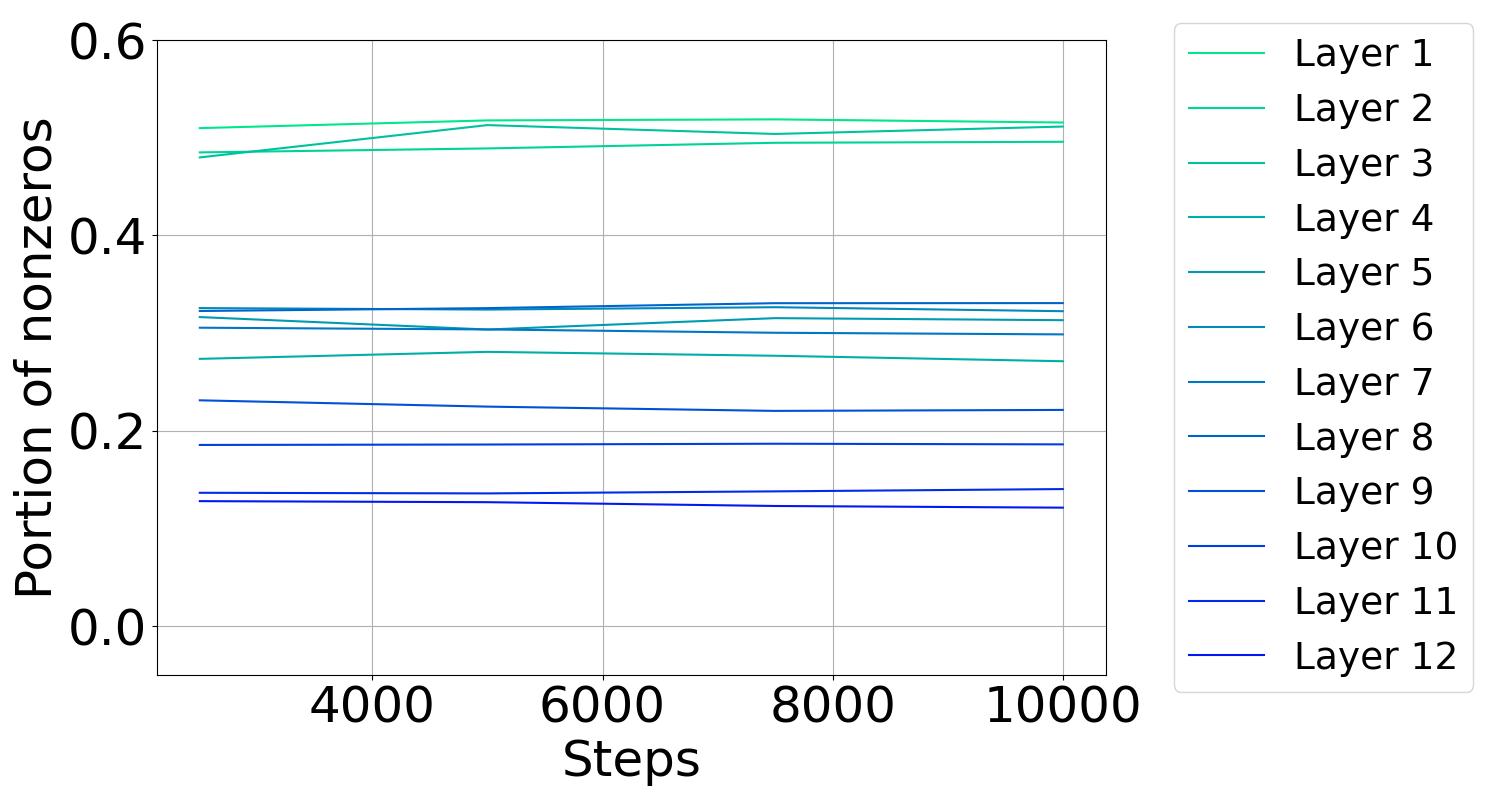}
        \caption{\tiny Encoder of Vanilla T5.}\label{figure:finetuning_vit_vanilla_encoder}
    \end{subfigure}
    \begin{subfigure}[t]{0.24\textwidth}
        \myincludegraphics[width=\textwidth]{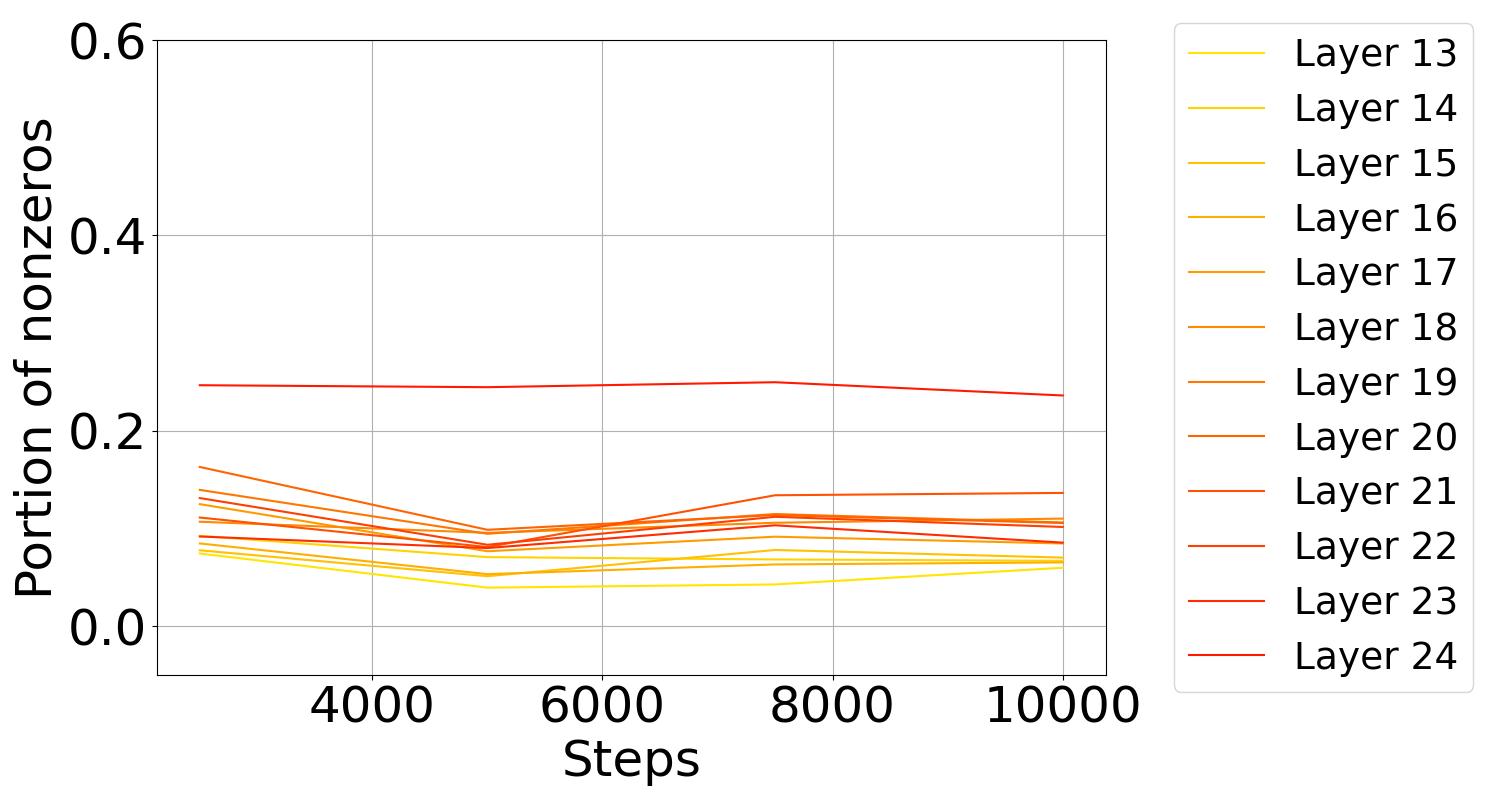}
        \caption{\tiny Decoder of modified T5.}\label{figure:finetuning_t5_sparsified_decoder}
    \end{subfigure}
    \begin{subfigure}[t]{0.24\textwidth}
        \myincludegraphics[width=\textwidth]{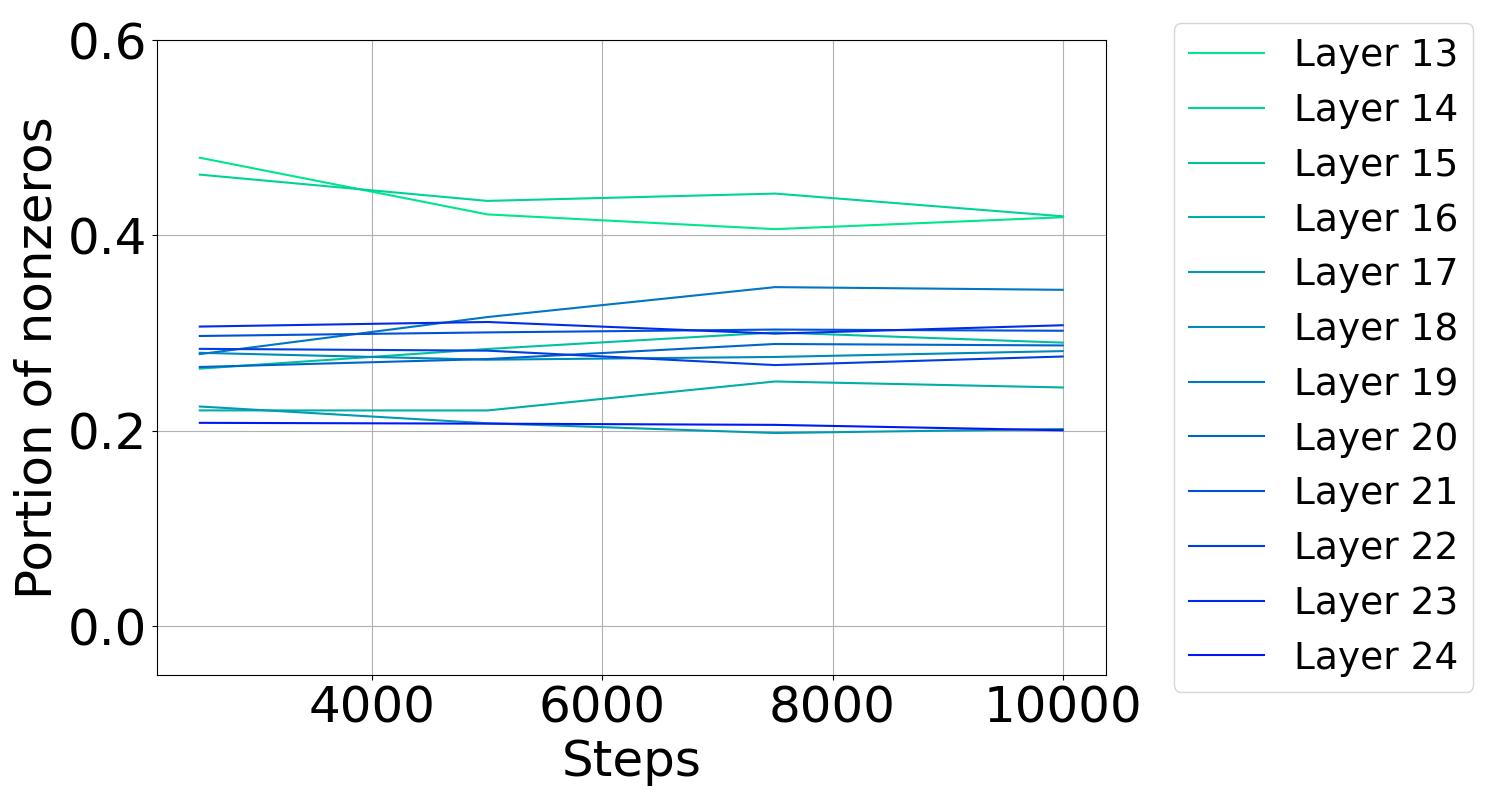}
        \caption{\tiny Decoder of Vanilla ViT-Base.}\label{figure:finetuning_t5_vanilla_decoder}
    \end{subfigure}
    \begin{subfigure}[t]{0.24\textwidth}
        \myincludegraphics[width=\textwidth]{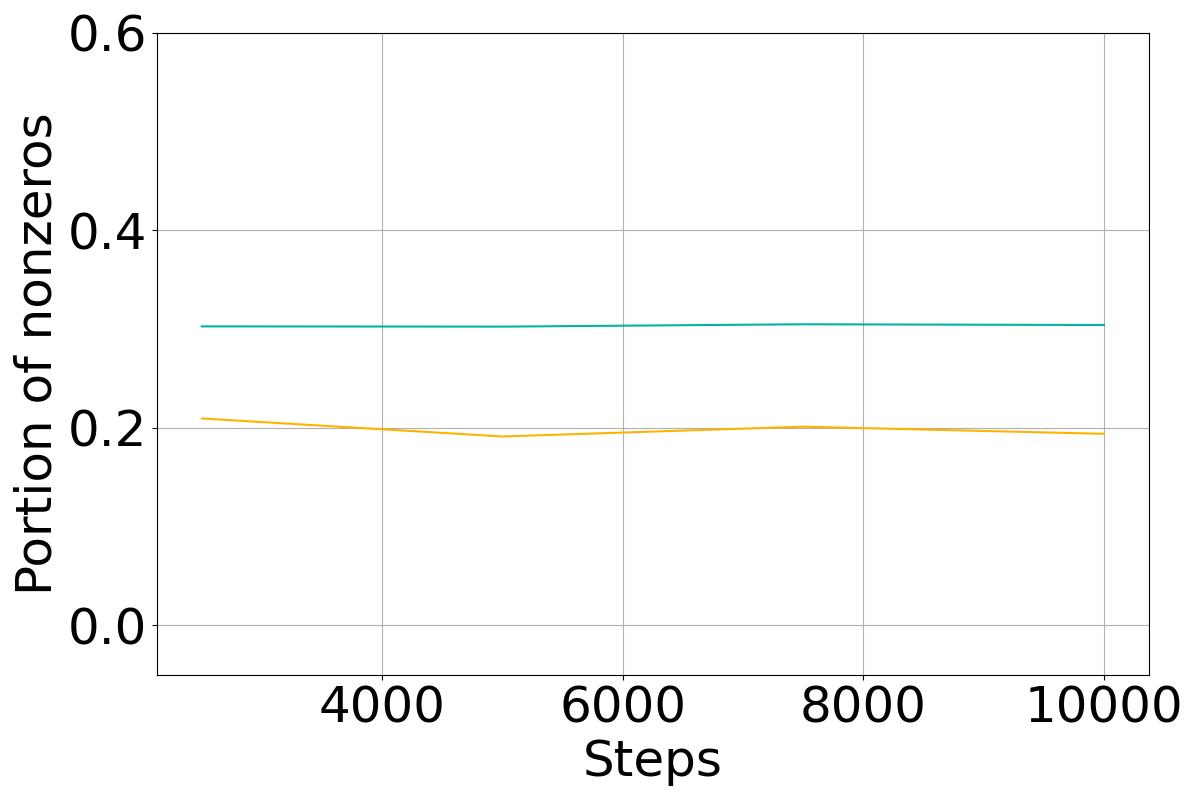}
        \caption{\tiny Testing sparsity averaged across layers during finetuning.}\label{figure:finetuning_t5_average}
    \end{subfigure}
    \begin{subfigure}[t]{0.24\textwidth}
        \myincludegraphics[width=\textwidth]{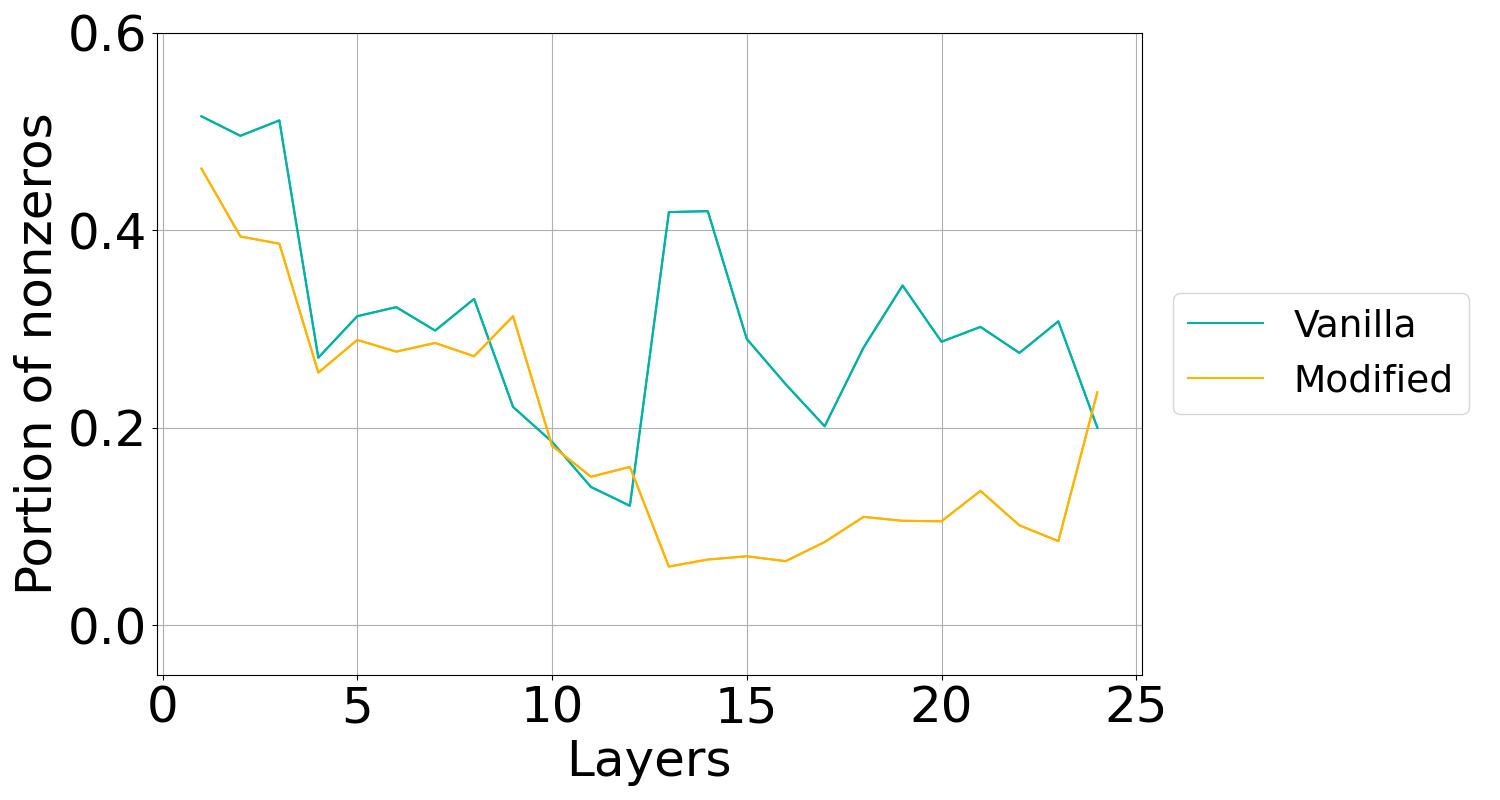}
        \caption{\tiny Layerwise testing sparsity at the end of finetuning.}\label{figure:finetuning_t5_end}
    \end{subfigure}
    \caption{Testing sparsity during finetuning T5-Base for sparsity on C4. Red and yellow are used for the modified model while blue and green indicate the vanilla one. }\label{figure:finetuning_t5}
\end{figure}
\begin{table}
    \centering
    \begin{tabular}{lrrrrrrrrrrr}
        \toprule
                        &   \multicolumn{3}{c}{ViT-Base/16 on ImageNet-1K}                                                                          &   \multicolumn{3}{c}{T5-Base on C4}\\
                        &   \multicolumn{1}{c}{Testing Sparsity}                &   \multicolumn{1}{c}{Acc@1}   &   \multicolumn{1}{c}{Acc@5}       &   \multicolumn{2}{c}{Testing Sparsity}                &   \multicolumn{1}{c}{Testing Loss}\\ 
        \midrule    
        Before Finetuning   &   $\mathbf{0.087}$                                    &   $\mathbf{77.35\%}$          &   $\mathbf{93.50\%}$          &   $0.299$                 &                           &   $4.88$\\
        Vanilla         &   $0.122$                                             &   $77.17\%$                   &   $93.44\%$                       &   $0.304$                 &                           &   $\mathbf{4.59}$\\
        Modified        &   $0.102$                                             &   $74.46\%$                   &   $92.10\%$                       &   $\mathbf{0.199}$        &   $\downarrow 33.42\%$    &   $4.61$\\
        \bottomrule
    \end{tabular}
    \caption{Summary of finetuning for sparsity.}\label{table:finetuning}
\end{table}

For ViT, the finetuning damages both testing sparsity and accuracy. However, modified finetuning still has advantages over vanilla LoRA finetuning in sparsity so it seems the drop is mainly due to finetuning. For T5 the improvement is significant. After finetuning for sparsity, T5 obtains similar testing sparsity compared to training for sparsity from scratch, indicating that existing language models can become sparser at a relatively low finetuning cost. This sparsity improvement only brings 0.02 degradation in testing loss. Despite the success of finetuning for sparsity in the NLP tasks, training for sparsity from scratch is still recommended for new models because it brings better testing sparsity as well as potential significant reduction in training costs from better training sparsity.
In ViT, sparsity improvement mainly happens in shallow and middle layers as illustrated in \cref{figure:finetuning_vit_end}, which can be explained in a similar way as \cref{sec:from_scratch}. For T5 the result is more interesting. According to \cref{figure:finetuning_t5_end} the sparsity improvement mainly happens in decoder layers. This may be related to architectural differences between encoders and decoders, which is of interest to investigate in future works.

\subsection{Role of LayerNorm Layers}

Although the experiments reported above are conducted with restricted or uplifted LayerNorm layers, we also conducted tentative experiments before these refinements. The results are not displayed in this manuscript due to length limit\jmlronly{ but they are available on our Huggingface repository\footnote{\logrepo{}}}. Without restricted LayerNorm, ViT has only about relatively 20\% sparsity improvements when trained from scratch, and with frozen not-uplifted LayerNorm, sparsity rarely changes during finetuning ViT from vanilla checkpoint whose scaling factors are observed very small. 
Both theory and experiments indicate the critical role of LayerNorm layers when it comes to sparsity.
Last but not least, \citet{observation} observe that there is \emph{no} activation sparsity in token mixing layers in MLP-Mixer. They attribute it to the small token-stacking dimension. Here based on the theoretical and empirical findings, we provide another conjecture that it is because MLP-Mixer does not have immediate LayerNorm layers before token-mixing MLP blocks \citep{mixer}, but with transposition lying in between.

%% file: experiments/theoretical.tex
\section{Empirical Supports for Assumptions}\label{sec:t_experiments}

In this section, evidence of assumptions required in \cref{sec:spectral_init} and \cref{sec:spectral_training} is empirically examined. 
The observations are taken during or after training from scratch or finetuning described by \cref{sec:p_experiments}.

\subsection{Spectral Increase in $\kkT$ and $M$}\label{sec:t_exp:spectral_increase}

\addvalue{kkT}{$\kkT$}
\addvalue{M}{$M$}
\addvalue{taskvanilla}{finetuning}
\addvalue{tasksparsified}{imagenet1k}
\addvalue{fulltaskvanilla}{finetuning for sparsity}
\addvalue{modelnamevanilla}{vanilla}
\addvalue{modelnamesparsified}{modified}
\addvalue{basemodeldirViT}{imagenet1k}
\addvalue{basemodeldirT5}{T5}
\addvalue{fullbasemodelViT}{ViT-Base/16 on ImageNet-1K}
\addvalue{fullbasemodelT5}{T5-Base on C4}
The traces of $\kkT$ and $M$ during training from scratch of ViT and T5 are displayed in \cref{figure:spectral_increase_ViT} and \cref{figure:spectral_increase_T5}, respectively.

\foreach \basemodel in {ViT, T5}{
    \begin{figure}
        \resetHeight{}
        \centering
        \foreach \modeltype in {sparsified, vanilla}{
            \foreach \matrixtype in {kkT, M}{
                \begin{subfigure}[t]{0.22\textwidth}
                    \centering
                    \myincludegraphics[width=\textwidth]{pic/results/dumps/\usevalue{basemodeldir\basemodel}/from_scratch/\modeltype/spectral_increase/\matrixtype.jpg}
                    \caption{\tiny \usevalue{\matrixtype} of \usevalue{modelname\modeltype} \basemodel}\label{figure:spectral_increase_\basemodel_\modeltype_\matrixtype}
                \end{subfigure}
            }
        }
        \caption{
            The traces of \usevalue{kkT} and \usevalue{M} of \usevalue{fullbasemodel\basemodel} during training from scratch.
        }\label{figure:spectral_increase_\basemodel}
    \end{figure}
}

In \cref{figure:spectral_increase_ViT}, both modified and vanilla ViTs have increasing $\trace{M} = \norm{g^l_K}_2^2$ and $\trace{\kkT}$ that increases rapidly first and decays slowly. Note that the decay in the \emph{late} phase of the training does not directly conflict with the assumption of theoretical results because the sparsity improvements of both modified and vanilla ViTs happen mainly in the \emph{early} stage of the training, where main sparsity improvemnts happen as well according to \cref{figure:productive_vit}. As a result, better implicit adversarial robustness and flatness cannot be obtained by shriking $\kkT$ or $M$.
Intriguingly in \cref{figure:spectral_increase_T5} for T5, $\kkT$ and $M$ swap their roles. $\trace{M}$ drop rapidly first and then stables. This indicates that although $\norm{g^l_V}_2^2$ decreases flatness can also be parallelly obtained through sparsity.

\subsection{Moderate Alignment in $\eta^l$ between $g^l$ and $\gamma^l$}\label{sec:t_exp:align_eff}

\cref{sec:t_exp:spectral_increase} have demonstrated that $\norm{\eta^l}_2^2$ is not decreased by decreasing $\norm{g^l}_2^2$, at least in ViT. 
To bring $\eta^l$ and $\gamma^l$ closer, we further demonstrate that decreased $\norm{\eta^l}_2^2 = \norm{g^l \hadamard \gamma^l}_2^2$ is not done by adversarially misaligning $g^l$ with $\gamma^l$, i.e., multiplying gradients $g^l$ of large magnitudes with derivatives of non-activated neurons and leaving gradients of small magnitudes to activated neurons.

We load checkpoints after training models on ImageNet-1K or C4, as described in \cref{sec:p_experiments}. For each checkpoint, we randomly select one batch of samples (2048 samples for ImageNet-1K or 256 samples for C4) to simulate batches encountered in training and use backward propagation to compute $g^l$, gradients w.r.t. activations, at all layers. The layerwise distribution of squared entries in $g^l$ after mixing all samples, tokens and entries together, as well as the distribution conditioned on activated neurons, is displayed in \cref{fig:alignment_eff}.

In \cref{fig:alignment_eff}, we observe that for most layers, activated neurons have gradients with a distribution similar to the distribution of all gradients. Furthermore, the former is even more concentrated, shorter-tailed and righter-shifted than the latter, in both ViT and T5 of both vanilla and modified architectures throughout training.
This moderate alignment indicates that there is no adversarial misalignment in $\eta^l$ and a considerable portion of entries in $\gamma^l$ have weights from $g^l$ that are relatively large and can be lowerbounded or approximated by $\ex{\norm{g^l}_2^2} / d$. 
Specifically under $\relu$-activation, \cref{lemma:eff_and_sparsity} is made meaningful, since $\ex{\left(g^l_{K, i, j}\right)^2 \mid \gamma^l_{i, j} > 0}$ is close to or even larger than $\ex{\left(g^l_{K, i, j}\right)^2}$ (see \cref{fig:alignment_ratio_imagenet1k_vanilla} and \cref{fig:alignment_ratio_T5_vanilla}) for most layers, so the former cannot be reduced since the increase in the latter at least in ViT as observed in \cref{sec:t_exp:spectral_increase}, and the only way to reduce effective gradient sparsity measured in $L_2$ norms is to improve gradient or activation sparsity directly measured in $L_0$ norms. As a result, our $\eta^l$-based theories are tightly related to activation sparsity measured in $L_0$ norm for $\relu$ networks and are approximately related for $\jrelu$ networks.

\begin{figure}
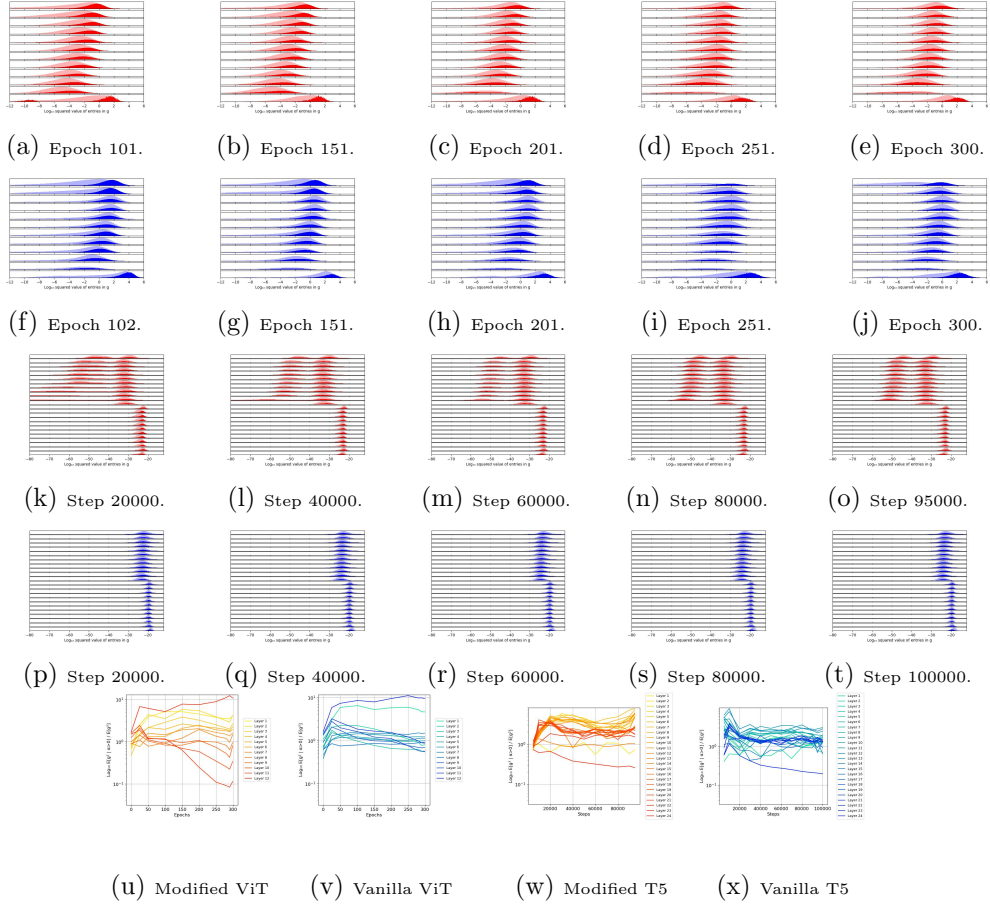

    \centering
    \resetHeight
    \foreach \model in {sparsified,vanilla}{
        \foreach \i in {100,150,200,250,299}{
            \ifthenelse{\equal{\model}{vanilla} \AND \equal{\i}{100}}
            {\renewcommand{\i}{101}}
            {}
            \pgfmathtruncatemacro{\epochid}{1+\i} 
            \begin{subfigure}[t]{0.15\textwidth}
                \myincludegraphics[width=\textwidth]{pic/results/dumps/imagenet1k/gradient_density/\model/\i.jpg}
                \caption{\tiny Epoch \epochid.}\label{fig:alignment_imagenet1k_\model_\i}
            \end{subfigure}
        }
        \\
    }

    \foreach \model in {sparsified,vanilla}{
        \foreach \i in {20000,40000,60000,80000,95000}{
            \ifthenelse{\equal{\model}{vanilla} \AND \equal{\i}{95000}}
            {\renewcommand{\i}{100000}}
            {}
            \begin{subfigure}[t]{0.15\textwidth}
                \myincludegraphics[width=\textwidth]{pic/results/dumps/T5/gradient_density/\model/\i.jpg}
                \caption{\tiny Step \i.}\label{fig:alignment_T5_\model_\i}
            \end{subfigure}
        }
        \\
    }

    \resetHeight
    \newcommand{\modelname}{}
    \newcommand{\taskname}{}
    \foreach \task in {imagenet1k,T5}{
        \foreach \model in {sparsified,vanilla}{
            \begin{subfigure}[t]{0.15\textwidth}
                \myincludegraphics[width=\textwidth]{pic/results/dumps/\task/gradient_density/\model/ratio.jpg}
                \ifthenelse{\equal{\model}{sparsified}}
                    {\renewcommand{\modelname}{Modified}}
                    {\renewcommand{\modelname}{Vanilla}}
                \ifthenelse{\equal{\task}{imagenet1k}}
                    {\renewcommand{\taskname}{ViT}}
                    {\renewcommand{\taskname}{T5}}
                \caption{\tiny \modelname{} \taskname{}}\label{fig:alignment_ratio_\task_\model}
            \end{subfigure}
        }
    }
    \caption{%
        Empirical distribution of $\log_{10}$-ed squared values of entries in $g^l$, observed in modified (red) and vanilla (blue) ViT-Base/16 (\cref{fig:alignment_imagenet1k_sparsified_100}-\cref{fig:alignment_imagenet1k_vanilla_299}) and T5 (\cref{fig:alignment_T5_sparsified_20000}-\cref{fig:alignment_T5_vanilla_100000}).
        The top rows indicate the first layers of the encoder while the bottom row indicates the last layer of the encoder or decoder.
        The semi-transparent distribution indicates the distribution of \emph{all} entries in $g^l$s, while the opaque distribution indicates that of only entries corresponding to activated neurons. The two distributions are visualized in an unnormalized manner using counts in equal-width bins, but the latter distribution is scaled larger to have better visualization.
        \cref{fig:alignment_ratio_imagenet1k_sparsified}-\cref{fig:alignment_ratio_T5_vanilla} compute $\log_{10} \frac{\ex[X, l, i, j]{\left(g^l_{K, i, j}\right)^2 \mid \alpha^l_{i, j} > 0}}{\ex[X, l, i, j]{\left(g^l_{K, i, j}\right)^2}}$ from the histograms.
    }\label{fig:alignment_eff}
\end{figure}

\subsection{Contribution of Gradient Sparsity and Direct Activation Sparsity}\label{sec:t_exp:contribution}

\cref{theorem:main}, \cref{theorem:main_with_hidden_vectors_and_layernorm} and \cref{theorem:main_with_effective_duplication} indicate that gradient sparsity and activation sparsity together lowerbound the implicit adversarial robustness and flatness, with coefficients of $\left(\sqrt{d} - c\right)^2$ and $\norm{g_V^l}_2^2$ empirically. To provide support for our emphasis on gradient sparsity, we compare $\left(\sqrt{d} - c\right)^2$ and $\norm{g_V^l}_2^2$. 
Since scaling factors are controlled and tracked only in modified models, only results from modified training are provided.

\begin{figure}
    \centering
    \resetHeight{}
    \begin{subfigure}[t]{0.22\textwidth}
        \myincludegraphics[width=\textwidth]{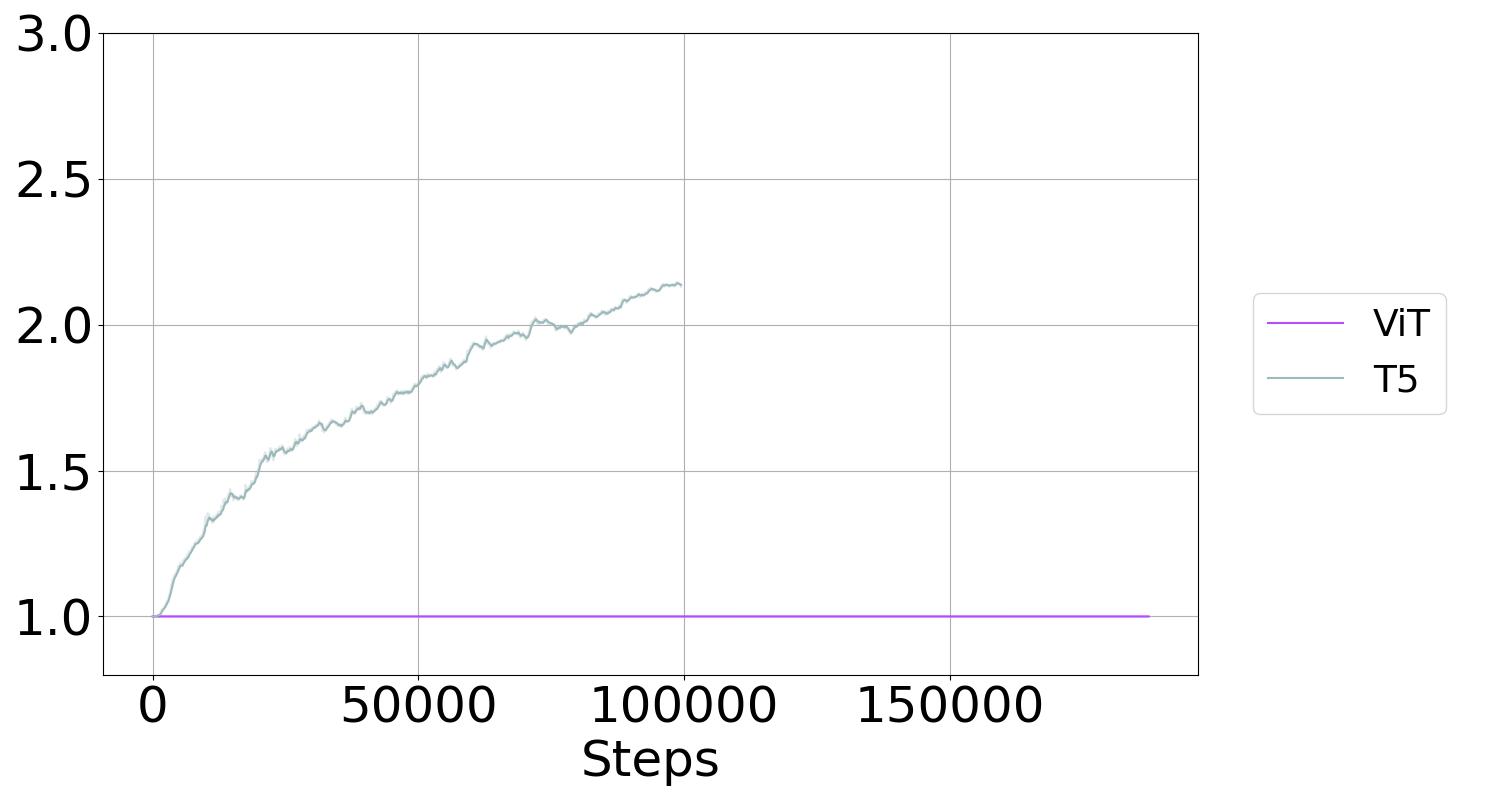}
        \caption{\tiny Change of scaling factors in LayerNorm layers.}\label{figure:layernorm}
    \end{subfigure}
    \begin{subfigure}[t]{0.22\textwidth}
        \myincludegraphics[width=\textwidth]{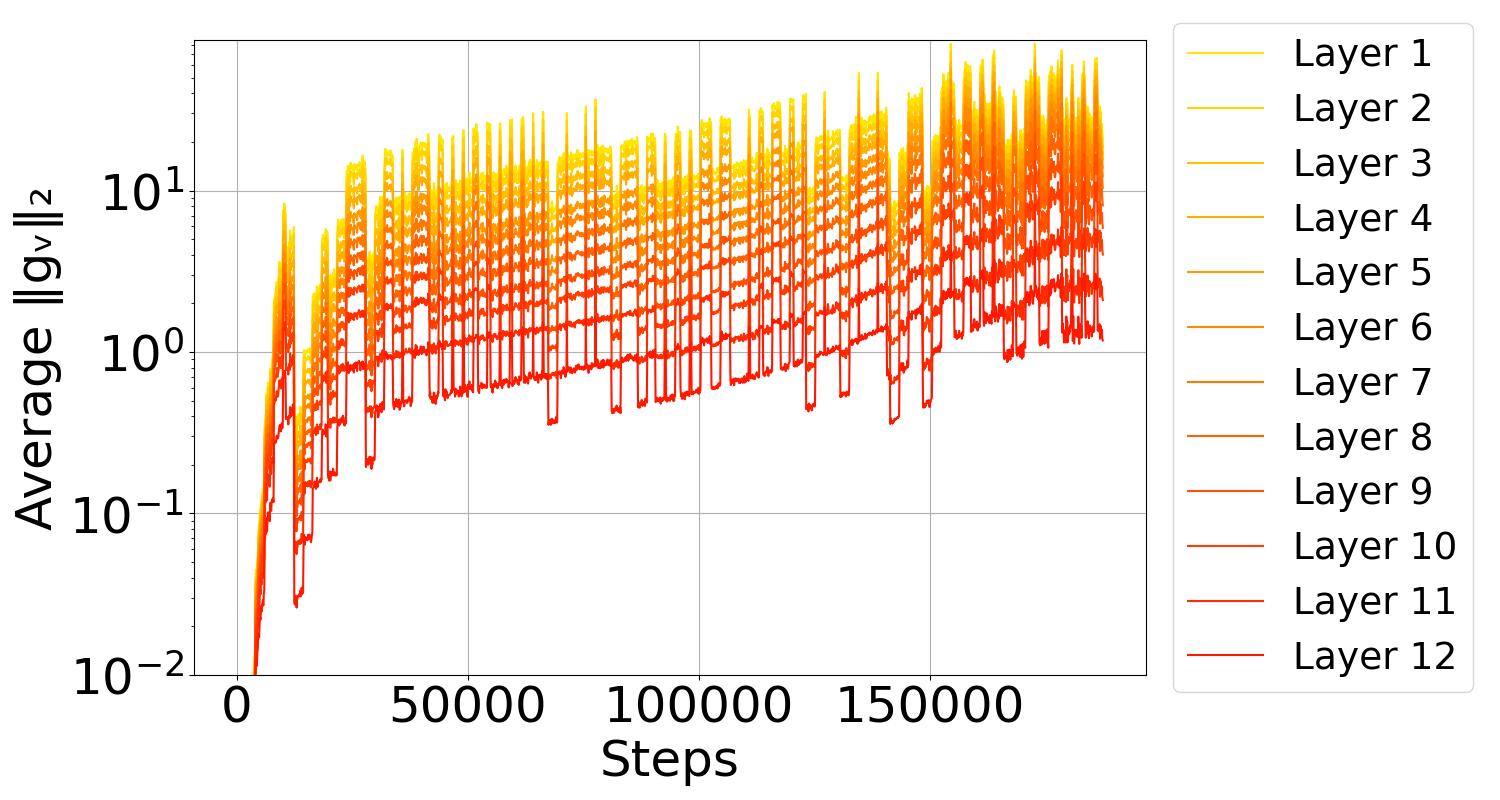}
        \caption{\tiny Average $\norm{g_V^l}_2$ in modified ViT.}\label{figure:g_V_vit}
    \end{subfigure}
    \begin{subfigure}[t]{0.22\textwidth}
        \myincludegraphics[width=\textwidth]{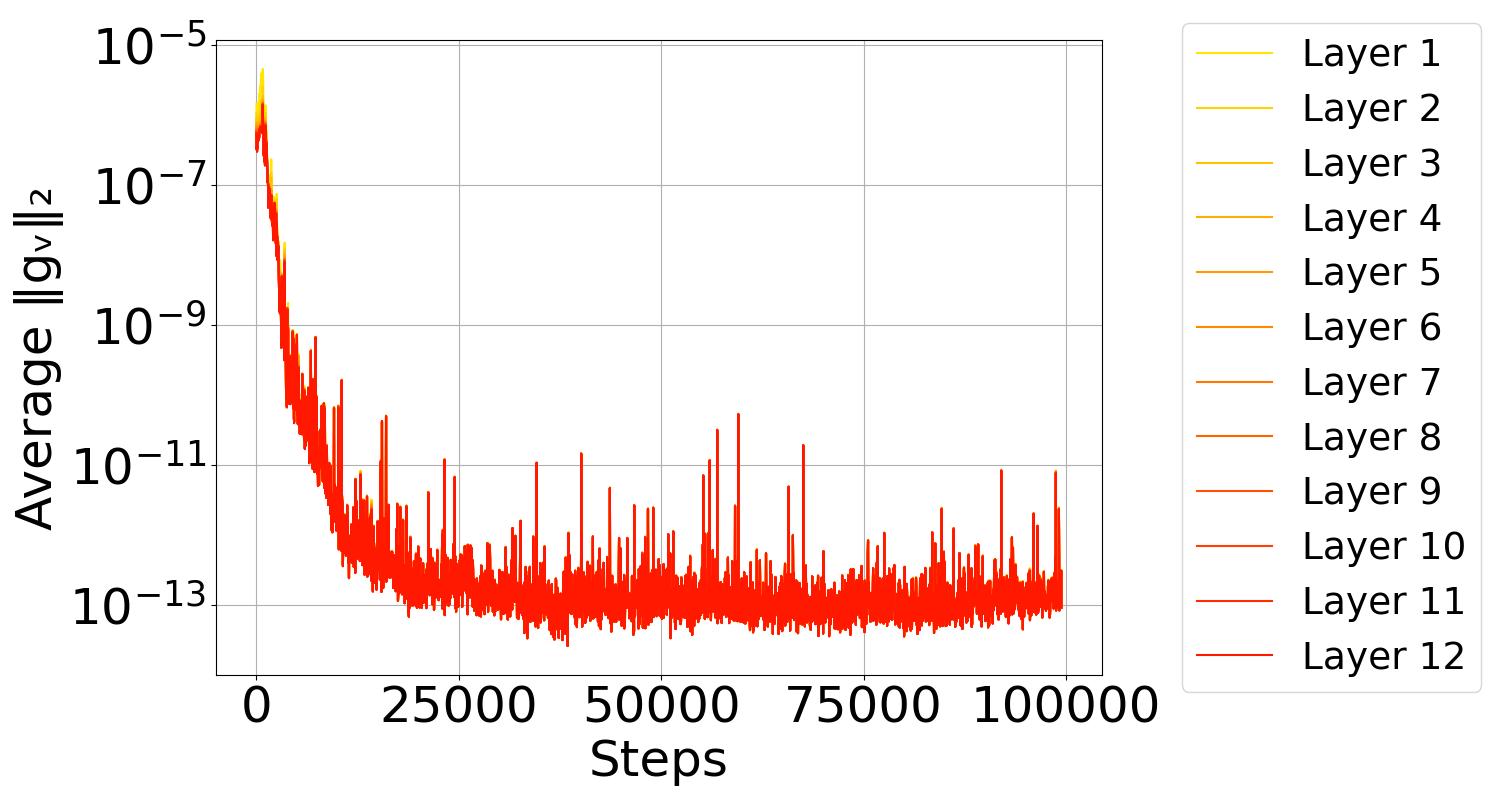}
        \caption{\tiny Average $\norm{g_V^l}_2$ in encoder layers of modified T5.}\label{figure:g_V_T5_encoder}
    \end{subfigure}
    \begin{subfigure}[t]{0.22\textwidth}
        \myincludegraphics[width=\textwidth]{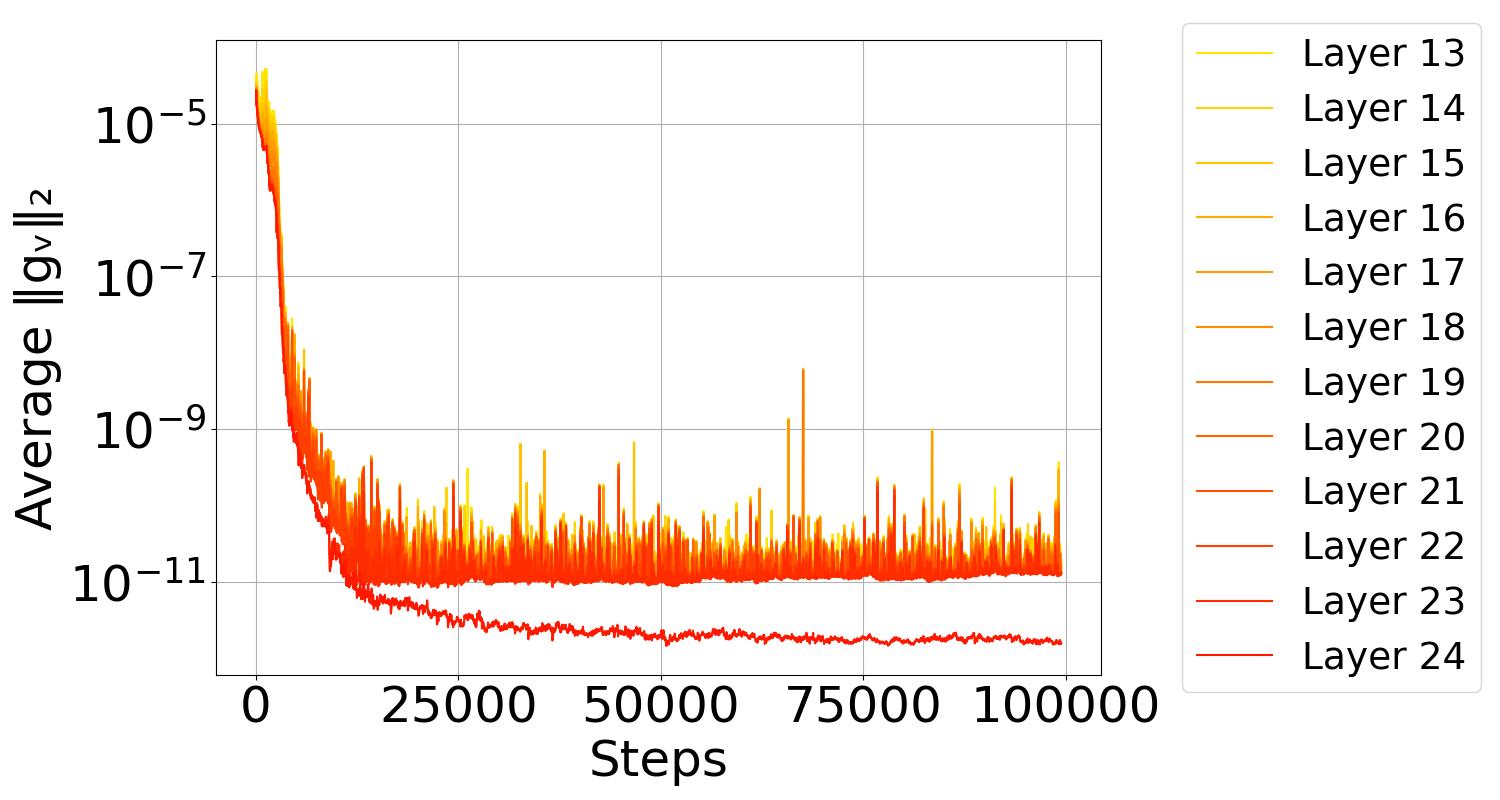}
        \caption{\tiny Average $\norm{g_V^l}_2$ in decoder layers of modified T5.}\label{figure:g_V_T5_decoder}
    \end{subfigure}
    \caption{Contribution of gradient and direct activation sparsity.}\label{figure:contribution}
\end{figure}

\cref{figure:layernorm} shows an intriguing difference between ViT and T5, i.e., the scaling factors in T5 increase slowly while those in ViT stick to $1$, indicating the latter's tendency to hold still or to decrease. Whether it is because of architectural or subtle implementation differences or the more interesting differences between CV and NLP tasks remains mysterious. Nevertheless, \cref{figure:layernorm} indicates that the coefficients of gradient sparsity in ViT and T5 is at least $\left(0.9 \sqrt{d}\right)^2 = 0.81d \approx 622$ and $\sim \left(1.5 \times 0.9\sqrt{d}\right)^2 \approx 1.8 d \approx 1400$, respectively, most of the time.
On the other hand, \cref{figure:g_V_vit} suggests in ViT, $\norm{g_V^l}_2$ varies a lot across different layers and steps. In shallow layers it can be larger than $25$ in the late phase of training, surpassing the coefficients of gradient sparsity. But in deep layers, it is close to $1$, indicating gradient sparsity's dominance. As a result, for ViT, sparsity in shallow layers and late training relies on direct activation sparsity while that in deep layers is dominated by gradient sparsity. This observation accounts for the gradient sparsity's weakening in late training as displayed in \cref{fig:validation_full} because at that time the activation sparsity has enough coefficients to compete with gradient sparsity.
In T5, $\norm{g_V^l}_2$ is extremely small, indicating that sparsity is fully due to gradient sparsity instead of direct activation sparsity.

\subsection{Spectral Concentration of $\kkT$}\label{sec:t_exp:spectral_concentration}

The log-scaled histograms of eigenvalues of $\kkT$ at several epochs of modified ViT's entire training on ImageNet-1K are displayed in \cref{fig:eigenvalues_of_kkT}.
\begin{figure}
    \centering
    \resetHeight{}
    \begin{subfigure}[t]{0.22\textwidth}
        \centering
        \myincludegraphics[width=\textwidth]{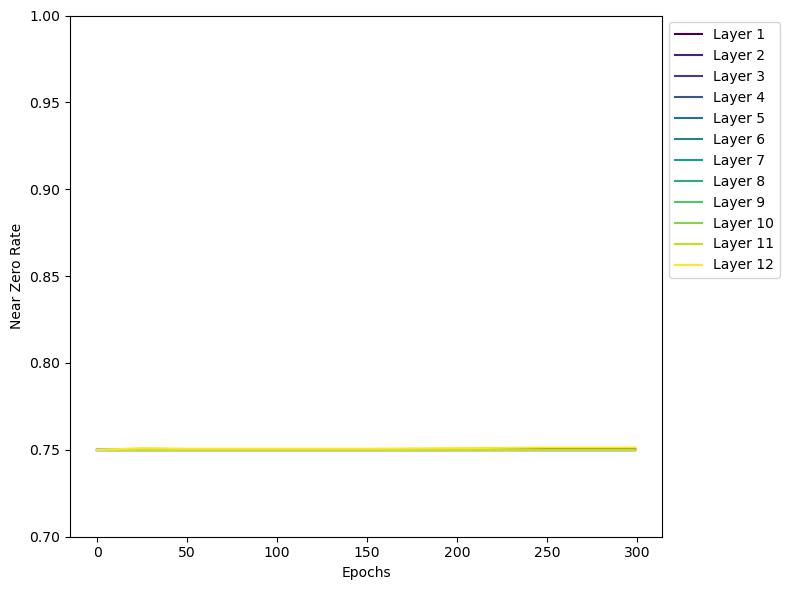}
        \caption{\scriptsize Portion of near-zero ($<10^{-3}$) eigenvalues.}\label{fig:near_zero_rate}
    \end{subfigure}
    \foreach \i in {0, 50, 100, 150, 200, 250, 299}{
        \begin{subfigure}[t]{0.2\textwidth}
            \centering
            \myincludegraphics[width=\textwidth]{pic/results/dumps/imagenet1k/from_scratch/sparsified/spectral/eigenvalues_\i.jpg}
            \pgfmathtruncatemacro{\epochid}{1+\i} 
            \caption{\scriptsize Epoch $\epochid$.}\label{fig:spectral_epoch_\i}
        \end{subfigure}
    }
    \\
    \begin{subfigure}[t]{0.3\textwidth}
        \centering
        \includegraphics[width=\textwidth]{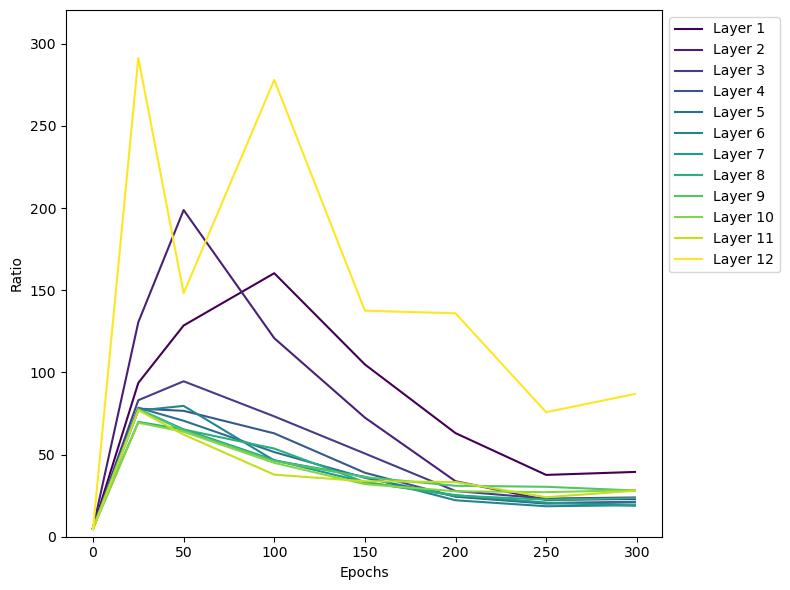}
        \caption{\scriptsize Ratios between upper- and lower-bounds of the majority of eigenvalues.}\label{fig:ratio_full}
    \end{subfigure}
    \begin{subfigure}[t]{0.3\textwidth}
        \centering
        \includegraphics[width=\textwidth]{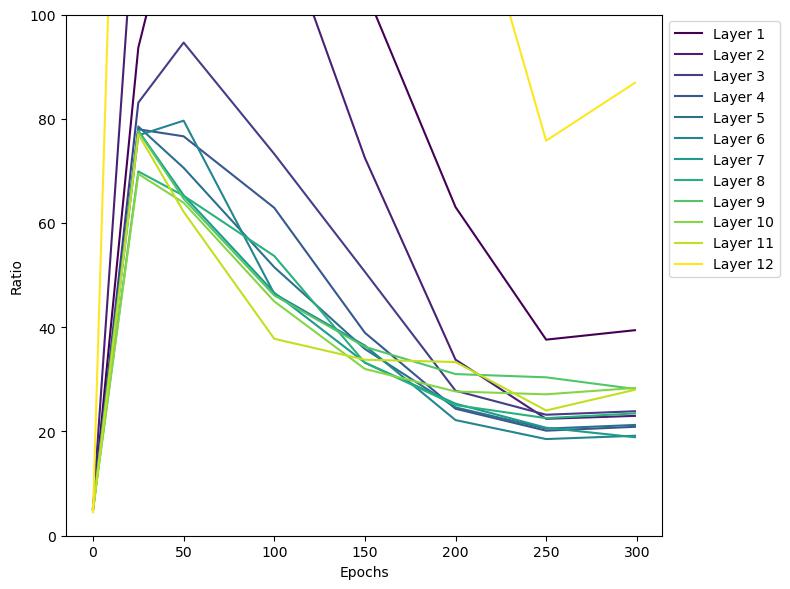}
        \caption{\scriptsize Ratios between upper- and lower-bounds of the majority of eigenvalues, focusing on deep layers and the later half of training.}\label{fig:ratio}
    \end{subfigure}
    \caption{Empirical spectral distribution of the first weight matrices in MLP blocks, observed in modified ViT-Base/16 trained on ImageNet-1K with PyTorch's recipe \citep{pytorch_recipe}. We consider eigenvalues $< 10^{-3}$ as near-zero eigenvalues. \cref{fig:near_zero_rate} displays the portion of near-zero eigenvalues during training. \cref{fig:spectral_epoch_0}-\cref{fig:spectral_epoch_299} display layerwise distributions of \emph{non-near-zero} eigenvalues at different checkpoints by showing the histograms of $\log_{10} \lambda_i$. The top row indicates Layer $1$ while the bottom one represents Layer $12$.  In these distributions, we annotate the majority, i.e., the shortest interval that covers at least 70\% (approximately one sigma) of eigenvalues, with red vertical lines. We compute the width of the majority and un-log it, illustrating the ratio within which the majority varies in \cref{fig:ratio_full} and \cref{fig:ratio}.
    }\label{fig:eigenvalues_of_kkT}
\end{figure}
It can be seen that there is a portion of near-zero ($< 10^{-3}$) eigenvalues, which are well separated from non-near-zero ones. 
The portion of near-zero values in \cref{fig:near_zero_rate} is stable and consistent with the theoretical expectation computed by the shape of $K^l$.
For non-near-zero eigenvalues, shallow layers and the last layer tend to have non-concentrating eigenvalues, but all of them still vary within the ratio of $300$ and most of them vary within the ratio of $100$, despite that there are about $3072 \times \frac{1}{4} = 768$ non-near-zero eigenvalues. For deep layers, the majority varies within a ratio of $80$, and the ratio decreases to about $40$ during the ending phase of training, indicating a tendency toward concentration as the training proceeds.

\begin{figure}
    \centering
    \resetHeight{}
    \begin{subfigure}[t]{0.22\textwidth}
        \centering
        \myincludegraphics[width=\textwidth]{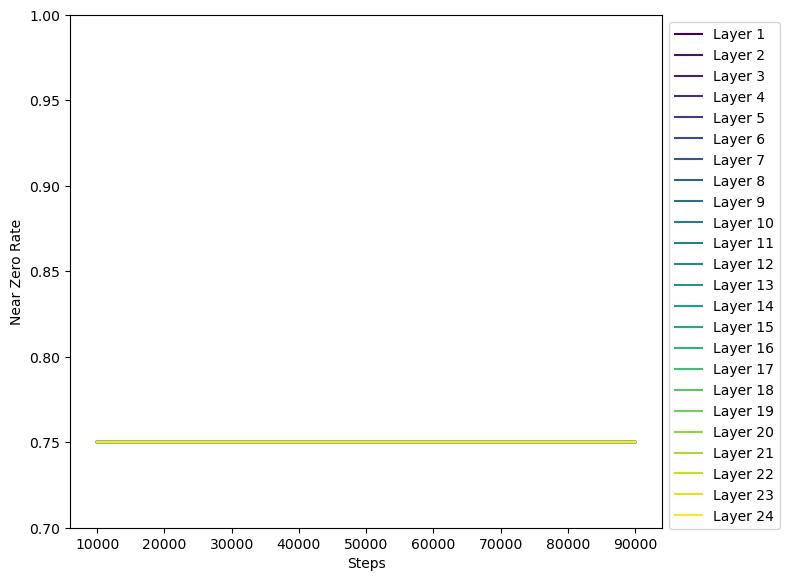}
        \caption{\scriptsize Portion of near-zero ($<10^{-1}$) eigenvalues.}\label{figure:spectral_concentration_t5_near_zero_rate}
    \end{subfigure}
    \foreach \i in {10000, 30000, 50000, 70000, 90000}{
        \begin{subfigure}[t]{0.24\textwidth}
            \centering
            \includegraphics[width=\textwidth]{pic/results/dumps/T5/from_scratch/sparsified/spectral/eigenvalues_\i.jpg}
            \caption{\scriptsize Step $\i$.}\label{figure:spectral_concentration_t5_step_\i}
        \end{subfigure}
    }
    \begin{subfigure}[t]{0.24\textwidth}
        \centering
        \includegraphics[width=\textwidth]{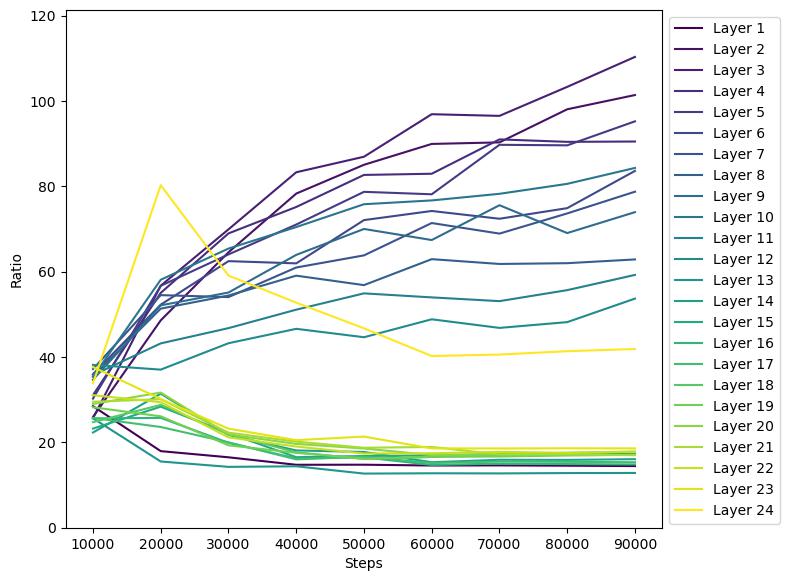}
        \caption{\scriptsize Ratios between upper- and lower-bounds of the majority of eigenvalues.}\label{figure:spectral_concentration_t5_ratio_full}
    \end{subfigure}
    \caption{Empirical spectral distribution of the first weight matrices in MLP blocks, observed in modified T5 trained on C4. We consider eigenvalues $< 10^{-1}$ as near-zero eigenvalues. \cref{figure:spectral_concentration_t5_near_zero_rate} displays the portion of near-zero eigenvalues during training. \cref{figure:spectral_concentration_t5_step_10000}-\cref{figure:spectral_concentration_t5_step_90000} display layerwise distributions of \emph{non-near-zero} eigenvalues at different checkpoints by showing the histograms of $\log_{10} \lambda_i$. The top row indicates Layer $1$ in the encoder while the bottom one represents Layer $12$ in the decoder.  In these distributions, we annotate the majority, i.e., the shortest interval that covers at least 70\% (approximately one sigma) of eigenvalues, with red vertical lines. We compute the width of majority and un-log it, illustrating ratio within which the majority varies in \cref{figure:spectral_concentration_t5_ratio_full}.
    }\label{figure:eigenvalues_of_kkT_t5}
\end{figure}
As displayed in \cref{figure:eigenvalues_of_kkT_t5}, similar phenomena can be observed in the decoder of T5, even though ViT is closer to an encoder. 
In contrast, the encoder tends to have non-concentrating eigenvalues, demonstrating another interesting difference between encoders and decoders. Nevertheless, the ratio between extreme non-zero eigenvalues in the encoder is not too large during the training. 
Another difference with ViT is that eigenvalues are much larger in T5, indicating even stronger drives toward (effective) gradient sparsity through the first terms.

\subsection{Applicability of \cref{theorem:spectral_of_accumulated}}\label{sec:t_exp:anisotropy}

We measure anisotropy $\sqrt{p \trace{\left(T - I\right)^\transpose \left(T - I\right)}}$ and $a$ empirically, as required in \cref{sec:spectral_training}. They are compared to hidden dimension $p=d$ or $n$ to build empirical evidence for application of \cref{theorem:spectral_of_accumulated}.

Since hidden dimension and weight decay have to be altered, the experiment is conducted on a small data set MNIST\citep{mnist} and a tiny pure MLP with 1 input layer, 1 classifier layer and $4$ hidden linear layers with $\relu$, each with hidden dimension $d$ to be altered. Skip connections and LayerNorm layers are equipped. The models are trained using Cross Entropy Loss and SGD, according to the implicit assumption of \cref{theorem:spectral_of_accumulated}. Inputs of the hidden layers and gradients w.r.t. linear layers' outputs are collected as $x_k$ vectors and are used to compute $a$, after which $u_k$s are computed by rescaling. 
\cref{theorem:spectral_of_accumulated} requires normalization on $x_k$ but it is inefficient to  re-normalize after every step. Observing $a S^p$ is scaling invariant w.r.t. $S^p$ and $a$ essentially captures the expected norm upperbound in $u_k$ in the proof, we directly compute $T^p = a S^p$ without normalizing $S^p$ and compute the expected norms in $u_k$, or $\trace{T^p}$ to illustrate $a$.

\arxivonly{\foreach \mywd in {0.0,0.01,0.1,0.3}{
\begin{figure}
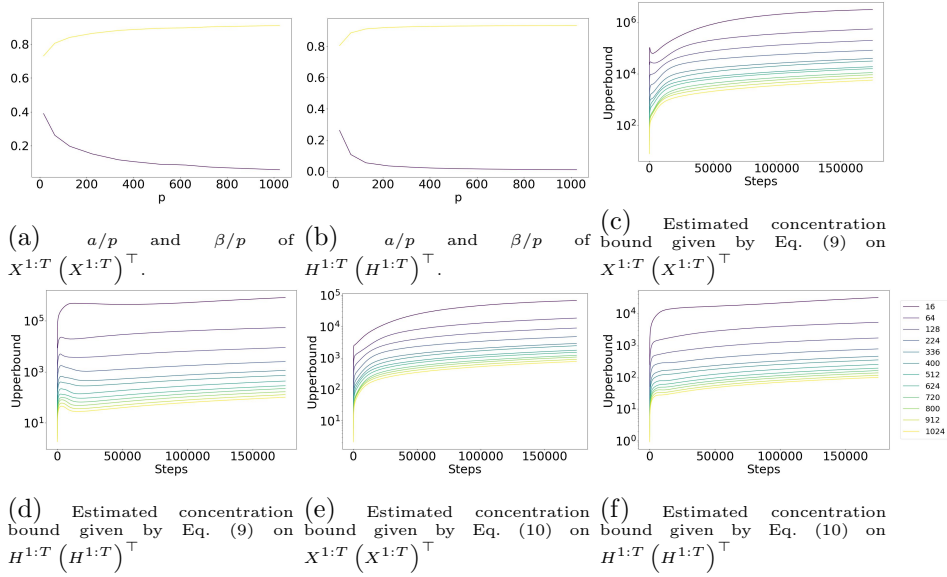

    \centering
    \resetHeight{}
    \begin{subfigure}{0.25\textwidth}
        \centering
        \expandafter\myincludegraphics[width=\textwidth]{pic/results/dumps/mp/lr\_epoch100/wd\mywd/sample.jpg}
        \caption{\tiny $a / p$ and $\beta / p$ of $X^{1:T} \left(X^{1:T}\right)^\transpose$.}\label{fig:spectral_assumption_sample_\mywd}
    \end{subfigure}
    \begin{subfigure}{0.25\textwidth}
        \centering
        \expandafter\myincludegraphics[width=\textwidth]{pic/results/dumps/mp/lr\_epoch100/wd\mywd/gradient.jpg}
        \caption{\tiny $a / p$ and $\beta / p$ of $\Eta^{1:T} \left(\Eta^{1:T}\right)^\transpose$.}\label{fig:spectral_assumption_gradient_\mywd}
    \end{subfigure}
    \begin{subfigure}{0.25\textwidth}
        \centering
        \expandafter\myincludegraphics[width=\textwidth]{pic/results/dumps/mp/lr\_epoch100/wd\mywd/sample_im_bound.jpg}
        \caption{\tiny Estimated concentration bound given by \cref{eq:im_bound} on $X^{1:T} \left(X^{1:T}\right)^\transpose$}\label{fig:spectral_assumption_sample_im_bound_\mywd}
    \end{subfigure}
    \begin{subfigure}{0.25\textwidth}
        \centering
        \expandafter\myincludegraphics[width=\textwidth]{pic/results/dumps/mp/lr\_epoch100/wd\mywd/gradient_im_bound.jpg}
        \caption{\tiny Estimated concentration bound given by \cref{eq:im_bound} on $\Eta^{1:T} \left(\Eta^{1: T}\right)^\transpose$}\label{fig:spectral_assumption_gradient_im_bound_\mywd}
    \end{subfigure}
    \begin{subfigure}{0.25\textwidth}
        \centering
        \expandafter\myincludegraphics[width=\textwidth]{pic/results/dumps/mp/lr\_epoch100/wd\mywd/sample_re_bound.jpg}
        \caption{\tiny Estimated concentration bound given by \cref{eq:re_bound} on $X^{1:T} \left(X^{1:T}\right)^\transpose$}\label{fig:spectral_assumption_sample_re_bound_\mywd}
    \end{subfigure}
    \begin{subfigure}{0.25\textwidth}
        \centering
        \expandafter\myincludegraphics[width=\textwidth]{pic/results/dumps/mp/lr\_epoch100/wd\mywd/gradient_re_bound.jpg}
        \caption{\tiny Estimated concentration bound given by \cref{eq:re_bound} on $\Eta^{1: T} \left(\Eta^{1: T}\right)^\transpose$}\label{fig:spectral_assumption_gradient_re_bound_\mywd}
    \end{subfigure}
    \caption{When weight decay is \mywd, values of $a / p, \beta / p$ required by \cref{theorem:spectral_of_accumulated} and concentration bounds of tiny pure MLPs with different choice of hidden dimension on MNIST. $a / p$ and $\beta / p$ do not change much during training so for each hidden dimension, values of $a / p$ and $\beta / p$ are averaged across steps to ease presentation. Layers are also averaged for the same reason.} \label{fig:assumptions_of_spectral_concentration_\mywd}
\end{figure}
}}

\jmlronly{\foreach \mywd in {0.0,0.3}{
\begin{figure}
    \centering
    \resetHeight{}
    \begin{subfigure}{0.25\textwidth}
        \centering
        \expandafter\myincludegraphics[width=\textwidth]{pic/results/dumps/mp/lr\_epoch100/wd\mywd/sample.jpg}
        \caption{\tiny $a / p$ and $\beta / p$ of $X^{1:T} \left(X^{1:T}\right)^\transpose$.}\label{fig:spectral_assumption_sample_\mywd}
    \end{subfigure}
    \begin{subfigure}{0.25\textwidth}
        \centering
        \expandafter\myincludegraphics[width=\textwidth]{pic/results/dumps/mp/lr\_epoch100/wd\mywd/gradient.jpg}
        \caption{\tiny $a / p$ and $\beta / p$ of $\Eta^{1:T} \left(\Eta^{1:T}\right)^\transpose$.}\label{fig:spectral_assumption_gradient_\mywd}
    \end{subfigure}
    \begin{subfigure}{0.25\textwidth}
        \centering
        \expandafter\myincludegraphics[width=\textwidth]{pic/results/dumps/mp/lr\_epoch100/wd\mywd/sample_im_bound.jpg}
        \caption{\tiny Estimated concentration bound given by \cref{eq:im_bound} on $X^{1:T} \left(X^{1:T}\right)^\transpose$}\label{fig:spectral_assumption_sample_im_bound_\mywd}
    \end{subfigure}
    \begin{subfigure}{0.25\textwidth}
        \centering
        \expandafter\myincludegraphics[width=\textwidth]{pic/results/dumps/mp/lr\_epoch100/wd\mywd/gradient_im_bound.jpg}
        \caption{\tiny Estimated concentration bound given by \cref{eq:im_bound} on $\Eta^{1:T} \left(\Eta^{1: T}\right)^\transpose$}\label{fig:spectral_assumption_gradient_im_bound_\mywd}
    \end{subfigure}
    \begin{subfigure}{0.25\textwidth}
        \centering
        \expandafter\myincludegraphics[width=\textwidth]{pic/results/dumps/mp/lr\_epoch100/wd\mywd/sample_re_bound.jpg}
        \caption{\tiny Estimated concentration bound given by \cref{eq:re_bound} on $X^{1:T} \left(X^{1:T}\right)^\transpose$}\label{fig:spectral_assumption_sample_re_bound_\mywd}
    \end{subfigure}
    \begin{subfigure}{0.25\textwidth}
        \centering
        \expandafter\myincludegraphics[width=\textwidth]{pic/results/dumps/mp/lr\_epoch100/wd\mywd/gradient_re_bound.jpg}
        \caption{\tiny Estimated concentration bound given by \cref{eq:re_bound} on $\Eta^{1: T} \left(\Eta^{1: T}\right)^\transpose$}\label{fig:spectral_assumption_gradient_re_bound_\mywd}
    \end{subfigure}
    \caption{When weight decay is \mywd, values of $a / p, \beta / p$ required by \cref{theorem:spectral_of_accumulated} and concentration bounds of tiny pure MLPs with different choices of hidden dimension on MNIST. $a / p$ and $\beta / p$ do not change much during training so for each hidden dimension, values of $a / p$ and $\beta / p$ are averaged across steps to ease presentation. Layers are also averaged for the same reason.} \label{fig:assumptions_of_spectral_concentration_\mywd}
\end{figure}
}}

\cref{fig:assumptions_of_spectral_concentration_0.0}-\cref{fig:assumptions_of_spectral_concentration_0.3} show anisotropies, norms and spectral concentration bounds computed by effective window sizes under weight decay of \arxivonly{$0.0, 0.01, 0.1$ and $0.3$.}\jmlronly{$0.0$ and $0.3$. More results for weight decay $0.01$ and $0.1$ can be found in the Huggingface repository\footnote{\logrepo{}} for this manuscript.}
From subfigures (a) and (b) in them, $a / p$ significantly drops as $p$ increases while $\beta / p$ increases but never exceeds $1$. Since in bounds \cref{eq:im_bound} and \cref{eq:re_bound}, $\beta / p$ is only found within the square root, the increase of $\beta / p$ has relatively small effects on the bounds. As a result, in subfigures (c)-(f), the estimated bounds decrease as the hidden dimension increases under all experimented weight decay intensities. 
Among the two bounds, it is theoretically and empirically observed that \cref{eq:re_bound} based on real parts is more desirable.

As discussed in \cref{sec:spectral_training}, weight decay effectively limits $T$ and stops bounds' divergence. From \cref{fig:assumptions_of_spectral_concentration_0.0} to \cref{fig:assumptions_of_spectral_concentration_0.3} we can see that as weight decay intensifies, the bounds computed by effective windows size drop significantly\arxivonly{, especially in \cref{fig:assumptions_of_spectral_concentration_0.3} where $w=0.3$}.
Weight decay not only upperbounds $\frac{1}{c}$, which happens at approximately $50,000$ steps in \cref{fig:assumptions_of_spectral_concentration_0.3}, but also decreases $a / p$. Note this is not trivial because $\trace{a S^p}$ is scaling invariant. How weight decay causes this drop is worth investigating in future works.
When training ViTs, weight decay is usually as large as 0.1 or 0.3 \citep{vit,pytorch_recipe}, so bounds estimated in \arxivonly{\cref{fig:assumptions_of_spectral_concentration_0.1} and }\cref{fig:assumptions_of_spectral_concentration_0.3} suit the practices best, where both $X^{1:T} \left(X^{1:T}\right)^\transpose$ and $\Eta^{1: T} \left(\Eta^{1: T}\right)^\transpose$ have relatively low spectral ratio bounds smaller than $400$ or $100$ when hidden dimension is relatively large, which is a low value even though there are hundreds of dimensions.

%% file: conclusion/conclusion.tex
\section{Conclusion}

In this work, we explain the activation sparsity observed by \citet{observation} with gradient sparsity and effective gradient sparsity, which emerges from flatness and implicit adversarial robustness. 
Proofs are done for pure MLPs and other architectures including Transformers under hypothetical massively perturbed training. LayerNorm plays a critical role in them.
To argue the effectiveness of zeroth biases, we analyze the phenomenon of spectral concentration in weight matrices by introducing random matrix theory to the research of training dynamics. 
We propose two sparsity-oriented architectural plug-and-play modifications and one radical modification. 
Experiments for verification, where our theory predicts well, rule out some potential explanations and provide support for our emphasis on gradient sparsity. 
We test our modification on ImageNet-1K and C4 to demonstrate its great practical effectiveness. 
Finetuning for sparsity demonstrates a cheaper way toward better sparsity for existing models. Assumptions made during analyses are empirically validated.

Having demonstrated how sparsity emerges from noises and the robustness against them, we wildly conjecture that a similar explanation applies to sparsity, and thus energy efficiency, in human brains and cognition, which face enormous environmental, sensory and synaptic noises as well as small-batch learning every day and exhibit great robustness against them. Our work also demonstrates that some architectural designs may greatly ease theoretical analyses in deep neural networks, and more theoretically oriented architectures can be proposed in the future. We introduce random matrix theory as a powerful tool in analyzing training dynamics, hoping for its broader application in the machine learning community.

%% file: theory/appendix.tex
\section{Proof of Lemmas}

\begin{proof}[Proof of \cref{lemma:flatness_and_grad_norm}]\label{proof:flatness_and_grad_norm}
    Expanding definitions and simple rearrangement give
    \begin{align}
        \hessian
        =&  \nabla_{\theta}^2 \ex[(X, Y) \sim \dataset]{\ce(f_\theta, (X, Y))} 
        = -\ex[(X, Y) \sim \dataset]{\nabla_\theta^2 \log f(Y \mid \theta, X)}. \label{step:expanded_hessian_definition}
    \end{align}
    We first deal with the Hessian within the expectation, i.e., for any $(y, x)$ there is
    \begin{align}
        \nabla_\theta^2 \log f(y \mid \theta, x)
        =&  J_\theta \left( \frac{\nabla_\theta f(y \mid \theta, x)}{f(y \mid \theta, x)} \right)\\
        =&  \frac{\left(\nabla_\theta^2 f(y \mid \theta, x)\right) f(y \mid \theta, x) - \left(\nabla_\theta f(y \mid \theta, x)\right)\left(\nabla_\theta f(y \mid \theta, x)\right)^\transpose}{f(y \mid \theta, x) f(y \mid \theta, x)}\\
        =&  \frac{\nabla_\theta^2 f(y \mid \theta, x)}{f(y \mid \theta, x)} - \left(\frac{\nabla_\theta f(y \mid \theta, x)}{f(y \mid \theta, x)}\right)\left(\frac{\nabla_\theta f(y \mid \theta, x)}{f(y \mid \theta, x)}\right)^\transpose\\
        =&  \frac{\nabla_\theta^2 f(y \mid \theta, x)}{f(y \mid \theta, x)} - \left(\nabla_\theta \log f(y \mid \theta, x)\right)\left(\nabla_\theta \log f(y \mid \theta, x)\right)^\transpose.
    \end{align}
    Plugging this equality back to \cref{step:expanded_hessian_definition} gives
    \begin{align}
        \hessian
        =&  -\ex[(X, Y) \sim \dataset]{\frac{\nabla_\theta^2 f(Y \mid \theta, X)}{f(Y \mid \theta, X)} - \left(\nabla_\theta \log f(Y \mid \theta, X)\right)\left(\nabla_\theta \log f(Y \mid \theta, X)\right)^\transpose}\\
        =&  \ex[(X, Y) \sim \dataset]{\left(\nabla_\theta \log f(Y \mid \theta, X)\right)\left(\nabla_\theta \log f(Y \mid \theta, X)\right)^\transpose}
             -\ex[(X, Y) \sim \dataset]{\frac{\nabla_\theta^2 f(Y \mid \theta, X)}{f(Y \mid \theta, X)}}, 
    \end{align}
    the trace of which is
    \begin{align}
        \trace{\hessian}
        =&  \ex[(X, Y) \sim \dataset]{\trace{\left(\nabla_\theta \log f(Y \mid \theta, X)\right)\left(\nabla_\theta \log f(Y \mid \theta, X)\right)^\transpose}}
         \\&- \ex[\dataset(X)]{\trace{\ex[\dataset(Y \mid X)]{\frac{\nabla_\theta^2 f(Y \mid \theta, X)}{f(Y \mid \theta, X)}}}}\\
        =&  \ex[(X, Y) \sim \dataset]{\trace{\left(\nabla_\theta \log f(Y \mid \theta, X)\right)^\transpose \left(\nabla_\theta \log f(Y \mid \theta, X)\right)}}
         \\&- \ex[\dataset(X)]{\trace{\ex[\dataset(Y \mid X)]{\frac{\nabla_\theta^2 f(Y \mid \theta, X)}{f(Y \mid \theta, X)}}}}\\
        =&  \ex[(X, Y) \sim \dataset]{\norm{\nabla_\theta \log f(Y \mid \theta, X)}_2^2} 
            - \ex[\dataset(X)]{\trace{\ex[\dataset(Y \mid X)]{\frac{\nabla_\theta^2 f(Y \mid \theta, X)}{f(Y \mid \theta, X)}}}}.
    \end{align}
    For well learned model, there is $f(Y \mid \theta, X) \approx \dataset(Y \mid X)$, so given the finiteness of label space there is
    \begin{align}
        \ex[\dataset(Y \mid X)]{\frac{\nabla_\theta^2 f(Y \mid \theta, X)}{f(Y \mid \theta, X)}}
        =&  \sum_{y} \dataset(y \mid X) \cdot \frac{\nabla_\theta^2 f(y \mid \theta, X)}{f(y \mid \theta, X)}\\
        \approx&    \sum_{y} \nabla_\theta^2 f(y \mid \theta, X) = \nabla_\theta^2 \sum_y f(y \mid \theta, X) \\
        =& \nabla_\theta^2 1 = \mathbf{0}, \label{step:sum_probabilities}
    \end{align}
    where \cref{step:sum_probabilities} follows the assumption that $f$ outputs a distribution over label space that sums to $1$.
\end{proof}

\input{theory/more-preliminary.tex}
\input{theory/ext-mp.tex}

%% file: theory/more-preliminary.tex
\section{More Preliminaries}\label{appendix:more_preliminary}

In this section, we introduce more preliminaries and notations used when proving \cref{theorem:spectral_of_accumulated}.

Recall from \cref{sec:preliminary} that $\norm{\cdot}_p$ for matrices is Schatten norms, or $L_p$ norm conducted on singular values. 
Additionally, let $\norm{\cdot} \defeq \norm{\cdot}_\infty$ indicate the spectral norm, i.e. $L_\infty$ norm or the maximum magnitude of singular values. 
One major difference with the main text is that complex matrices are involved (although most vectors are still real vectors). So conjugate transposition is substituted for transposition in traces' connection to Schatten 2-norms and elementwise $L_2$ norms. 
More properties of traces are involved, for example, \cref{lemma:abs_trace_and_schatten_1} is used to introduce Schatten 1-norm to bound approximation errors.

\begin{lemma}[Absolute Trace and Schatten 1-norm]\label{lemma:abs_trace_and_schatten_1}
    For real or complex $n \times n$ matrix $A \in \complexes^{n \times n}$,
    \begin{align}
        \abs{\trace{A}}  \le \sum_{i} \abs{\lambda_i(A)} \le \sum_{i} \abs{\sigma_i(A)} \defto \norm{A}_1,
    \end{align}
    where $\sigma_i(A)$ indicates singular values of $A$ and equals to the absolute value of the $i$-th eigenvalue in normal matrices.
\end{lemma}
\begin{proof}
    This proof is from \citet{sum_eigenvalue_and_singular_value}.
    By Schur decomposition, there exists unitary $S$ and upper triangular matrix $T$ such that $A = S T S^{-1}$. Note that $T$ shares the same eigenvalues, which are its diagonal entries, as $A$. There exists $d_i$ such that $\lambda_i(A) = \abs{\lambda_i(A)} d_i$ and $\abs{d_i}=1$. Let $D \defeq \diag{d_i}$ and note that $D$ is unitary. As a result, there exist unitary matrices $D^* S^*$ and $S$ such that
    \begin{align}
        \trace{D^* S^* A S} = \sum_{i=1}^n \abs{\lambda_i(A)}.
    \end{align}

    On the other hand, the sum of singular values is also
    \begin{align}
        \sum_{i} \sigma_i(A) = \max_{X, Y \in U_n} \abs{\trace{X A Y}} \label{eq:another_expression_of_singular_sum},
    \end{align}
    where $U_n$ is the set of all unitary matrices in $\complexes^{n \times n}$. 
    To see this, first notice that the singular value decomposition $A = U \Sigma V^*$ naturally provides a pair of $X, Y$ that achieve $\sum_{i} \sigma_i(A)$. For other unitary matrices, $\trace{X A Y} = \trace{V^* Y X U \Sigma} = \sum_{i} \left(V^* Y X U\right)_{i, i} \sigma_i$ and
    \begin{align}
        \abs{\trace{X A Y}} = \abs{\sum_i \left(V^* Y X U\right)_{i, i} \sigma_i} \le \sum_i \abs{\left(V^* Y X U\right)_{i, i}} \sigma_i.
    \end{align}
    Note that $V^* Y X U$ is also unitary, whose entries' magnitudes are no greater than $1$, which implies $\abs{\trace{X A Y}} \le \sum_{i} \sigma_i$ and \cref{eq:another_expression_of_singular_sum} is proved.

    With $D^* S^*$ and $S$ also unitary, the lemma directly follows.
\end{proof}

Schatten 1-norms are usually not direct from assumptions and conditions, so we apply Hölder's inequality on Schatten norms to produce
\begin{align}
    \norm{X Y}_{1} \le \norm{X}_{1} \norm{Y}_{\infty},
\end{align}
where $\frac{1}{p} + \frac{1}{q} = \frac{1}{1} + \frac{1}{\infty} = 1$. 
If $X$ is real symmetric positive semi-definite then $\norm{X}_1 = \sum_{i} \abs{\lambda_i} = \sum_{i} \lambda_i = \trace{X}$. If $X$ is further an outer product of real vectors, then $\norm{X}_1 = \trace{X}$ can be bounded by upperbounds on vector norms. One can also prove
\begin{align}
    \norm{x}_1 \le \sqrt{\dim x} \norm{x}_2
\end{align}
to relate $L_1$ and $L_2$ norms, and thus Schatten 1-norms and 2-norms by seeing them as $L_p$ norms conducted on singular values.
The bound on $\norm{Y}_{\infty} = \norm{Y}$ is often provided by \cref{lemma:3.1_from_mp_quadratic_form} introduced later.

%% file: theory/ext-mp.tex
\NewDocumentCommand{\uut}{}{u^p \left(u^p\right)^\transpose}
\NewDocumentCommand{\utau}{O{u}}{\left(#1^p\right)^\transpose A^p #1^p}
\NewDocumentCommand{\xtax}{}{\utau[x]}

\section{Proof of \cref{theorem:spectral_of_accumulated}}

\NewDocumentCommand{\sti}{}{Stieltjes}
\RenewDocumentCommand{\uut}{O{k}}{u_{#1} u_{#1}^\transpose}
\NewDocumentCommand{\xxt}{O{k}}{x_{#1} x_{#1}^\transpose}

More notations and preliminary information are used in the proof. They are listed in \cref{appendix:more_preliminary}.

To have a clear goal, we first point out how the ratios in \cref{theorem:spectral_of_accumulated} emerge. 
In random matrix theory (RMT), a powerful tool is \sti{} transform when it comes to spectral distribution. 
\sti{} transform of a distribution $\mu$ is 
\begin{align}
    s^\mu(z) \defeq \int_{\reals} \frac{1}{\lambda - z} \mu(\dd \lambda), \forall z \in \positivecomplex,
\end{align}
where $\positivecomplex \defeq \set{u + v i: u, v \in \reals, v > 0} \subset \complexes$.
It can be seen as the expected value of inverses. So if the \sti{} transform of the spectral distribution' bound is computed and multiplied with the average eigenvalue that can be computed by trace as the sum of eigenvalues, then the expected fraction is bounded. \cref{lemma:sti_and_eigenvalue_ratio} fulfills this intuition.

\NewDocumentCommand{\approxquadratic}{m}{\sqrt{\frac{2}{v} \left(c + #1\right)}}
\begin{lemma}[\sti{} transform and eigenvalue ratio]\label{lemma:sti_and_eigenvalue_ratio}
    For any density $\mu$ over real numbers that is only supported on positive numbers, there is
        \begin{align}
            \rpart{s^{\mu}(v i)}
            =&  \int \frac{1}{\lambda + \frac{v^2}{\lambda}} \mu(\dd \lambda) = \ex{\frac{1}{\lambda + \frac{v^2}{\lambda}}},\\
            \ipart{s^{\mu}(v i)}
            =&  \int \frac{1}{\left(\frac{\lambda}{\sqrt{v}}\right)^2  + v} \mu(\dd \lambda) = \ex{\frac{1}{\left(\frac{\lambda}{\sqrt{v}}\right)^2 + v}},
        \end{align}
    where $v \in \reals^+$.
    As a result, if upperbound or mean of $\lambda$ is obtained, one can estimate the averaged ratio between eigenvalues.
\end{lemma}
\begin{proof}
    By definition,
    \begin{align}
        s^{\mu}(z)
        \defeq& \int \frac{1}{\lambda - z} \mu(\dd \lambda)
        =  \int \frac{\lambda - \bar{z}}{\left(\lambda - z\right)\left(\lambda - \bar{z}\right)} \mu(\dd \lambda)
        =  \int \frac{\lambda - u + v i}{\lambda^2 - 2 u \lambda + u^2 + v^2} \mu(\dd \lambda), \\
        \rpart{s^{\mu}(z)}
        =&  \int \frac{\lambda - u}{\lambda^2 - 2 u \lambda + u^2 + v^2} \mu(\dd \lambda)
        =   \int \frac{1}{\lambda - u + \frac{v^2}{\lambda - u}} \mu(\dd \lambda), \\
        \ipart{s^{\mu}(z)}
        =&  \int \frac{v}{(\lambda - u)^2 + v^2} \mu(\dd \lambda).
    \end{align}
    By setting $z = 0 + v i$, there is
    \begin{align}
        \rpart{s^{\mu}(z)}
        =&  \int \frac{1}{\lambda + \frac{v^2}{\lambda}} \mu(\dd \lambda),
        \ipart{s^{\mu}(z)}
        =   \int \frac{1}{\left(\frac{\lambda}{\sqrt{v}}\right)^2  + v} \mu(\dd \lambda).
    \end{align}
\end{proof}

We adapt the proof of Theorem 2.1 from \citet{mp_quadratic_form}, where a quadratic equation $\frac{\ex{S_p}}{1 + \ex{S_p}} - z \ex{S_p} = c + o(1) = \frac{p}{b T} + o(1)$ bridges the \sti{} transform of the spectral distribution and value $c$. We follow a similar path, but in a non-asymptotic and dependent scenario where it is hard to obtain an $o(1)$ residual. So we turn to a generalized form where the residual term is bounded. \cref{lemma:bound_of_sti} explores how we can bound $\ex{S_p}$, which is scaled \sti{} transform of the empirical spectral distribution, by $c$ and the magnitude of the residual term.

\begin{lemma}[Bounds on continuous roots of a equation]\label{lemma:bound_of_sti}
    Suppose $S$ is a function of $z, p, b, T$ such that
    \begin{align}
        \frac{S}{1 + S} - z S = c + t,
    \end{align}
    where $z = 0 + v i \in \positivecomplex (v \in \reals^+), c = p / b T \in [0, 1]$, and $t \in \complexes$ is also a function of $z, p, b, T$.
    Assume $S$ is continuous w.r.t. $z$ and further assume $S$ always has a non-negative real part.
    Assume $\abs{t}$ is upperbounded by function $\tau(p, b, T) / v$ where $\tau$ is constant w.r.t. $v$.
    If for some $v_0 \ge 2 c$, there is 
    \begin{align}
        \frac{v_0 + (1 - c)}{\sqrt{2}} > \frac{\tau}{v_0} + 2 \sqrt{c v_0} + 2 \sqrt{\tau}
    \end{align}
    then for all $v \ge v_0$, there is
    \begin{align}
        \abs{\ipart{S}}, \abs{\rpart{S}} \le \abs{S} \le&  \sqrt{\frac{2}{v} \abs{c + t}}.
    \end{align}
\end{lemma}
\begin{proof}
    Assume $p, b, T$ is fixed and $v \ge v_0$. 
        
    First of all, since $t = \frac{S}{1 + S} - z S - c$, where $S$ is continuous w.r.t. $z$, when well defined $t$ is also continuous w.r.t. $z$. Since $S$ is well defined for all $z = v i$ and $S$ has a non-negative real part, meaning $1 + S \neq 0$, $t$ is defined and continuous for all $z = v i$.

    $\frac{S}{1 + S} - z S = c + t$ implies
    \begin{align}
        S^2 + \frac{z - 1 + c + t}{z} S + \frac{c + t}{z} = (S + b)^2 - b^2 + \frac{c + t}{z} = 0,
    \end{align}
    where $b = \frac{z - 1 + c + t}{2 z}$.
    Consider a modified version of this quadratic equation by altering the zeroth order term
    \begin{align}
        \left(R + b\right)^2 - b^2 + \frac{c + t}{z} u = 0. \label{eq:modified_quadratic}
    \end{align}
    where $R = R(z, p, T, u)$ is a solution of the modified equation. Since $z, t, u$  depend on $z, u$ in a continuous manner and there is no singular point given $\rpart{z} > 0$, there exist two continuous solutions $R_1$ and $R_2$ to the equation that is well defined for all $z=v i$ and $u$ where $v \in \reals^+, u \in [0, 1]$. Conversely, all possible solutions $R$ are contained in the union of manifolds given by $R_1$ and $R_2$.
    Since
    \begin{align}
        &\forall u \in [0,1], - b^2 + \frac{c + t}{z} u \neq 0
        \impliedby  \abs{b^2} > \abs{\frac{c + t}{z}}\\
        \impliedby& \abs{v i - 1 + c + t} > 2 \sqrt{v \abs{c + t}}
        \impliedby  \sqrt{v^2 + (1 - c)^2} - \abs{t} > 2 \sqrt{v} \sqrt{c + \abs{t}}\\
        \impliedby& \sqrt{v^2 + (1 - c)^2} > \abs{t} + 2 \sqrt{v} \sqrt{c + \abs{t}}
        \impliedby  \sqrt{v^2 + (1 - c)^2} > \frac{\tau}{v} + 2 \sqrt{c v + \tau}\\
        \impliedby& \frac{v + (1 - c)}{\sqrt{2}} > \frac{\tau}{v} + 2 \sqrt{c v} + 2 \sqrt{\tau} \quad\quad\left(\text{LHS: concavity of $\sqrt{x}$}\right)\\
        \impliedby& \frac{v_0 + (1-c)}{\sqrt{2}} > \frac{\tau}{v_0} + 2 \sqrt{c v_0} + 2 \sqrt{\tau} \\&    \land \forall v \ge v_0,
            \frac{\dd \left(\frac{v + (1 - c)}{\sqrt{2}} - \left(\frac{\tau}{v} + 2 \sqrt{c v} + 2 \sqrt{\tau}\right)\right)}{\dd v} = \frac{1}{\sqrt{2}} + \frac{\tau}{v^2} - \frac{\sqrt{c}}{\sqrt{v}} \ge \frac{1}{\sqrt{2}} - \frac{\sqrt{c}}{\sqrt{v}} \ge 0\\
        \impliedby& \frac{v_0 + (1 - c)}{\sqrt{2}} > \frac{\tau}{v_0} + 2 \sqrt{c v_0} + 2 \sqrt{\tau} \land v_0 \ge 2 c,
    \end{align}
    $R_1$ and $R_2$ has no intersection when $v \ge v_0$. As a result, continuous $S$, as $R(u=1)$ for some $R$, can only be found in exactly one of $R_1$ and $R_2$ and thus is either $R_1(u=1)$ or $R_2(u=1)$.

    Consider $R$s' behaviours when $u=0$. $R_1$ and $R_2$ are either $0$ or $-2b$. By their being continuous and disjoint, $R_1$ and $R_2$ can only be universally one of them otherwise there would be discontinuity. Without loss of generality, let
    \begin{align}
        R_1(z, p, T, 0) =&  0,\\
        R_2(z, p, T, 0) =&  -2 b =  \frac{1 - z - c - t}{z} = \frac{-i - v + c i + t i}{v}.
    \end{align}
        
    Assume $S = R_2(u=1)$ for contradiction. If so, $\lim_{v \to \infty} \rpart{R_2(u=1)} \ge 0$. Consider the following trajectory of $(z, u)$ to conclude that $\lim_{v \to \infty} \rpart{R_2(u=1)} \le -0.5$:
    \begin{enumerate}
        \item Start from $z=i, u=0$, where $R_2 = -\frac{i}{v} - 1 + \frac{c + t}{v} i$.
        \item Increase $v$ alone until $v \ge 100 (c + \tau(z, b, p, T) / v)$. After this step, term $\frac{c + t}{v} i$ will not flip the sign of $-1$ given that $\abs{\frac{c + t}{v} i} \le \frac{c + \tau / v}{v} \le \frac{1}{100}$. Therefore, after this step $\rpart{R_2} \le -0.99$.
        \item Increase $u$ alone to $1$. During this step, the term $\frac{c + t}{z} u$ in \cref{eq:modified_quadratic}, which appears under squared root in $R_2$, changes $R_2$ at most by $\sqrt{\abs{\frac{c + t}{z}}} \le \sqrt{\frac{c + \tau}{v}} \le \frac{1}{10}$. Therefore, it will not flip the sign and $\rpart{R_2} \le -0.89$. 
        \item Increase $v$ to infinity. The influences of $\frac{c + t}{v}$ and $\frac{c + t}{v} u$ will further drop and become ignorable. So $R_2 \le -0.5$ in later parts of the trajectory.
    \end{enumerate}
    Note since $\tau/v$ monotonically decreases as $v$ increases, step (3) will happen when $v < \infty$ and the limit following this trajectory is equal to $\lim_{v \to \infty} R_2(z)$.
    Therefore, $\lim_{v \to \infty} R_2(z) \le -0.5$, $S \neq R_2$ and $S = R_1(u=1)$. As a result, $S$ can be approximated from $R_1=0$ at $u=0$.

    \NewDocumentCommand{\derivativenorm}{}{\frac{\abs{c + t}}{2 \sqrt{\abs{v^2 b^2 - v(c + t) u}}}}
    To this end, taking differentiation (the derivative w.r.t. $u$ exists except for a finite number of points given the elementary dependence on it) gives
    \begin{align}
        2(R(u) + b) \dd R + \frac{c + t}{z} \dd u = 0
    \end{align}
    or
    \begin{align}
        \abs{\frac{\dd R}{\dd u}} = \abs{-\frac{c + t}{2 z (R(u) + b)}} = \frac{\abs{c + t}}{2 v \sqrt{\abs{b^2 - \frac{c + t}{z} u}}} \le \derivativenorm.
    \end{align}
    Starting from $R(u=0)=0$, there is
    \begin{align}
        &\abs{S - R(u=0)}
        =  \abs{\int_{0}^1 \frac{\dd S}{\dd u} \dd u}\\
        \le&    \int_{0}^1 \derivativenorm \dd u
        \le     \frac{1}{2 v}\int_{0}^1 \frac{v \abs{c + t}}{\sqrt{\abs{v^2 \abs{b^2} - v \abs{c + t} u}}} \dd u\\
        =&      \frac{1}{2 v}\int_{v^2 \abs{b}^2 - v \abs{c + t}}^{v^2 \abs{b}^2} \frac{1}{\sqrt{\abs{w}}} \dd w
        \le     \frac{1}{2 v}\int_{- v \abs{c + t}/2}^{v \abs{c + t} /2} \frac{1}{\sqrt{\abs{w}}} \dd w
        =       \sqrt{\frac{2}{v} \abs{c + t}},
    \end{align}
    where the second inequality follows the triangle inequality for subtracting edges. The third inequality is done by moving $[v^2 \abs{b}^2 - v \abs{c + t}, v^2 \abs{b^2}]$ towards zero until it is symmetric, during which the end with the largest absolute value is moved to the other end and its absolute value is decreased, enlarging the inverted square root.
\end{proof}

As a result, if the residual term is bounded, so are the \sti{} transform and the expected fraction. In the proof of \citet{mp_quadratic_form} this quadratic equation is extracted by convergence, which is hard in non-asymptotic analysis. We extract it with approximation, whose error is bounded by vector norm bound $\alpha / p$, anisotropy bound $\beta / p$. 

Some matrices are also involved in the approximation, where a technical lemma by \citet{mp_quadratic_form} is very useful. \cref{lemma:3.1_from_mp_quadratic_form} states that a special matrix, which we will encounter many times in the following proof, has a bounded spectral norm. Combined with Hölder's inequality, bounds on matrices' spectral norms and vector norms ($\alpha/p$) will suppress the approximation error.

\begin{lemma}[Lemma 3.1 by \citet{mp_quadratic_form}]\label{lemma:3.1_from_mp_quadratic_form}
    Let $C \in \reals^{p \times p}$ be a real symmetric positive semi-definite matrix and $x \in \reals^p$. If $z \in \complexes$ is such that $v = \ipart{z} > 0$, then
    \begin{enumerate}
        \item\label{lemma:spectral_bound_of_reverse} 
            $\norm{\left(C - z I\right)^{-1}} \le 1 / v$ ;
        \item\label{lemma:x_means_little_in_reversed_sample_covariance}
            $\abs{\trace{C + x x^\transpose - z I}^{-1} - \trace{C - z I}^{-1}} \le 1 / v$;
        \item $\abs{x^\transpose\left(C + x x^\transpose - z I\right)^{-1} x} \le 1 + \abs{z} / v$;
        \item $\ipart{z + z \trace{C - z I}^{-1}} \ge v$ and $\ipart{\trace{C - z I}^{-1}} > 0$;
        \item $\ipart{z + z x^\transpose (C - z I)^{-1} x} \ge v$.
    \end{enumerate}
\end{lemma}

With these three lemmas, we can begin the proof of \cref{theorem:spectral_of_accumulated}. Following \citet{mp_quadratic_form}, we relate \sti{} transform of the spectral distribution with the quadratic equation. The quadratic equation is formed by approximation, whose error becomes $t$ in \cref{lemma:bound_of_sti}. After concluding the bounds of the \sti{} transform, applying \cref{lemma:sti_and_eigenvalue_ratio} gives the final conclusion.

\def\StateSpectralOfAccumulated{proof}

\input{theory/theorems/ext_mp.tex}

%% file: theory/effective-window.tex
\section{Effective Window Size}\label{appendix:effective_window_size}

We argue that the effective window size should be considered by $r = 1 - \eta_{\mathrm{lr}} w$ instead of $\sqrt{r} = \sqrt{1 - \eta_{\mathrm{lr}} w}$ to have better results, where $\eta$ is learning rate and $w$ is the strength of weight decay.
The argument relies on another formulation of eigenvalues: $\lambda_i(A)$ corresponds to critical points of Rayleigh quotient
\begin{align}
    \max_{v \in \reals^n \setminus \set{0}} \frac{v^\transpose A v}{v^\transpose v} = \max_{v \in \reals^n: \norm{v}_2 = 1} v^\transpose A v.
\end{align}
Particularly for weight decayed sample covariances where $S = U U^\transpose = \sum_{t, l} \sqrt{r}^{T - t} u_{t, l} \sqrt{r}^{T - t} u_{t, l}^\transpose$, its eigenvalues are given by critical points of
\begin{align}
        \max_{v \in \reals^n: \norm{v}_2 = 1} v^\transpose U U^\transpose v
    =&  \max_{v \in \reals^n: \norm{v}_2 = 1} \sum_{t, l} r^{T - t} \inner{u_{t, l}}{v}^2 \label{eq:all_samples}.
\end{align}

So terms are weighted by $r$ instead of $\sqrt{r}$. Now consider the situation where effective window $k$ is used to produce $S'$, whose eigenvalues are given by
\begin{align}
    \max_{v \in \reals^n: \norm{v}_2 = 1} \sum_{t=T-k+1} \sum_l r^{T - t} \inner{u_{t, l}}{v}^2 \label{eq:effective_window_samples},
\end{align}
and consider its difference from those of all samples.

For the smallest non-zero eigenvalues, there is $\lambda_{\min(p, b T)}(S) \ge \lambda_{\min(p, b T)}(S')$ since adding a real symmetric positive semi-definite matrix to $S'$ to obtain $S$ increases all eigenvalues. So we will not be overoptimistic in the smallest non-zero eigenvalues when it comes to bounding the fraction between the largest and smallest non-zero eigenvalues.

\NewDocumentCommand{\sums}{O{} O{} O{\inner{u_{t, l}}{v_{\max}}^2}}{\sum_{t#1}^{#2} \sum_l r^{T - t} #3}
For the largest eigenvalues, there is $\lambda_1(S) \le \lambda_1(S') + \sums[=1][T-k][\norm{u_{t, l}}_2^2]$.
To see this, let $v_{\max}$ be the unit vector that achieves $\lambda_1(S)$.
\begin{align}
    \lambda_1(S)
    =&      \sums[=1][T]
    =       \sums[=T-k+1][T] + \sums[=1][T-k]\\
    \le&    \max_{\norm{v}_2=1} v^\transpose S' v + \sums[=1][T-k][\norm{u_{t, l}}_2^2]
    =       \lambda_1(S') + \sums[=1][T-k][\norm{u_{t, l}}_2^2].
\end{align}
So at most will be overoptimistic by $\sums[=1][T-k][\norm{u_{t, l}}_2^2]$. 

As a result, if the exponential weight sum of the out-of-window tail, controlled by $r$ instead of $\sqrt{r}$, is small, then the over-optimism will be suppressed.

%% file: experiments/appendix.tex
\section{Experimental Details}\label{appendix:experimental_details}

In experiments for validation (\cref{sec:v_experiments}), ViT-Base/16 is trained on CIFAR-10 using Adam. Hyperparameters include learning rate of $10^{-4}$, batch size of $64$, no dropout, weight decay, learning rate decay or gradient clipping. Random cropping and random horizontal flipping is used as data augmentation. The training lasts for 100 epochs with the first 5 epochs of linear warmup. Automatic mixed precision (AMP) is not used. Note that only $\weird$ and unrefined $\dbmlp$ are used. $\dbmlp$ is \emph{not} restricted and there are \emph{no} restricted LayerNorm layers as well. The ViT implementation is based on that of Torch Vision. Logging happens every 10 steps.

During training ViT-Base/16 from scratch on ImageNet-1K, the recipe is adapted from \citet{pytorch_recipe}. Specifically, the optimizer is AdamW with learning rate $3 \times 10^{-3}$, $(\beta_1, \beta_2)=(0.9, 0.999)$, weight decay $0.3$, dropout $0.0$ and gradient clipping of global norm $1.0$. Model EMA is \emph{removed} to ease experiment tracking. Learning rate is scheduled with cosine annealing without warmup. Batch size is $2048$, distributed on 4 GPUs for 300 epochs. Other tricks include mixup of alpha $0.2$, cutmix of alpha $1.0$, auto augmentation of policy \texttt{ra}, label smoothing of $0.11$. AMP with \texttt{fp16} is used to accelerate training. Images are resized to $256 \times 256$ to save storage. Logging happens every 100 steps.
During finetuning, the differences include modifications for sparsity described in \cref{sec:refine} and \cref{sec:p_experiments}, LoRA of rank $192$, and epoch number reduced to $15$, among which LayerNorm uplifting takes $5$ epochs.

During training T5-Base from scratch on C4, the recipe is adapted from \citet{t5_recipe} and \citet{observation}. Specifically, the optimizer is AdamW with learning rate $0.01$, $(\beta_1, \beta_2)=(0.9, 0.999)$, weight decay $0.001$, dropout $0.0$ and gradient clipping of global norm $1.0$. Learning rate is scheduled with inverse square root scheduling with linear warmup of $10,000$ steps, while the full training lasts for $100,000$ steps. Batch size is $256$. To avoid downloading the entire data set, we use the streaming functionality of Huggingface \texttt{datasets} to download the first $25,600,000$ training samples without shuffling and the full validation split. Recipe from \citet{t5_recipe} includes data preparation that concatenates all raw sentences and re-splits them into samples. This step consumes too much storage so we confine the concatenation within a batch and empties are filled with padding tokens that are ignored during computing the loss. This compromise reduces the number of usable samples during training, but there are still approximately at least $50$ non-empty samples for every 64 samples. The generation of spans is left unchanged, where 15\% of the input is corrupted without any mixture. The encoder receives $512$ tokens while $114$ tokens are used on the decoder side. Logging happens every 25 steps. When training modified T5, the rented GPUs expired at step $95,000$. Nevertheless, the effectiveness of modification is already obvious at that step.
During finetuning, LoRA of rank $192$ and modifications described in \cref{sec:refine} and \cref{sec:p_experiments} are equipped. The training step reduces to $10,000$, among which learning rate warmup takes $1,000$ steps and LayerNorm uplifting and activation function mixing last for $3,000$ steps. Logging happens every $50$ steps during finetuning T5.

For experiments in \cref{sec:t_exp:anisotropy}, the tiny MLPs have an input layer of with $32 \times 32$, five hidden layers of width $p$ (or equivalently four hidden linear layers whose matrices are of size $p \times p$) and one output layer of width $10$. Skip connection is placed across each of the 4 hidden linear layers. LayerNorm layers with affine parameters \emph{turned off} are also placed before hidden layers. The optimizer is SGD with learning rate $10^{-3}$. Weight decay is altered as a part of the experiment. Batch size is $32$, the number of epochs is $100$. No learning rate warmup is used because it alters the weight decay after multiplying learning rate during training and bothers the computation of $r$. Data augmentation includes random cropping and random rotation of at most 30 degrees. Logging happens every 20 steps.

%% file: theory/magic.tex
\section{Theoretically Guided Magic: Massive Perturbation with Small Computation Cost and Good Parallelism}\label{appendix:magic}

\cref{lemma:flatness_of_perturbed_model} and \cref{theorem:main_with_effective_duplication} state that by adding massive synaptic noises (non-massive synaptic noises for single tokens can be found in works of \citet{synaptic_noise_1,synaptic_noise_2}) to weight matrices during forward propagation as in \cref{eq:massive_perturbation}, we are training a $k$ times larger effective model and optimizing its flatness w.r.t. these $k$ times parameters, where $k$ is the number of tokens. 
These noises can potentially greatly improve the generalization of deep CNNs, Transformers, MLP-Mixers or other architectures as long as they split input into multiple tokens.

However, \cref{eq:massive_perturbation} requires adding noises to weight matrices independently for each token before matrix multiplication, which is unfriendly to GPUs who prefer stacked matrix multiplication. Direct application of \cref{eq:massive_perturbation} also requires $k \times n \times d$ independent Gaussian noises. Therefore, here we propose an algorithm called $\magic$ that is friendly to parallelism and only requires $k \times n$ independent noises, which is the same number as dropout.

The core motivation of $\magic$ is to add noises \emph{after} matrix multiplication, which is possible by merging multiple Gaussian noises into one.
Let $Z^l$ be the result of multiplication with $\proxyW^{l, i}$s in \cref{eq:massive_perturbation} and $z^l_i \defeq Z^l_{\cdot,i}$ be the $i$-th (column) hidden vector/token in it. Then there is
\begin{align}
    z^l_i 
    \defeq& \proxyW^{l, i} x_i
    =   W^l x_i + E^{l, i} x_i
    =   W^l x_i + \begin{bmatrix}
        \inner{E^{l, i}_j}{x_i}
    \end{bmatrix}_j\\
    =&   W^l x_i + \begin{bmatrix}
        \sum_{q} X^l_{q, i} \cdot E^{l, i}_{j, q}
    \end{bmatrix}_j.
\end{align}
The summation is a linear combination of independent Gaussian variables, which is also Gaussian. Since entries in $E^{l, i}$ are assumed to follow $\gaussian{0}{\sigma^2}$, there is
\begin{align}
    \sum_{q} X^l_{q, i} \cdot E^{l, i}_{j, q} \sim \gaussian{0}{\sigma^2 \sum_{q} \left(X^l_{q, i}\right)^2} = \gaussian{0}{\sigma^2 \norm{x_i}_2^2}.
\end{align}
Note that $E^{l, i}_{j}$ is only used once for each $(i, j)$, so the linear combinations are independent and the perturbed matrix product simplifies to
\begin{align}
    z^l_i 
    =&   W^l x_i + e^{l, i}, 
\end{align}
where $n$-variate Gaussian noise $e^{l, i} \sim \gaussian{0}{\sigma^2 \norm{x_i}_2^2 I}$.
Therefore, we only need to compute the norm of input tokens, then sample $k \times n$ Gaussian noises columnly scaled by them and add the scaled noises to the matrix product, which is $\magic$. If zeroth biases are removed and LayerNorm layers are placed before linear layers, the squared norm further simplifies to constant $\norm{x_i}_2^2 = d$.

Interestingly, $\magic$ is theoretically ensured by \cref{lemma:flatness_of_perturbed_model} and \cref{theorem:main_with_effective_duplication} to encourage flatness if the empirical loss is small enough, even though it is only a first-order optimization and requires only one forward and one backward propagations. 
Moreover, by using $\magic$, one is magically expanding the number of parameters. Although their values are correlated, their flatness w.r.t. non-correlated perturbations is indeed optimized. To give a sense of this magical parameter expansion, for ViT-Base, $\magic$-ed ViT-Base is effectively 197 times larger, with approximately 17B flatness-effective parameters. In NLP tasks this expansion is even larger given the larger number of tokens.
$\magic$ can be easily implemented as a PyTorch forward \texttt{Hook} that is registered on all linear layers, which is useful in almost any architectures. For CNN, which can be seen as MLPs with parameters de-duplicated, the computation of input norms can be tricky but is also possible and efficient by all-1-kernel convolution on squared input feature maps.

$\sigma^2$ is a hyperparameter that must be determined beforehand. The noise should be small compared to the magnitudes of entries in weight matrices to allow approximation. However, the norm of parameters may change drastically during training, especially under weight decay. So a more practical solution is to adjust $\sigma^2$ dynamically. We consider a simple solution: for each weight matrix, we compute the average $L_1$ norm of entries in that matrix and multiply it with a pre-determined $\rho < 1$ to obtain $\sigma^2$. We call it $\magic[Adaptive]$.

The experiments of $\magic$ is under preparation and will come in future works. We introduce it here only for completeness.